\documentclass{article} 
\usepackage{arxiv_style,times}

\usepackage{amsmath,amsfonts,bm}

\def\eqref#1{equation~\ref{#1}}

\def\1{\bm{1}}

\def\vzero{{\bm{0}}}

\def\vtheta{{\bm{\theta}}}

\def\mF{{\bm{F}}}

\DeclareMathAlphabet{\mathsfit}{\encodingdefault}{\sfdefault}{m}{sl}
\SetMathAlphabet{\mathsfit}{bold}{\encodingdefault}{\sfdefault}{bx}{n}

\def\sR{{\mathbb{R}}}

\DeclareMathOperator*{\argmin}{arg\,min}

\usepackage{hyperref}
\usepackage{url}

\usepackage{amsmath}
\usepackage{amssymb}
\usepackage{amsthm}
\usepackage{tabularx}
\usepackage{ragged2e}
\usepackage{subcaption}
\usepackage{wrapfig}
\usepackage{graphicx}
\usepackage{thmtools}
\usepackage{booktabs}

\title{Towards Eliminating Catastrophic Forgetting in the Curriculum Learning of Math Reasoning Tasks}

\author{\shortstack[l]{Zengyan Yang\textsuperscript{1}, Yangyang Wu\textsuperscript{1}, Kai Huang\textsuperscript{2}, Pengfei Lyu\textsuperscript{2}, Tianyi Zhang\textsuperscript{2}, Mengying Zhu\textsuperscript{1}} \\
\textsuperscript{1}Zhejiang University, \textsuperscript{2}Ant Group
}

\iclrfinalcopy 

\newtheorem{theorem}{Theorem}

\begin{document}

\maketitle

\begin{abstract}
Curriculum learning has found broad application across numerous domains. Nevertheless, its effectiveness is intrinsically curtailed by catastrophic forgetting, driven by the shifts in model parameter distributions between curriculum tasks. In this paper, we investigate the phenomenon of catastrophic forgetting in this training paradigm, building on the established efficacy of curriculum learning. Our theoretical analyses of parameter update dynamics demonstrate that catastrophic forgetting in curriculum learning stems from the divergence of task optima, which is generally essential to the faster convergence of curriculum learning; therefore, forgetting cannot be completely eliminated. Based on this finding, we augment the training process and propose \textsf{IV-EWC}, which incorporates Elastic Weight Consolidation (EWC) into the curriculum learning objective to curb catastrophic forgetting in mathematical reasoning, a prototypical curriculum learning scenario. \textsf{IV-EWC} employs the influence function to construct a representative validation set from the curriculum's training data, which is used to drive dynamic regularization during training. We further present an extended theoretical analysis to show that EWC-based regularization methods mitigate catastrophic forgetting in curriculum learning, thereby providing theoretical support for \textsf{IV-EWC}. Empirical evaluations on three backbone models and three benchmarks indicate that curriculum learning exhibits catastrophic forgetting. \textsf{IV-EWC} alleviates this issue, reducing forgetting by 162\% on average relative to vanilla curriculum learning and yielding positive backward transfer, as evidenced by improved performance on easier tasks after subsequent training on challenging tasks.
\end{abstract}

\section{Introduction}

In the field of machine learning, curriculum learning can help improve generalization, accelerate convergence and find better local minima \citep{Bengio2009CurriculumLearning, Guo2018CurriculumNet, Kumar2010SelfPaced, Hacohen2019OnThePowerOfCL}. Nowadays, it has been widely adopted in the post-training scenario, combined with new methods such as reinforcement learning to improve the reasoning capabilities \citep{Xi2024ReverseCurriculum, Wen2025Light-R1, Parashar2026E2HReasoner} of large language models.

Curriculum learning decomposes the original problem into an ordered sequence of subtasks. The resulting non-stationary training distribution can induce catastrophic forgetting \citep{McCloskey1989CatastrophicInterference, French1999CatastrophicForgetting, Hetherington2014IsThere}: the accuracy on easier tasks drops and the model is inclined to apply strategies learned for challenging tasks to easier tasks, where they may be suboptimal. Yet most methods optimize only the curriculum design for forward progress and overlook backward stability, leaving performance vulnerable to forgetting.

Catastrophic forgetting can be understood as optimization-driven interference in the parameter space. Training on a newly introduced task directs gradient updates toward parameter regions that minimize the new task's loss. When the loss minima across tasks differ from each other, these updates move the model away from regions that previously achieved low loss on earlier tasks, leading to performance degradation. Importantly, forgetting is defined relative to task delineations and their ordering: the extent of interference depends on the task objectives and data distributions that constitute the curriculum. Thus, overparameterization alone does not prevent forgetting; even large language models trained sequentially can exhibit path dependence that favors challenging tasks.

In this study, we conduct theoretical analyses of curriculum learning in an idealized regime in which the loss of the target problem and all intermediate curriculum problems are smooth and strongly convex. Examining the parameter updating dynamics, we find that curriculum learning's benefit derives from accelerated convergence. At the same time, catastrophic forgetting arises from non-coincident optima across tasks. In particular, achieving a reduced iteration count generally requires that the optima of the easier tasks do not coincide with those of the challenging tasks. This entails an inherent trade-off: the very differences that enable faster convergence also induce forgetting, rendering it impossible to completely eliminate forgetting in this setting.

To this end, we address catastrophic forgetting within curriculum learning by intervening in the training process. We introduce Influence-Validation EWC (\textsf{IV-EWC}), which dynamically applies EWC regularization \citep{Kirkpatrick2017EWC} using a validation set selected from the training data of curriculum tasks via the influence function \citep{Koh2017InfluenceFunction}, thereby protecting parameters critical to easier tasks of the curriculum. We focus on mathematical reasoning tasks because they exhibit clear difficulty hierarchies and explicit solution steps, enabling objective curriculum construction and transparent assessment of the reasoning process. Building on the analytical framework of forgetting in curriculum learning, we further show that EWC-based regularization alleviates this phenomenon under the same setting, backing up \textsf{IV-EWC} in theory. To summarize, our main contributions are described as follows.

\begin{itemize}
  \item We identify and formalize the phenomenon of catastrophic forgetting in curriculum learning, providing theoretical analysis from the parameter update perspective in gradient descent to establish that forgetting can be attributed to the differences of optima across curriculum tasks, which is required for fewer iterations in curriculum learning, thus forgetting exists and cannot be eliminated.
  \item We propose \textsf{IV-EWC}, which integrates influence-function-guided selection of representative training samples with a dynamic EWC regularizer, aiming to stabilize parameters important to easier curriculum tasks while adapting to more challenging ones in mathematical reasoning tasks.
  \item Within the same theoretical setting, we investigate EWC-based regularization in curriculum learning and prove that it attenuates catastrophic forgetting by constraining the distribution shifts of the parameters, thereby providing a theoretical justification for \textsf{IV-EWC}.
  \item We conduct extensive experiments with three base models on three math reasoning datasets; the results corroborate the presence of forgetting in curriculum learning and demonstrate that \textsf{IV-EWC} substantially mitigates this degradation.
\end{itemize}

\section{Related Work}

\subsection{Curriculum Learning on Reasoning Tasks}

Curriculum learning \citep{Bengio2009CurriculumLearning} organizes training from easier to more difficult instances or tasks and has inspired self-paced, teacher-student, and automated curricula \citep{Graves2017AutomatedCL, Tambet2020TeacherStudentCL, Jiang2018MentorNet, Hacohen2019OnThePowerOfCL}. These strategies have proved effective across deep learning, including neural machine translation \citep{KocmiBojar2017CurriculumNMT, Zhang2019CurriculumForTranslation, Platanios2019CompetenceCLForTranslation, Liu2020NormCLForTranslation} and computer vision with noisy labels \citep{Guo2018CurriculumNet, Xiao2024CLIP-VG}. Curriculum learning has also been explored in medical report generation \citep{Liu2021CurriculumForMedicalReport}, biological sequence prediction \citep{Zhang2025CurriculumForBiological}, as well as robotic control \citep{Florensa2017ReverseCLForRobot, Wang2019POET, Wang2020EnhancedPOET, Leyendecker2021CurriculumForRobot}. With the emergence of reinforcement learning methods such as DPO \citep{Rafailov2023DPO} and GRPO \citep{Shao2024GRPO}, there has been a growing body of research that combines curriculum learning with reinforcement learning to improve the reasoning capabilities of models. R\textsuperscript{3} \citep{Xi2024ReverseCurriculum} uses outcome supervision and achieves the effect of process supervision by constructing a reverse curriculum sequence. Light-R1 \citep{Wen2025Light-R1} constructs a three-stage curriculum post-training strategy that leverages carefully selected reasoning data to train long-CoT models. E2H Reasoner \citep{Parashar2026E2HReasoner} designs a scheduling strategy that utilizes data with difficulty annotations to automate the curriculum learning process, thereby improving the performance on complex tasks.

\subsection{Mitigating Forgetting in Continual Learning}

When investigating forgetting in curriculum learning, some methods in continual learning \citep{Thrun1995ContinualLearning} can serve as references. iCaRL \citep{Rebuffi2017iCaRL} employs a herding-based strategy to select representative samples for replay while adding a distillation loss to the objective. EWC \citep{Kirkpatrick2017EWC} applies parameter-wise regularization based on the Fisher Information Matrix. Deep Generative Replay \citep{Shin2017DGR} trains another generative model to generate replay samples that mimic past tasks. CLEAR \citep{Rolnick2019CLEAR} leverages off-policy learning and behavioral cloning for replay and on-policy learning for new samples. GEM \citep{LopezPaz2017GEM} utilizes episodic memory to control the gradient update process by projecting the gradient onto the feasible region defined by previous tasks' memories. A-GEM \citep{Chaudhry2019AGEM} simplifies the constraints of GEM and updates gradients according to sampled memories, achieving better efficiency. UPGD \citep{Elsayed2024UPGD} uses utility to assess the importance of each parameter in the training process and parameters with high utility are protected. \citet{Jiang2025FunctionVector} investigate the correlation between the function vector and forgetting, applying a function-vector based regularization to minimize forgetting.

\section{Problem Formulation}
In this section, we establish the effectiveness of curriculum learning for optimization problems with smooth and strongly convex loss functions. It further shows, from the parameter updating process, that forgetting is intrinsic to curriculum learning and impairs its eventual performance.

Let the loss of the original problem be a smooth and strongly convex function $f(\vtheta)$, with the global optimal parameter $\vtheta^*$ and the condition number $\kappa$. Define the loss of the easier curriculum problem as $f_{e}(\vtheta)$, with its global optimal parameter $\vtheta_{e}^{*}$ and its condition number $\kappa_e$. There exists a $\delta>0$ which satisfies $|| \vtheta_e^*-\vtheta^* || \leq \delta$. Both the original and the easier problems are optimized using gradient descent, where $\vtheta_0$ is the initial parameter, $T_f$ and $T_e$ are the iterations when the original problem and the easier curriculum problem converge respectively.

Under standard smoothness and strong convexity assumptions, we first bound the distance between the current iterate and the optimal parameter in Lemma \ref{lem:GDconvergence} via the update recursion, then we quantify the number of iterations required to reduce the parameter error below a prescribed tolerance in Lemma \ref{lem:GDtimes}. These lemmas provide the foundational results for the subsequent theorems on the effectiveness of curriculum learning and the phenomenon of forgetting therein.

\begin{restatable}{lemma}{GDconvergenceLemma}\label{lem:GDconvergence}
For the target problem with a $\mu$-strongly convex and $L$-smooth loss, let $t$ be the optimization step, $\kappa = \frac{L}{\mu}$ be the condition number of the loss, and $\vtheta^t$, $\vtheta^*$, $\vtheta_0$ be the parameter at step $t$, the optimal parameter of the target problem and the initial parameter of the target problem, respectively. The distance between parameters follows:
\begin{equation}
  || \vtheta_t-\vtheta^* ||^2 \leq \left( 1-\frac{1}{\kappa}\right) ^t || \vtheta_0-\vtheta^* ||^2 \label{eq:GDconvergence}
\end{equation}
\end{restatable}

The proof of Lemma \ref{lem:GDconvergence} is detailed in Appendix \ref{sec:GDconvergenceProof}, where we first analyze the distance between consecutive iterates using the smoothness and strong convexity and then derive the final estimate for the distance from the current iterate to the optimal parameter. 

\begin{restatable}{lemma}{GDtimesLemma}\label{lem:GDtimes}
For the target problem with a $\mu$-strongly convex and $L$-smooth loss, define the convergence condition $|| \vtheta_T-\vtheta^* || \leq \varepsilon$ where $\vtheta^T$ and $\vtheta^*$ are the parameter at convergence and the optimal parameter of the target problem respectively, let $C > 0$ be a constant and $\kappa$ be the condition number of the loss. The number of iterations required by gradient descent can be expressed as:
\begin{equation} 
  T \approx C \kappa \log \left(\frac{|| \vtheta_0-\vtheta^* ||}{\varepsilon}\right) \label{eq:GDtimes}
\end{equation}
\end{restatable}

The proof of Lemma \ref{lem:GDtimes} is detailed in Appendix \ref{sec:GDtimesProof}, which relies on Lemma \ref{lem:GDconvergence} to obtain a refined upper bound on the parameter distance and derive an estimate of the number of iterations required.

From the perspective of parameter updating, we utilize Lemma \ref{lem:GDtimes} to estimate the iterations required for curriculum learning and for non-curriculum training under the same convergence condition in Theorem \ref{thm:CLeffectiveness}. Under the assumptions that the loss of the easier curriculum problem has a lower condition number than the loss of the original problem and that their optimal parameters are close in distance, we establish that initializing on the easier problem and then transitioning to the original problem accelerates convergence, which is the effectiveness of curriculum learning.

\begin{restatable}{theorem}{CLeffectivenessTheorem}
  Given the convergence condition of the easier curriculum problem $|| \vtheta_{T_e} - \vtheta_e^* || \leq \varepsilon_e$ and the convergence condition of the original problem $|| \vtheta_T-\vtheta^* || \leq \varepsilon$, for curriculum satisfying $|| \vtheta_e^* - \vtheta^* || \ll || \vtheta_0-\vtheta^* || - \varepsilon_e$ and $\kappa_e<\kappa$, curriculum learning requires fewer iterations to converge compared with learning at once.
\label{thm:CLeffectiveness}
\end{restatable}

The detailed proof of Theorem \ref{thm:CLeffectiveness} is given in Appendix \ref{sec:CLeffectivenessProof}. First we employ Lemma \ref{lem:GDtimes} to establish the relation of the iteration counts and the parameter distance during curriculum learning, then we formulate the iterations saved to show the effectiveness of curriculum learning.

Theorem \ref{thm:CLeffectiveness} demonstrates that appropriately designed curricula can reduce the number of iterations required to converge. Complementing this, we take a specific curriculum as an example to analyze why and when forgetting emerges during curriculum learning in Theorem \ref{thm:CLforget}, which shows that discrepancies between optima of curriculum tasks inevitably produce forgetting in curriculum learning. This forgetting is essential to the construction of curricula in Theorem \ref{thm:CLeffectiveness} with losses satisfying increasing condition numbers, which drive accelerated convergence. Hence, achieving acceleration via curriculum learning necessarily entails some degree of forgetting.

\begin{theorem}
  For the curricula with loss functions $f_t(\vtheta) = \frac{L_t}{2}||\vtheta - \vtheta_t^*||^2$, where $\vtheta_t^*$ is the optimal parameter, $t = 1, 2, ..., T$, $L_t > L_{t - 1} > 0$, $||\vtheta_t^* - \vtheta_{t - 1}^*||^2 > 0$ and $||\vtheta_{t - 1}^*||^2 > 0$, when optimized starting from $\vtheta_0 = \vzero$ with learning rate $\eta > 0$, forgetting exists and cannot be eliminated.
\label{thm:CLforget}
\end{theorem}

\begin{proof}

Forgetting is defined as the average loss of previous tasks. Let the parameter after learning the last task be $\vtheta_T$, and forgetting can be formulated as:
\begin{equation}
  \mathrm{Fgt} = \frac{1}{T}\sum_{t = 1}^{T} \left(f_t(\vtheta_T) - f_t^*(\vtheta_t^*)\right) \label{eq:IPMforget}
\end{equation}

In accordance with IPM \citep{Cai2025IPM}, the parameter update for optimizing the $t$-th task with learning rate $\eta$ is given by:
\begin{equation}
    \vtheta_t = \mathrm{prox}_{\eta f_t}(\vtheta_t) = \argmin_{\vtheta \in \sR^d} \left(\frac{1}{2\eta} || \vtheta - \vtheta_{t - 1} ||^2 + f_t(\vtheta) \right) \label{eq:IPMupdate}
\end{equation}

Using Eq. \ref{eq:IPMupdate}, for the specified curriculum tasks we have:
\begin{equation}
  \vtheta_t = \argmin_{\vtheta \in \sR^d} \left(\frac{1}{2\eta} ||\vtheta - \vtheta_{t - 1}||^2 + \frac{L_t}{2}||\vtheta - \vtheta_t^*||^2 \right) \label{eq:IPMCurriculum}
\end{equation}

The optimization problem in Eq. \ref{eq:IPMCurriculum} is quadratic and possesses a global minimum. Setting its derivative to zero yields
\begin{align}
    \left(\frac{1}{\eta} + L_t\right) \vtheta_t &= \frac{\vtheta_{t - 1}}{\eta} + L_t \vtheta_t^* \nonumber \\
    \vtheta_t &= \frac{\vtheta_{t - 1} + \eta L_t \vtheta_t^*}{1 + \eta L_t} \label{eq:IPMCase}
\end{align}

To simplify the expression for $\vtheta_t$, we define $\alpha_t = \frac{1}{1 + \eta L_t}$. Eq. \ref{eq:IPMCase} can be rewritten as:
\begin{equation}
  \vtheta_t = \alpha_t \vtheta_{t - 1} + (1 - \alpha_t) \vtheta_t^* \label{eq:CurriculumUpdate}
\end{equation}

From Eq. \ref{eq:CurriculumUpdate}, the process of parameter updating consists of two parts. $\alpha_t \vtheta_{t - 1}$ is the starting point of the update representing the regularization loss, while $(1 - \alpha_t) \vtheta_t^*$ corresponds to the target of the update representing the target loss.

The curriculum design satisfies an ascending $L_t$ and a constant $\eta$, thus $\alpha_t$ decreases with respect to $t$. As the curriculum concludes, $\vtheta_t$ approaches $\vtheta_t^*$.

We now turn to the forgetting in the curriculum. From Eq. \ref{eq:IPMforget}, we have:
\begin{equation}
  \mathrm{Fgt} = \frac{1}{T}\sum_{t = 1}^{T} \left( \frac{L_t}{2} ||\vtheta_T - \vtheta_t^*||^2 \right) \label{eq:Forget}
\end{equation}

Based on the analysis of Eq. \ref{eq:CurriculumUpdate}, since at the end of the curriculum $\vtheta_T \approx \vtheta_T^*$, we have
\begin{equation}
  \mathrm{Fgt}^{vanilla} \approx \frac{1}{T}\sum_{t = 1}^{T} \left( \frac{L_t}{2} ||\vtheta_T^* - \vtheta_t^*||^2 \right) \label{eq:ForgetVanilla}
\end{equation}

Following the curriculum design we have $L_t > 0$ and $||\vtheta_T^* - \vtheta_t^*||^2 > 0$. Consequently, $\mathrm{Fgt} > 0$ always holds, and forgetting cannot be eliminated.

\end{proof}

\section{\textsf{IV-EWC} Method}

In this section, we introduce \textsf{IV-EWC} based on the preceding analysis of the effectiveness and forgetting in curriculum learning.
We first detail the augmented training strategy of \textsf{IV-EWC}. We then theoretically analyze the effect of incorporating the Elastic Weight Consolidation (EWC) loss term, which utilizes the Fisher Information Matrix computed from the previous training data to provide parameter-wise regularization, characterizing the ability of EWC-based methods to mitigate catastrophic forgetting in curriculum learning.

\subsection{Influence-Validation EWC}

As formalized in Eq. \ref{eq:ForgetVanilla}, forgetting in curriculum learning arises because different tasks induce distinct optima for the model parameters. Since tasks in the curriculum generally do not share the same optima, forgetting cannot be completely eliminated through curriculum design alone.
Nevertheless, constraining parameter updates during training can limit drift toward the optima of the challenging tasks, providing a new perspective to eliminate catastrophic forgetting.
Motivated by this observation and the relatedness among tasks in the curriculum, \textsf{IV-EWC} employs the influence function to select a progressively expanding validation set from the curriculum's training data to compute the Fisher Information Matrix and apply EWC regularization, thereby attenuating the impact of noisy samples that contribute little to an informative regularization signal.

In curriculum learning, challenging tasks generally subsume the knowledge required for easier tasks. Excessive regularization during training on the challenging tasks can bias the model toward simple representations that favor the easier tasks, thereby limiting its ability to capture the additional complexity of the challenging tasks. To mitigate this effect, we employ a dynamic regularization coefficient $\lambda = \frac{\gamma}{t}$, where $\lambda$ is the regularization weight of EWC loss, $\gamma > 0$ is a hyperparameter controlling the initial regularization weight, and $t$ denotes the task index starting from 0. Regularization starts from task number 1 and progressively reduces as the curriculum advances, enabling more effective parameter updating while preserving previously acquired knowledge during the training of challenging tasks. The loss function of the $t$-th task during training is
\begin{equation}
  L(\vtheta_t) = 
  \begin{cases}
  L_0(\vtheta_0), & t = 0 \\
  L_t(\vtheta_t) + \sum_{i} \lambda \mF_{\vtheta_i} ||\vtheta_{t, i} - \vtheta_{t - 1, i}||^2, & t > 0
  \end{cases}
\end{equation}
where $\vtheta_t$ is the model parameter trained during task $t$, $\vtheta_{t-1}$ is the model parameter trained after task $t - 1$, $\vtheta_0$ is the model parameter during task $0$ training, $L_t$ is the task loss of task $t$, $L_0$ is the task loss of task $0$, $\mF$ is the Fisher Information Matrix computed using the validation set after task $t - 1$ and $i$ labels the model parameters.

\subsection{Theoretical Analysis}

To assess the effectiveness of \textsf{IV-EWC}, we examine EWC regularization in curriculum learning under the setup of Theorem \ref{thm:CLforget}. Using forgetting defined in Eq. \ref{eq:Forget}, Theorem \ref{thm:EWCeffectiveness} shows that EWC-based regularization reduces forgetting and can eliminate it entirely with a suitably tuned regularization weight, providing a formal justification for \textsf{IV-EWC}.

\begin{theorem}
  For the curricula with loss functions $f_t(\vtheta) = \frac{L_t}{2}||\vtheta - \vtheta_t^*||^2$, where $\vtheta_t^*$ is the optimal parameter, $t = 1, 2, ..., T$, $L_t > L_{t - 1} > 0$, $||\vtheta_t^* - \vtheta_{t - 1}^*||^2 > 0$ and $||\vtheta_{t - 1}^*||^2 > 0$, when the training data of task $t$ follows a Gaussian distribution $\mathcal{N}(\vtheta_t^*, \frac{1}{L_t})$ and the optimization starts from $\vtheta_0 = \vzero$ with learning rate $\eta > 0$, incorporating an EWC loss calculated from data across previous tasks with the regularization weight $\lambda$ ($0 < \lambda \leq 1$) is able to alleviate forgetting.
\label{thm:EWCeffectiveness}
\end{theorem}

\begin{proof}

To better compare curriculum learning with EWC regularization and vanilla curriculum learning, we first simplify Eq. \ref{eq:Forget}. Under the specified curriculum setting, $\mathrm{Fgt}$ has a global minimum, where the optimal $\vtheta_T^{opt}$ is given by
\begin{equation}
  \vtheta_T^{opt} = \frac{\sum_{t = 1}^{T}L_t \vtheta_t^*}{\sum_{t = 1}^{T}L_t}
\end{equation}

Then forgetting of the last task in the curriculum can be simplified as follows
\begin{align}
  \mathrm{Fgt} &= \frac{1}{2T} \left( \sum_{t = 1}^{T}L_t ||\vtheta_T||^2 - 2\sum_{t = 1}^{T}L_t \langle\vtheta^{opt}_T, \vtheta_T \rangle + \sum_{t = 1}^{T}L_t ||\vtheta^*_t||^2 \right) \nonumber \\
  &= \frac{1}{2T} \sum_{t = 1}^{T}L_t || \vtheta_T - \vtheta^{opt}_T ||^2 + C' \label{eq:SimplifyForget}
\end{align}
where $\langle \cdot \rangle$ is the dot product of two vectors, $C' = \frac{1}{2T} \left( \sum_{t = 1}^{T}L_t ||\vtheta^*_t||^2 - \sum_{t = 1}^{T}L_t ||\vtheta_T^{opt}||^2 \right)$ is a constant defined by the curricula.

Under the assumption of the $t$-th task's training data distribution $\mathcal{N}(\vtheta_t^*, \frac{1}{L_t})$, the Fisher Information Matrix used in EWC can be simplified to the second derivative of the loss, which is $L_t$. Adding an EWC loss across previous tasks to the original task yields the following new loss function for task $t$:
\begin{equation}
  f_t^{EWC}(\vtheta) = \frac{L_t}{2}||\vtheta - \vtheta_t^*||^2 + \frac{\lambda}{2} \sum_{t' = 1}^{t - 1} \left( L_{t'} ||\vtheta - \vtheta^*_{L_{t'}}||^2 \right)
\end{equation}

The loss of the last task in the curriculum with EWC is $f_T^{EWC}(\vtheta)$, setting its derivative to zero gives
\begin{equation}
  \vtheta_T^{EWC} = \frac{L_T \vtheta_T^* + \lambda \sum_{t' = 1}^{T - 1}L_{t'}\vtheta_{t'}^*}{L_T + \lambda \sum_{t' = 1}^{T - 1}L_{t'}}
\end{equation}

Forgetting of the last task with EWC regularization is
\begin{equation}
  \mathrm{Fgt}^{EWC} = \frac{1}{T}\sum_{t = 1}^{T} \left( \frac{L_t}{2} ||\vtheta_T^{EWC} - \vtheta_t^*||^2 \right) \label{eq:ForgetEWC}
\end{equation}

Comparing Eq. \ref{eq:ForgetVanilla} with Eq. \ref{eq:ForgetEWC}, let $\Delta = \mathrm{Fgt}^{vanilla} - \mathrm{Fgt}^{EWC}$. Using Eq. \ref{eq:SimplifyForget} to simplify $\Delta$ we have
\begin{equation}
  \Delta \approx \frac{1}{2T} \sum_{t = 1}^{T}L_t \left( ||\vtheta_T^* - \vtheta^{opt}_T||^2 - ||\vtheta_T^{EWC} - \vtheta^{opt}_T||^2 \right) \label{eq:ForgetDeltaSimple}
\end{equation}

When $0 < \lambda \leq 1$, $\Delta > 0$ is granted, since $||\vtheta_T^{EWC} - \vtheta_T^{opt}||^2 \geq 0$ and $||\vtheta_T^* - \vtheta_T^{EWC}||^2 > 0 $. Curriculum learning using EWC demonstrates reduced forgetting. If $\lambda$ is appropriately chosen, specifically $\lambda = 1$, forgetting can be reduced to zero.

\end{proof}

\section{Experiments}

In this section we present empirical evidence that substantiates the preceding theoretical analysis of forgetting in curriculum learning. We further benchmark \textsf{IV-EWC} against four continual learning methods designed to mitigate forgetting in this setting.

\subsection{Experimental Settings}

\begin{table}[t]
\caption{Evaluation on three base models across three math reasoning tasks. \textbf{Bold} and \underline{underlined} values indicate the \textbf{best} and \underline{second-best} results, respectively.}
\label{tb:Compare}
\begin{center}
\small
\begin{tabularx}{\textwidth}{@{}>{\hsize=1.96\hsize}X|>{\hsize=0.84\hsize\Centering}X>{\hsize=0.84\hsize\Centering}X|>{\hsize=0.84\hsize\Centering}X>{\hsize=0.84\hsize\Centering}X|>{\hsize=0.84\hsize\Centering}X>{\hsize=0.84\hsize\Centering}X@{}}
  \toprule
   & \multicolumn{2}{c|}{GSM8K} & \multicolumn{2}{c|}{MATH} & \multicolumn{2}{c}{cn\_k12} \\
   & \multicolumn{1}{c}{Forget} & \multicolumn{1}{c|}{ACC} & \multicolumn{1}{c}{Forget} & \multicolumn{1}{c|}{ACC} & \multicolumn{1}{c}{Forget} & \multicolumn{1}{c}{ACC} \\ \midrule
  Qwen2.5-3B-Instruct &  0.0106 & 0.8084 &  0.0238 & 0.7191 &  0.0146 & 0.4421 \\
  + Random Replay     &  0.0153 & 0.7989 &  0.0496 & 0.6950 &  0.0137 & 0.4380 \\
  + EWC               &  0.0050 & \underline{0.8258} & -0.0031 & \underline{0.7406} &  0.0172 & \underline{0.4435} \\
  + A-GEM             &  0.0067 & 0.8054 &  0.0011 & 0.7203 &  0.0221 & 0.4377 \\
  + UPGD              & \underline{-0.0008} & 0.8127 & \underline{-0.0037} & 0.7366 &  \underline{0.0060} & 0.4422 \\
  + \textsf{IV-EWC}   & \textbf{-0.0103} & \textbf{0.8298} & \textbf{-0.0075} & \textbf{0.7461} & \textbf{-0.0095} & \textbf{0.4446} \\ \midrule
  Qwen2.5-7B-Instruct &  0.0062 & 0.8998 &  0.0114 & 0.8389 &  0.0157 & 0.4907 \\
  + Random Replay     &  0.0091 & 0.9025 &  0.0141 & 0.8410 &  0.0064 & 0.4944 \\
  + EWC               &  \underline{0.0021} & 0.9018 &  \underline{0.0053} & \underline{0.8417} & \underline{-0.0001} & 0.4950 \\
  + A-GEM             &  0.0036 & 0.8828 &  0.0224 & 0.8116 &  0.0093 & 0.4930 \\
  + UPGD              &  0.0054 & \textbf{0.9081} &  0.0108 & 0.8265 &  0.0006 & \underline{0.4991} \\
  + \textsf{IV-EWC}   & \textbf{-0.0066} & \underline{0.9042} & \textbf{-0.0028} & \textbf{0.8464} & \textbf{-0.0113} & \textbf{0.5005} \\ \midrule
  Llama-3.2-3B-Instruct &  0.0047 & \textbf{0.7483} &  0.0183 & 0.5446 &  0.0147 & 0.3532 \\
  + Random Replay     &  0.0067 & 0.7454 &  0.0293 & 0.5346 &  0.0221 & 0.3502 \\
  + EWC               &  0.0058 & 0.7434 &  \underline{0.0026} & \underline{0.5471} & 0.0244 & \underline{0.3538} \\
  + A-GEM             &  \underline{0.0014} & 0.7398 &  0.0081 & 0.5316 &  \underline{0.0143} & 0.3530 \\
  + UPGD              &  0.0030 & 0.7445 &  0.0111 & 0.5459 &  0.0193 & 0.3516 \\
  + \textsf{IV-EWC}   & \textbf{-0.0028} & \underline{0.7465} & \textbf{-0.0007} & \textbf{0.5512} & \textbf{-0.0141} & \textbf{0.3610} \\ \bottomrule
\end{tabularx}
\end{center}
\end{table}

\textbf{Datasets}.
Following previous works \citep{Parashar2026E2HReasoner, Xi2024ReverseCurriculum}, we selected two widely used math reasoning datasets: GSM8K \citep{Cobbe2021GSM8KDataset} and MATH \citep{Hendrycks2021MATHDataset}. We used the classified version of GSM8K created by \citet{Parashar2026E2HReasoner}, where the difficulties of questions are evaluated by zero-shot performance. We also created a cn\_k12 dataset from the NuminaMath-CoT \citep{LI2024NuminaMath} dataset to test the model's ability to solve problems with diverse formats. Detailed dataset descriptions are provided in Appendix \ref{sec:Dataset}.

\textbf{Metrics}.
Following \citet{Jiang2025FunctionVector}, we adopted the following metrics: (1) $\boldsymbol{\mathrm{FP}} = \frac{1}{T} \sum_{j = 1}^{T} a_T^j$, which is the final average performance across all tasks learned by the model, $T$ denotes the index of the last learned task and $a_T^j$ refers to the accuracy on the test set of task $j$ evaluated after learning task $T$. (2) $\boldsymbol{\mathrm{AP}} = \frac{1}{T} \sum_{j = 1}^{T} a_{j}^j$, which summarizes the average performance measured immediately after the learning of each task. (3) $\boldsymbol{\mathrm{Forget}} = \frac{FP - AP}{AP}$, which is the relative divergence between $\mathrm{FP}$ and $\mathrm{AP}$, evaluating overall performance differences between tasks. (4) $\boldsymbol{\mathrm{ACC}}$, which represents the overall accuracy of the model on the whole test set spanning all curriculum tasks, unlike $\mathrm{FP}$, it accounts for the uneven distribution of test sample difficulty that can arise in certain datasets.

\textbf{Baselines}.
To investigate the effectiveness of \textsf{IV-EWC}, we compared it with five baselines: (1) \textbf{vanilla curriculum}, where the model is trained directly on a sequence of curriculum tasks with the given learning rate. (2) \textbf{random replay}, where the training data includes a fixed number of selected samples from the training data of the previous task starting from the second task. (3) \textbf{EWC} \citep{Kirkpatrick2017EWC}, where an EWC loss is added to the original loss function and it uses the same regularization weight as \textsf{IV-EWC}. (4) \textbf{A-GEM} \citep{Chaudhry2019AGEM}, which utilizes reference samples selected from the training set of the previous task to control gradient updates; the number of reference samples is the same as random replay. (5) \textbf{UPGD} \citep{Elsayed2024UPGD}, which uses utility to measure the significance of each parameter and instruct the learning process.

\textbf{Implementation Details}.
In the experiments we adopted Qwen2.5-3B-Instruct, Qwen2.5-7B-Instruct and Llama-3.2-3B-Instruct as base models. The hyperparameter settings in the experiments are detailed in Appendix \ref{sec:HyperParameter}. All algorithms were implemented in Python and executed using the Python 3.12 interpreter. The 3B model experiments were carried out on one NVIDIA RTX PRO 6000 and the 7B model experiments were carried out on two NVIDIA RTX PRO 6000s. To ensure reliability, experiment data are presented as the mean of four independent runs.

\subsection{Comparison Study}

Table \ref{tb:Compare} summarizes performance for three base models over three tasks. The results indicate forgetting within the curriculum learning process, with greater model capacity corresponding to less forgetting and higher accuracy. Several approaches exhibit negative forgetting, where the performance on previously learned, easier tasks improves following training on more challenging tasks. This phenomenon is partly attributable to the inherent representation sharing in large language models and to the strong inter-level task relatedness induced by curriculum learning. \textsf{IV-EWC} achieves superior performance relative to all baselines on both forgetting and overall accuracy in nearly all tested configurations, with more substantial gains in smaller models. Negative forgetting is also observed in all scenarios using \textsf{IV-EWC}, confirming its effectiveness.

\begin{figure}[t]
\centering
\begin{subfigure}[b]{0.329\textwidth}
  \centering
  \includegraphics[height=3.07cm]{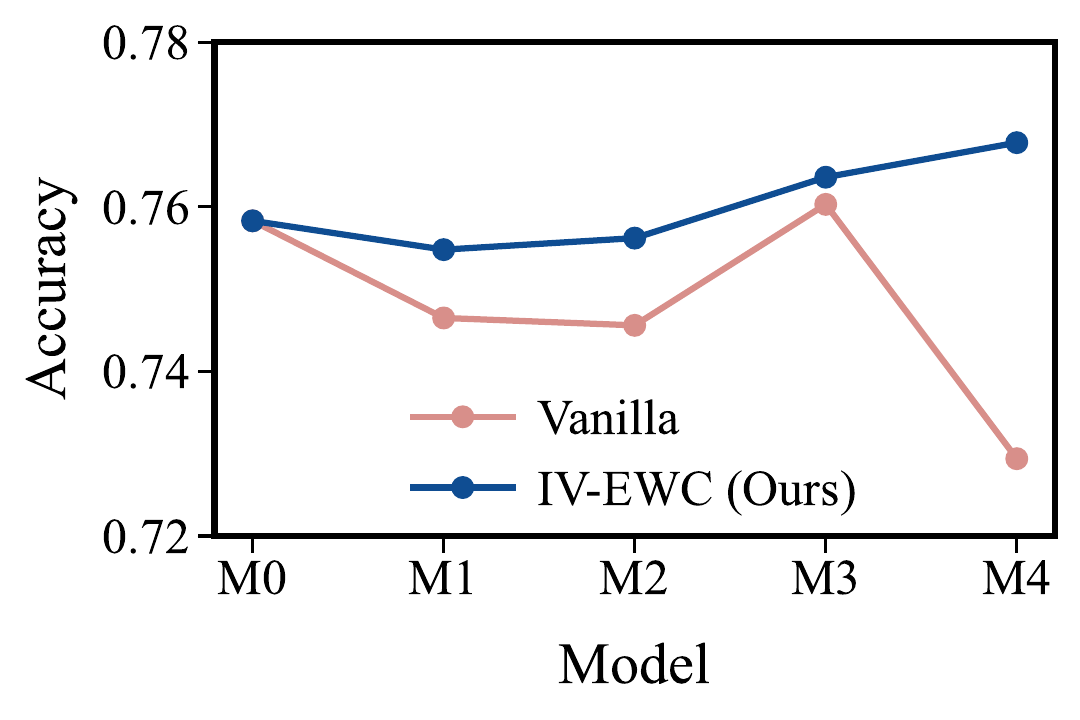}
  \caption{Easy Task Accuracy}
\end{subfigure}
\hfill
\begin{subfigure}[b]{0.329\textwidth}
  \centering
  \includegraphics[height=3.07cm]{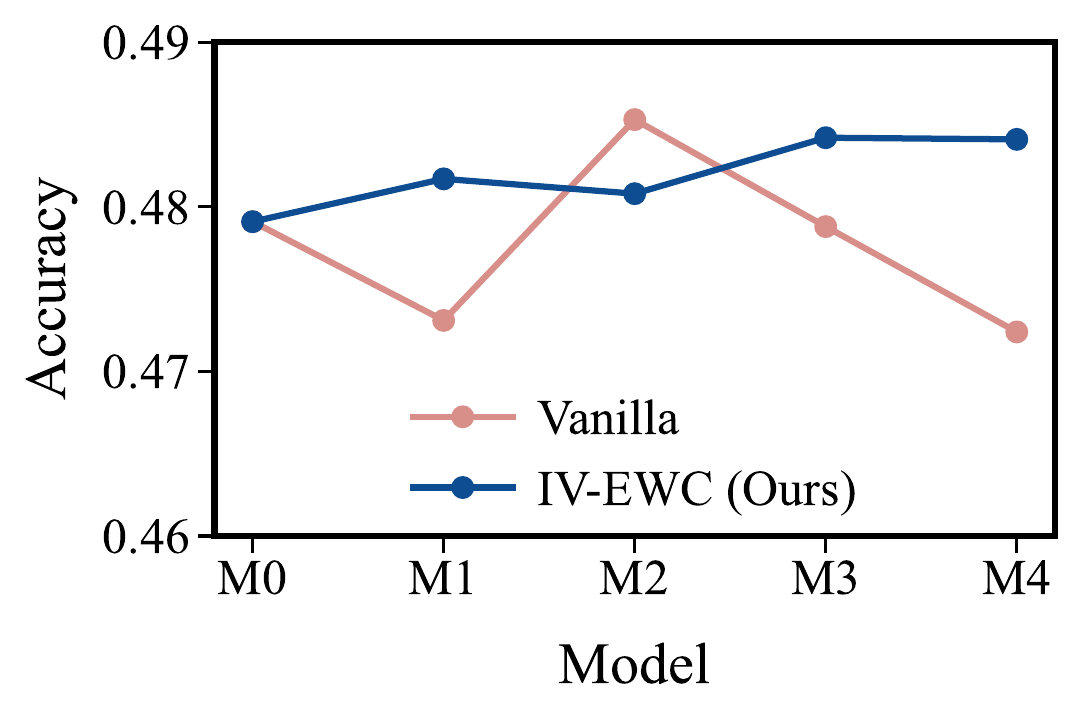}
  \caption{Hard Task Accuracy}
\end{subfigure}
\hfill
\begin{subfigure}[b]{0.329\textwidth}
  \centering
  \includegraphics[height=3.07cm]{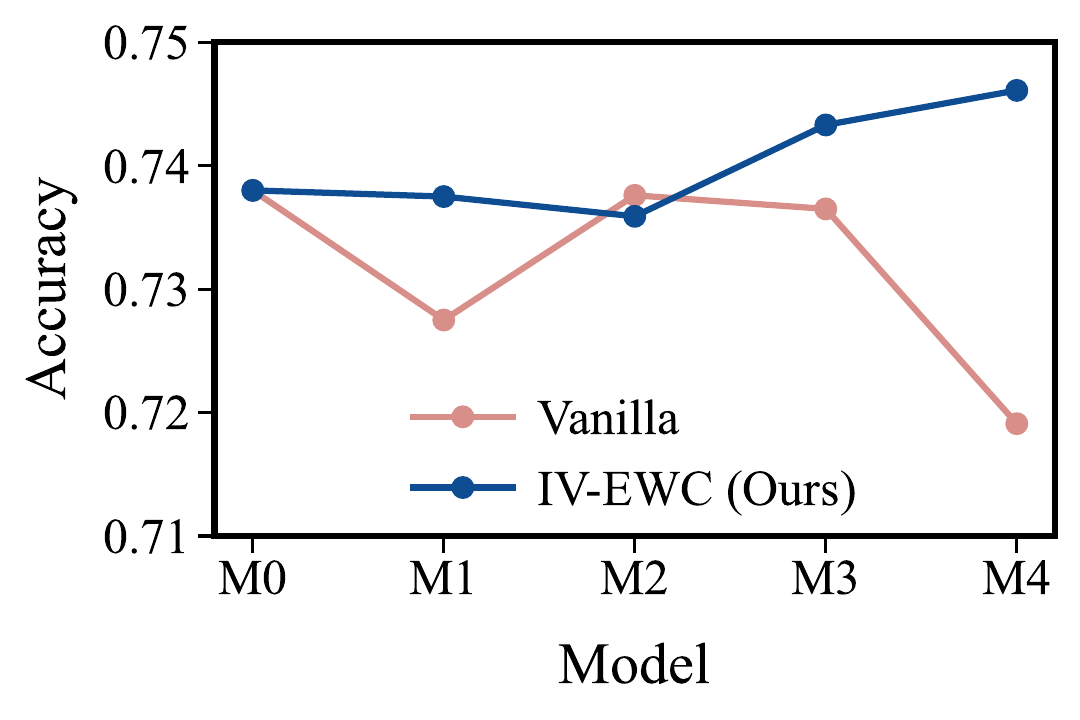}
  \caption{Overall Accuracy}
\end{subfigure}
\caption{Accuracy of models trained after each task using vanilla curriculum and \textsf{IV-EWC} on the MATH dataset based on Qwen2.5-3B-Instruct. Easy task accuracy denotes the average accuracy of level 2 or lower, hard task accuracy is the average accuracy of level 3 and higher.}
\label{fg:TrainingProcess}
\end{figure}

To better understand the curriculum learning process, we collected evaluation data on the test sets from models trained after each task. Figure \ref{fg:TrainingProcess} shows the grouped accuracy of vanilla curriculum and \textsf{IV-EWC} on the MATH dataset with Qwen2.5-3B-Instruct. While accuracy on the easy tasks deteriorates during vanilla curriculum learning, \textsf{IV-EWC} preserves performance on these tasks. On challenging tasks, vanilla curriculum learning exhibits pronounced fluctuations in accuracy, whereas \textsf{IV-EWC} maintains stable performance with a slight upward trend. Contrary to intuition, training on the hard task does not consistently yield accuracy gains on that task. We hypothesize that this pattern arises from the complex optimization dynamics of large language models and nontrivial inter-task dependencies. In comparison with the vanilla baseline, two key advantages are observed for \textsf{IV-EWC}: improved learning of new tasks, manifested as increasing accuracy on the hard tasks, and diminished forgetting, reflected in stable performance on easy tasks. More details about the performance during curriculum learning can be found in Appendix \ref{sec:CompareExp}.

\subsection{Ablation Study}

\begin{table}[t]
\caption{Evaluation of the replay method with varying replay data using Qwen2.5-3B-Instruct.}
\centering
\small
\begin{tabularx}{\textwidth}{@{}>{\hsize=1.96\hsize}X|>{\hsize=0.84\hsize\Centering}X>{\hsize=0.84\hsize\Centering}X|>{\hsize=0.84\hsize\Centering}X>{\hsize=0.84\hsize\Centering}X|>{\hsize=0.84\hsize\Centering}X>{\hsize=0.84\hsize\Centering}X@{}}
  \toprule
   & \multicolumn{2}{c|}{GSM8K} & \multicolumn{2}{c|}{MATH} & \multicolumn{2}{c}{cn\_k12} \\
   & \multicolumn{1}{c}{Forget} & \multicolumn{1}{c|}{ACC} & \multicolumn{1}{c}{Forget} & \multicolumn{1}{c|}{ACC} & \multicolumn{1}{c}{Forget} & \multicolumn{1}{c}{ACC} \\ \midrule
  Vanilla Replay & 0.0088 & 0.8223 & 0.0121 & 0.7375 & 0.0143 & 0.4785 \\
  + Select Across Levels & 0.0042 & 0.8336 & 0.0105 & 0.7359 & 0.0091 & 0.4750 \\
  + Influence Function & 0.0026 & 0.8336 & 0.0058 & 0.7360 & 0.0075 & 0.4775 \\ \bottomrule
\end{tabularx}
\label{tb:Ablation}
\end{table}

To assess the effectiveness of the validation set used in \textsf{IV-EWC}, an ablation experiment was conducted. This experiment used the replay method with Qwen2.5-3B-Instruct. To avoid overfitting, the learning rate was set to 3e-8, 5.5e-8 and 1e-7 for GSM8K, MATH and cn\_k12 datasets respectively. The replay sample count is incremented by 100 with each subsequent task, yielding 100, 200, 300, and 400 samples for Tasks 1 to 4, respectively. Vanilla replay is the baseline method, which randomly selects 100 samples from task $t - 1$ as replay samples for task $t$; replay samples are repeated to match the size of the training set. Select across levels method keeps the previously selected replay samples and adds 100 random samples of the previous task into replay before the current task. Influence function method directly uses the validation set of the corresponding task in \textsf{IV-EWC} for replay. The experiment results are illustrated in Table \ref{tb:Ablation}.

Keeping replay samples across levels decreases forgetting by roughly 0.0038 (34.0\% of vanilla replay) on average. On this basis, using the validation set as replay samples can further decrease forgetting by about 0.0026 (22.7\% of vanilla replay) on average. Results indicate that replaying an aggregated set of training examples drawn across curriculum tasks, alongside replaying examples selected via influence functions, better preserves knowledge acquired in earlier tasks and thereby reduces forgetting; employing both strategies in combination yields the strongest performance.

\subsection{Case Study}

\begin{figure}[t]
\begin{center}
\includegraphics[height=6cm]{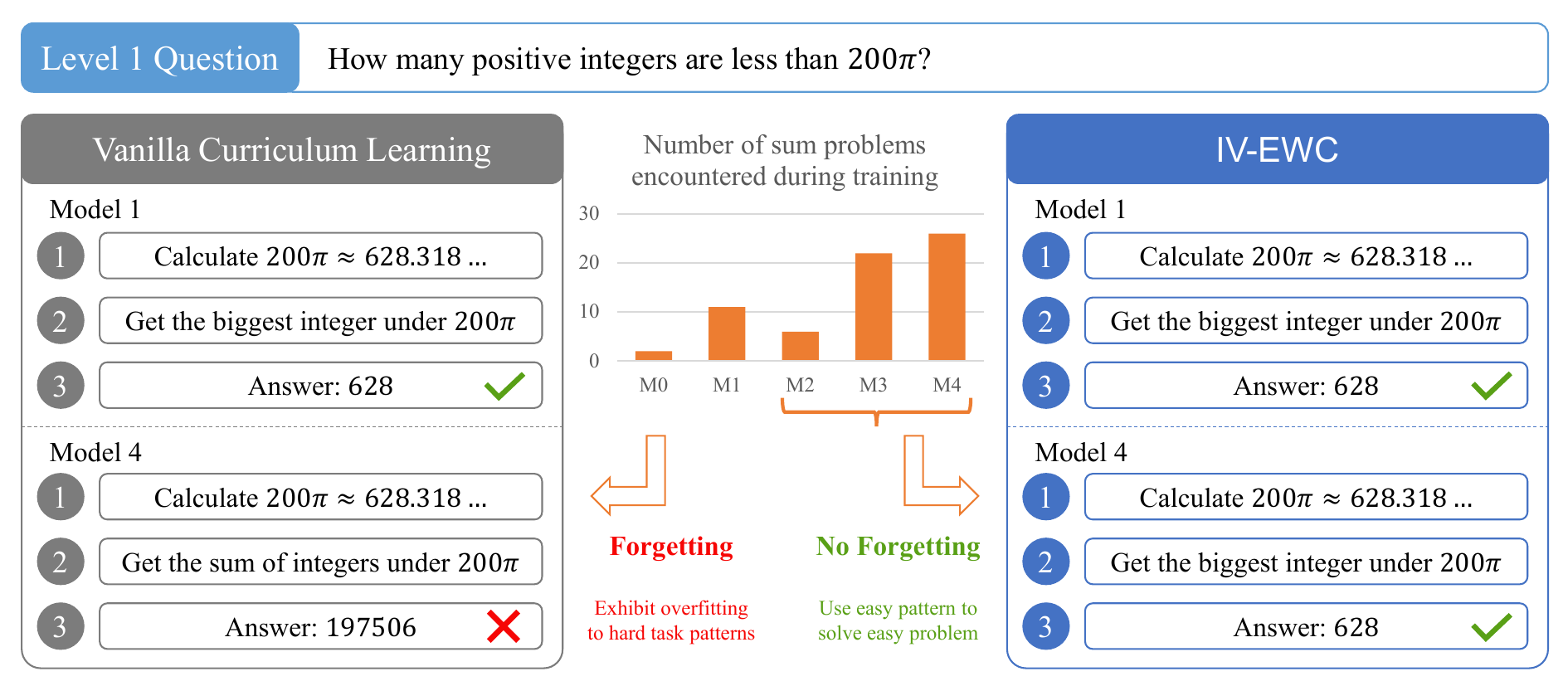}
\end{center}
\caption{Case study on a MATH test example comparing Qwen2.5-3B-Instruct models trained via vanilla curriculum learning and via \textsf{IV-EWC}. Following training on tasks from Level 2 to Level 4, the vanilla curriculum model shows overfitting to hard-task patterns, while the model trained with \textsf{IV-EWC} avoids such overfitting.}
\label{fg:CaseStudy}
\end{figure}

To examine forgetting in curriculum learning and assess how \textsf{IV-EWC} mitigates it, we collected responses of models at different training stages. Figure \ref{fg:CaseStudy} presents side-by-side the responses of the Qwen2.5-3B-Instruct model, trained on MATH with vanilla curriculum learning versus \textsf{IV-EWC}, for a representative problem from the MATH test set. After training on the Level 1 task, both models acquire the appropriate strategy for solving the task. After learning on the task with Level 4 data, the model trained with vanilla curriculum learning erroneously applies the heuristic for sum problems to the example, yielding an incorrect prediction. This misgeneralization is partly attributable to the training data distribution: during the training on tasks from Level 2 to Level 4, the model encounters more sum problems. Because vanilla curriculum learning introduces no explicit regularization, the model overfits to patterns characteristic of the more difficult tasks. In contrast, the model trained via \textsf{IV-EWC} resolves the task using the task-relevant pattern, suggesting that regularization limits overfitting to the solution approaches of challenging curriculum tasks and thus mitigates forgetting. Additional cases from the study are provided in Appendix \ref{sec:CaseStudy}.

\section{Conclusion}
In this paper we employ both theoretical and empirical approaches to analyze the phenomenon of catastrophic forgetting in curriculum learning and propose a novel framework \textsf{IV-EWC}, which incorporates EWC regularization into curriculum learning. \textsf{IV-EWC} utilizes the influence function to select training samples, forming a validation set for regularization. A dynamic regularization mechanism is further applied to facilitate effective learning. Theoretical investigation of EWC-based methods provides a foundation for \textsf{IV-EWC} in eliminating forgetting within curriculum learning. Empirical results under multiple settings demonstrate that \textsf{IV-EWC} consistently outperforms baselines, highlighting its effectiveness.

\bibliography{arxiv_reference}

@inproceedings{Bengio2009CurriculumLearning,
  author = {Bengio, Yoshua and Louradour, J\'{e}r\^{o}me and Collobert, Ronan and Weston, Jason},
  title = {Curriculum learning},
  year = {2009},
  isbn = {9781605585161},
  publisher = {Association for Computing Machinery},
  address = {New York, NY, USA},
  url = {https://doi.org/10.1145/1553374.1553380},
  doi = {10.1145/1553374.1553380},
  booktitle = {Proceedings of the 26th Annual International Conference on Machine Learning},
  pages = {41-48},
  numpages = {8},
  location = {Montreal, Quebec, Canada},
  series = {ICML '09}
}

@inproceedings{Hacohen2019OnThePowerOfCL,
  title={On The Power of Curriculum Learning in Training Deep Networks},
  author={Guy Hacohen and Daphna Weinshall},
  booktitle={International Conference on Machine Learning},
  year={2019}
}

@article{Rafailov2023DPO,
  title={Direct preference optimization: Your language model is secretly a reward model},
  author={Rafailov, Rafael and Sharma, Archit and Mitchell, Eric and Manning, Christopher D and Ermon, Stefano and Finn, Chelsea},
  journal={Advances in neural information processing systems},
  volume={36},
  pages={53728--53741},
  year={2023}
}

@misc{Shao2024GRPO,
  title={DeepSeekMath: Pushing the Limits of Mathematical Reasoning in Open Language Models}, 
  author={Zhihong Shao and Peiyi Wang and Qihao Zhu and Runxin Xu and Junxiao Song and Xiao Bi and Haowei Zhang and Mingchuan Zhang and Y. K. Li and Y. Wu and Daya Guo},
  year={2024},
  eprint={2402.03300},
  archivePrefix={arXiv},
  primaryClass={cs.CL},
  url={https://arxiv.org/abs/2402.03300}, 
}

@inproceedings{Parashar2026E2HReasoner,
  title={Curriculum Reinforcement Learning from Easy to Hard Tasks Improves LLM Reasoning},
  author={Parashar, Shubham and Gui, Shurui and Li, Xiner and Ling, Hongyi and Vemuri, Sushil and Olson, Blake and Li, Eric and Zhang, Yu and Caverlee, James and Kalathil, Dileep and others},
  booktitle={International Conference on Learning Representations},
  volume={2026},
  pages={156944--156972},
  year={2026}
}

@inproceedings{Xi2024ReverseCurriculum,
  author = {Xi, Zhiheng and Chen, Wenxiang and Hong, Boyang and Jin, Senjie and Zheng, Rui and He, Wei and Ding, Yiwen and Liu, Shichun and Guo, Xin and Wang, Junzhe and Guo, Honglin and Shen, Wei and Fan, Xiaoran and Zhou, Yuhao and Dou, Shihan and Wang, Xiao and Zhang, Xinbo and Sun, Peng and Gui, Tao and Zhang, Qi and Huang, Xuanjing},
  title = {Training large language models for reasoning through reverse curriculum reinforcement learning},
  year = {2024},
  publisher = {JMLR.org},
  booktitle = {Proceedings of the 41st International Conference on Machine Learning},
  articleno = {2217},
  numpages = {19},
  location = {Vienna, Austria},
  series = {ICML'24}
}

@inproceedings{Wen2025Light-R1,
    title = "Light-R1: Curriculum {SFT}, {DPO} and {RL} for Long {COT} from Scratch and Beyond",
    author = "Wen, Liang  and
      Cai, Yunke  and
      Xiao, Fenrui  and
      He, Xin  and
      An, Qi  and
      Duan, Zhenyu  and
      Du, Yimin  and
      Liu, Junchen  and
      Tang, Lifu  and
      Lv, Xiaowei  and
      Zou, Haosheng  and
      Deng, Yongchao  and
      Jia, Shousheng  and
      Zhang, Xiangzheng",
    editor = "Rehm, Georg  and
      Li, Yunyao",
    booktitle = "Proceedings of the 63rd Annual Meeting of the Association for Computational Linguistics (Volume 6: Industry Track)",
    month = jul,
    year = "2025",
    address = "Vienna, Austria",
    publisher = "Association for Computational Linguistics",
    url = "https://aclanthology.org/2025.acl-industry.24/",
    doi = "10.18653/v1/2025.acl-industry.24",
    pages = "318--327",
    ISBN = "979-8-89176-288-6",
}

@article{Xiao2024CLIP-VG,
  author={Xiao, Linhui and Yang, Xiaoshan and Peng, Fang and Yan, Ming and Wang, Yaowei and Xu, Changsheng},
  journal={IEEE Transactions on Multimedia}, 
  title={CLIP-VG: Self-Paced Curriculum Adapting of CLIP for Visual Grounding}, 
  year={2024},
  volume={26},
  number={},
  pages={4334-4347},
  doi={10.1109/TMM.2023.3321501}
}

@inproceedings{Liu2021CurriculumForMedicalReport,
    title = "Competence-based Multimodal Curriculum Learning for Medical Report Generation",
    author = "Liu, Fenglin  and
      Ge, Shen  and
      Wu, Xian",
    editor = "Zong, Chengqing  and
      Xia, Fei  and
      Li, Wenjie  and
      Navigli, Roberto",
    booktitle = "Proceedings of the 59th Annual Meeting of the Association for Computational Linguistics and the 11th International Joint Conference on Natural Language Processing (Volume 1: Long Papers)",
    month = aug,
    year = "2021",
    address = "Online",
    publisher = "Association for Computational Linguistics",
    url = "https://aclanthology.org/2021.acl-long.234/",
    doi = "10.18653/v1/2021.acl-long.234",
    pages = "3001--3012",
}

@inproceedings{Zhang2019CurriculumForTranslation,
    title = "Curriculum Learning for Domain Adaptation in Neural Machine Translation",
    author = "Zhang, Xuan  and
      Shapiro, Pamela  and
      Kumar, Gaurav  and
      McNamee, Paul  and
      Carpuat, Marine  and
      Duh, Kevin",
    editor = "Burstein, Jill  and
      Doran, Christy  and
      Solorio, Thamar",
    booktitle = "Proceedings of the 2019 Conference of the North {A}merican Chapter of the Association for Computational Linguistics: Human Language Technologies, Volume 1 (Long and Short Papers)",
    month = jun,
    year = "2019",
    address = "Minneapolis, Minnesota",
    publisher = "Association for Computational Linguistics",
    url = "https://aclanthology.org/N19-1189/",
    doi = "10.18653/v1/N19-1189",
    pages = "1903--1915",
}

@inproceedings{Platanios2019CompetenceCLForTranslation,
    title = "Competence-based Curriculum Learning for Neural Machine Translation",
    author = "Platanios, Emmanouil Antonios  and
      Stretcu, Otilia  and
      Neubig, Graham  and
      Poczos, Barnabas  and
      Mitchell, Tom",
    editor = "Burstein, Jill  and
      Doran, Christy  and
      Solorio, Thamar",
    booktitle = "Proceedings of the 2019 Conference of the North {A}merican Chapter of the Association for Computational Linguistics: Human Language Technologies, Volume 1 (Long and Short Papers)",
    month = jun,
    year = "2019",
    address = "Minneapolis, Minnesota",
    publisher = "Association for Computational Linguistics",
    url = "https://aclanthology.org/N19-1119/",
    doi = "10.18653/v1/N19-1119",
    pages = "1162--1172",
}

@inproceedings{Liu2020NormCLForTranslation,
    title = "Norm-Based Curriculum Learning for Neural Machine Translation",
    author = "Liu, Xuebo  and
      Lai, Houtim  and
      Wong, Derek F.  and
      Chao, Lidia S.",
    editor = "Jurafsky, Dan  and
      Chai, Joyce  and
      Schluter, Natalie  and
      Tetreault, Joel",
    booktitle = "Proceedings of the 58th Annual Meeting of the Association for Computational Linguistics",
    month = jul,
    year = "2020",
    address = "Online",
    publisher = "Association for Computational Linguistics",
    url = "https://aclanthology.org/2020.acl-main.41/",
    doi = "10.18653/v1/2020.acl-main.41",
    pages = "427--436",
}

@inproceedings{Leyendecker2021CurriculumForRobot,
  author={Leyendecker, Lars and Schmitz, Markus and Zhou, Hans Aoyang and Samsonov, Vladimir and Rittstieg, Marius and Lütticke, Daniel},
  booktitle={2021 Fifth IEEE International Conference on Robotic Computing (IRC)}, 
  title={Deep Reinforcement Learning for Robotic Control in High-Dexterity Assembly Tasks - A Reward Curriculum Approach}, 
  year={2021},
  volume={},
  number={},
  pages={35-42},
  doi={10.1109/IRC52146.2021.00012}
}

@inproceedings{Zhang2025CurriculumForBiological,
  author = {Zhang, Xiang and Wei, Jiaqi and Qiu, Zijie and Xu, Sheng and Dong, Nanqing and Gao, Zhiqiang and Sun, Siqi},
  title = {Curriculum learning for biological sequence prediction: the case of de novo peptide sequencing},
  year = {2025},
  publisher = {JMLR.org},
  booktitle = {Proceedings of the 42nd International Conference on Machine Learning},
  articleno = {3062},
  numpages = {19},
  location = {Vancouver, Canada},
  series = {ICML'25}
}

@inproceedings{Guo2018CurriculumNet,
  author="Guo, Sheng and Huang, Weilin and Zhang, Haozhi and Zhuang, Chenfan and Dong, Dengke and Scott, Matthew R. and Huang, Dinglong",
  editor="Ferrari, Vittorio and Hebert, Martial and Sminchisescu, Cristian and Weiss, Yair",
  title="CurriculumNet: Weakly Supervised Learning from Large-Scale Web Images",
  booktitle="Computer Vision -- ECCV 2018",
  year="2018",
  publisher="Springer International Publishing",
  address="Cham",
  pages="139--154",
  isbn="978-3-030-01249-6"
}

@inproceedings{Kumar2010SelfPaced,
author = {Kumar, M. Pawan and Packer, Benjamin and Koller, Daphne},
title = {Self-paced learning for latent variable models},
year = {2010},
publisher = {Curran Associates Inc.},
address = {Red Hook, NY, USA},
booktitle = {Proceedings of the 24th International Conference on Neural Information Processing Systems - Volume 1},
pages = {1189-1197},
numpages = {9},
location = {Vancouver, British Columbia, Canada},
series = {NIPS'10}
}

@inproceedings{Jiang2018MentorNet,
  title = {{M}entor{N}et: Learning Data-Driven Curriculum for Very Deep Neural Networks on Corrupted Labels},
  author = {Jiang, Lu and Zhou, Zhengyuan and Leung, Thomas and Li, Li-Jia and Fei-Fei, Li},
  booktitle = {Proceedings of the 35th International Conference on Machine Learning},
  pages = {2304--2313},
  year = {2018},
  editor = {Dy, Jennifer and Krause, Andreas},
  volume = {80},
  series = {Proceedings of Machine Learning Research},
  month = {10--15 Jul},
  publisher = {PMLR},
  url = {https://proceedings.mlr.press/v80/jiang18c.html},
}

@inproceedings{Florensa2017ReverseCLForRobot,
  title = {Reverse Curriculum Generation for Reinforcement Learning},
  author = {Florensa, Carlos and Held, David and Wulfmeier, Markus and Zhang, Michael and Abbeel, Pieter},
  booktitle = {Proceedings of the 1st Annual Conference on Robot Learning},
  pages = {482--495},
  year = {2017},
  editor = {Levine, Sergey and Vanhoucke, Vincent and Goldberg, Ken},
  volume = {78},
  series = {Proceedings of Machine Learning Research},
  month = {13--15 Nov},
  publisher = {PMLR},
  url = {https://proceedings.mlr.press/v78/florensa17a.html},
}

@inproceedings{KocmiBojar2017CurriculumNMT,
    title = "Curriculum Learning and Minibatch Bucketing in Neural Machine Translation",
    author = "Kocmi, Tom  and
      Bojar, Ond{\v{r}}ej",
    editor = "Mitkov, Ruslan  and
      Angelova, Galia",
    booktitle = "Proceedings of the International Conference Recent Advances in Natural Language Processing, {RANLP} 2017",
    month = sep,
    year = "2017",
    address = "Varna, Bulgaria",
    publisher = "INCOMA Ltd.",
    url = "https://aclanthology.org/R17-1050/",
    doi = "10.26615/978-954-452-049-6_050",
    pages = "379--386",
}

@inproceedings{Graves2017AutomatedCL,
author = {Graves, Alex and Bellemare, Marc G. and Menick, Jacob and Munos, R{\'e}mi and Kavukcuoglu, Koray},
title = {Automated curriculum learning for neural networks},
year = {2017},
publisher = {JMLR.org},
booktitle = {Proceedings of the 34th International Conference on Machine Learning - Volume 70},
pages = {1311-1320},
numpages = {10},
location = {Sydney, NSW, Australia},
series = {ICML'17}
}

@article{Tambet2020TeacherStudentCL,
  author={Matiisen, Tambet and Oliver, Avital and Cohen, Taco and Schulman, John},
  journal={IEEE Transactions on Neural Networks and Learning Systems}, 
  title={Teacher-Student Curriculum Learning}, 
  year={2020},
  volume={31},
  number={9},
  pages={3732-3740},
  doi={10.1109/TNNLS.2019.2934906}
}

@misc{Wang2019POET,
  title={Paired Open-Ended Trailblazer (POET): Endlessly Generating Increasingly Complex and Diverse Learning Environments and Their Solutions}, 
  author={Rui Wang and Joel Lehman and Jeff Clune and Kenneth O. Stanley},
  year={2019},
  eprint={1901.01753},
  archivePrefix={arXiv},
  primaryClass={cs.NE},
  url={https://arxiv.org/abs/1901.01753}, 
}

@inproceedings{Wang2020EnhancedPOET,
author = {Wang, Rui and Lehman, Joel and Rawal, Aditya and Zhi, Jiale and Li, Yulun and Clune, Jeff and Stanley, Kenneth O.},
title = {Enhanced POET: open-ended reinforcement learning through unbounded invention of learning challenges and their solutions},
year = {2020},
publisher = {JMLR.org},
booktitle = {Proceedings of the 37th International Conference on Machine Learning},
articleno = {922},
numpages = {12},
series = {ICML'20}
}

@article{French1999CatastrophicForgetting,
    title = {Catastrophic forgetting in connectionist networks},
    journal = {Trends in Cognitive Sciences},
    volume = {3},
    number = {4},
    pages = {128-135},
    year = {1999},
    issn = {1364-6613},
    doi = {https://doi.org/10.1016/S1364-6613(99)01294-2},
    url = {https://www.sciencedirect.com/science/article/pii/S1364661399012942},
    author = {Robert M. French}
}

@inproceedings{Thrun1995ContinualLearning,
author = {Thrun, Sebastian},
title = {Is Learning The n-th Thing Any Easier Than Learning The First?},
year = {1995},
publisher = {MIT Press},
address = {Cambridge, MA, USA},
booktitle = {Proceedings of the 9th International Conference on Neural Information Processing Systems},
pages = {640-646},
numpages = {7},
location = {Denver, Colorado},
series = {NIPS'95}
}

@incollection{McCloskey1989CatastrophicInterference,
  title = {Catastrophic Interference in Connectionist Networks: The Sequential Learning Problem},
  editor = {Gordon H. Bower},
  booktitle = {Psychology of Learning and Motivation},
  publisher = {Academic Press},
  volume = {24},
  pages = {109-165},
  year = {1989},
  issn = {0079-7421},
  doi = {https://doi.org/10.1016/S0079-7421(08)60536-8},
  url = {https://www.sciencedirect.com/science/article/pii/S0079742108605368},
  author = {Michael McCloskey and Neal J. Cohen}
}

@inproceedings{Hetherington2014IsThere,
  title={Is There "Catastrophic Interference" in Connectionist Networks?},
  author={Hetherington, Phil A and Seidenberg, Mark S},
  booktitle={11th Annual Conference Cognitive Science Society Pod},
  pages={26--33},
  year={2014},
  organization={Psychology Press}
}

@inproceedings{Rebuffi2017iCaRL,
  author={Rebuffi, Sylvestre-Alvise and Kolesnikov, Alexander and Sperl, Georg and Lampert, Christoph H.},
  booktitle={2017 IEEE Conference on Computer Vision and Pattern Recognition (CVPR)}, 
  title={iCaRL: Incremental Classifier and Representation Learning}, 
  year={2017},
  volume={},
  number={},
  pages={5533-5542},
  doi={10.1109/CVPR.2017.587}
}

@article{Kirkpatrick2017EWC,
  author = {James Kirkpatrick  and Razvan Pascanu  and Neil Rabinowitz  and Joel Veness  and Guillaume Desjardins  and Andrei A. Rusu  and Kieran Milan  and John Quan  and Tiago Ramalho  and Agnieszka Grabska-Barwinska  and Demis Hassabis  and Claudia Clopath  and Dharshan Kumaran  and Raia Hadsell },
  title = {Overcoming catastrophic forgetting in neural networks},
  journal = {Proceedings of the National Academy of Sciences},
  volume = {114},
  number = {13},
  pages = {3521-3526},
  year = {2017},
  doi = {10.1073/pnas.1611835114},
  URL = {https://www.pnas.org/doi/abs/10.1073/pnas1611835114},
  eprint = {https://www.pnas.org/doi/pdf/10.1073/pnas.1611835114},
}

@inproceedings{LopezPaz2017GEM,
author = {Lopez-Paz, David and Ranzato, Marc'Aurelio},
title = {Gradient episodic memory for continual learning},
year = {2017},
isbn = {9781510860964},
publisher = {Curran Associates Inc.},
address = {Red Hook, NY, USA},
booktitle = {Proceedings of the 31st International Conference on Neural Information Processing Systems},
pages = {6470-6479},
numpages = {10},
location = {Long Beach, California, USA},
series = {NIPS'17}
}

@inproceedings{Chaudhry2019AGEM,
  title={Efficient Lifelong Learning with A-GEM},
  author={Arslan Chaudhry and Marc'Aurelio Ranzato and Marcus Rohrbach and Mohamed Elhoseiny},
  booktitle={International Conference on Learning Representations},
  year={2019},
}

@inproceedings{Elsayed2024UPGD,
  title={Addressing Loss of Plasticity and Catastrophic Forgetting in Continual Learning},
  author={Mohamed Elsayed and A. Rupam Mahmood},
  booktitle={International Conference on Learning Representations},
  year={2024},
}

@inproceedings{Shin2017DGR,
author = {Shin, Hanul and Lee, Jung Kwon and Kim, Jaehong and Kim, Jiwon},
title = {Continual learning with deep generative replay},
year = {2017},
isbn = {9781510860964},
publisher = {Curran Associates Inc.},
address = {Red Hook, NY, USA},
booktitle = {Proceedings of the 31st International Conference on Neural Information Processing Systems},
pages = {2994-3003},
numpages = {10},
location = {Long Beach, California, USA},
series = {NIPS'17}
}

@inproceedings{Rolnick2019CLEAR,
author = {Rolnick, David and Ahuja, Arun and Schwarz, Jonathan and Lillicrap, Timothy P. and Wayne, Greg},
title = {Experience replay for continual learning},
year = {2019},
publisher = {Curran Associates Inc.},
address = {Red Hook, NY, USA},
booktitle = {Proceedings of the 33rd International Conference on Neural Information Processing Systems},
articleno = {32},
numpages = {11}
}

@inproceedings{Jiang2025FunctionVector,
  title={Unlocking the power of function vectors for characterizing and mitigating catastrophic forgetting in continual instruction tuning},
  author={Jiang, Gangwei and Li, Zhaoyi and Xue, Siqiao and ZHOU, JUN and Song, Linqi and Lian, Defu and Wei, Ying and others},
  booktitle={International Conference on Learning Representations},
  volume={2025},
  pages={47218--47243},
  year={2025}
}

@inproceedings{Hendrycks2021MATHDataset,
  title={Measuring Mathematical Problem Solving With the MATH Dataset},
  author={Hendrycks, Dan and Burns, Collin and Kadavath, Saurav and Arora, Akul and Basart, Steven and Tang, Eric and Song, Dawn and Steinhardt, Jacob},
  booktitle={Thirty-fifth Conference on Neural Information Processing Systems Datasets and Benchmarks Track (Round 2)},
  year={2021}
}

@misc{Cobbe2021GSM8KDataset,
  title={Training Verifiers to Solve Math Word Problems}, 
  author={Karl Cobbe and Vineet Kosaraju and Mohammad Bavarian and Mark Chen and Heewoo Jun and Lukasz Kaiser and Matthias Plappert and Jerry Tworek and Jacob Hilton and Reiichiro Nakano and Christopher Hesse and John Schulman},
  year={2021},
  eprint={2110.14168},
  archivePrefix={arXiv},
  primaryClass={cs.LG},
  url={https://arxiv.org/abs/2110.14168}, 
}

@misc{LI2024NuminaMath,
  author = {Jia Li and Edward Beeching and Lewis Tunstall and Ben Lipkin and Roman Soletskyi and Shengyi Costa Huang and Kashif Rasul and Longhui Yu and Albert Jiang and Ziju Shen and Zihan Qin and Bin Dong and Li Zhou and Yann Fleureau and Guillaume Lample and Stanislas Polu},
  title = {NuminaMath},
  year = {2024},
  publisher = {Numina},
  journal = {Hugging Face repository},
  howpublished = {\url{[https://huggingface.co/datasets/AI-MO/NuminaMath-CoT](https://github.com/project-numina/aimo-progress-prize/blob/main/report/numina_dataset.pdf)}}
}

@inproceedings{Cai2025IPM,
  title={Last iterate convergence of incremental methods as a model of forgetting},
  author={Cai, Xufeng and Diakonikolas, Jelena},
  booktitle={International Conference on Learning Representations},
  volume={2025},
  pages={102613--102647},
  year={2025}
}

@inproceedings{Koh2017InfluenceFunction,
author = {Koh, Pang Wei and Liang, Percy},
title = {Understanding black-box predictions via influence functions},
year = {2017},
publisher = {JMLR.org},
booktitle = {Proceedings of the 34th International Conference on Machine Learning - Volume 70},
pages = {1885-1894},
numpages = {10},
location = {Sydney, NSW, Australia},
series = {ICML'17}
}

@article{Agarwal2017LiSSA,
author = {Agarwal, Naman and Bullins, Brian and Hazan, Elad},
title = {Second-order stochastic optimization for machine learning in linear time},
year = {2017},
issue_date = {January 2017},
publisher = {JMLR.org},
volume = {18},
number = {1},
issn = {1532-4435},
journal = {J. Mach. Learn. Res.},
month = jan,
pages = {4148-4187},
numpages = {40}
}
\bibliographystyle{arxiv_reference_style}

\appendix
\section{Additional Theoretical Analysis}
\subsection{Parameter Distance of Gradient Descent}
\label{sec:GDconvergenceProof}
\GDconvergenceLemma*
\begin{proof}

Define the loss of the target problem as $f$, it is $\mu$-strongly convex and $L$-smooth so the following inequalities hold for any $x$ and $y$:
\begin{equation}
  \langle \nabla f(x) - \nabla f(y), x - y\rangle \geq \mu ||x - y||^2 \label{eq:Convex}
\end{equation}
\begin{equation}
  ||\nabla f(x) - \nabla f(y)||^2 \leq L \langle  \nabla f(x) - \nabla f(y), x - y\rangle  \label{eq:Smooth}
\end{equation}

Consider gradient descent with constant step size $\frac{1}{L}$, the parameter update follows:
\begin{gather}
\vtheta_{t+1} = \vtheta_t - \frac{1}{L} \nabla f(\vtheta_t) \nonumber \\
|| \vtheta_{t+1} - \vtheta^* ||^2 = || \vtheta_t - \frac{1}{L}\nabla f(\vtheta_t) - \vtheta^* ||^2 \label{eq:ParamDist}
\end{gather}

Note that $\nabla f(\vtheta^*) = 0$, Eq. \ref{eq:ParamDist} can be written as:
\begin{align}
|| \vtheta_{t+1} - \vtheta^* ||^2 &= || \vtheta_t - \vtheta^* - \frac{1}{L} \left(\nabla f(\vtheta_t) - \nabla f(\vtheta^*)\right) ||^2 \nonumber \\
&= || \vtheta_t - \vtheta^* ||^2 - \frac{2}{L} \langle  \nabla f(\vtheta_t) - \nabla f(\vtheta^*), \vtheta_t - \vtheta^*\rangle  + \frac{1}{L^2}|| \nabla f(\vtheta_t) - \nabla f(\vtheta^*) ||^2
\end{align}

Using Eq. \ref{eq:Smooth} to replace $\frac{1}{L^2}|| \nabla f(\vtheta_t) - \nabla f(\vtheta^*) ||^2$ we have:
\begin{equation}
|| \vtheta_{t+1} - \vtheta^* ||^2 \leq || \vtheta_t - \vtheta^* ||^2 - \frac{1}{L} \langle  \nabla f(\vtheta_t) - \nabla f(\vtheta^*), \vtheta_t - \vtheta^*\rangle 
\end{equation}

Using Eq. \ref{eq:Convex} to replace $\frac{1}{L} \langle \nabla f(\vtheta_t) - \nabla f(\vtheta^*), \vtheta_t - \vtheta^*\rangle $, we have:
\begin{align}
|| \vtheta_{t+1} - \vtheta^* ||^2 &\leq || \vtheta_t - \vtheta^* ||^2 - \frac{\mu}{L} || \vtheta_t - \vtheta^* ||^2 \nonumber \\
&= \left( 1 - \frac{\mu}{L} \right) || \vtheta_t - \vtheta^* ||^2 \nonumber \\
&= \left( 1 - \frac{1}{\kappa} \right) || \vtheta_t - \vtheta^* ||^2
\end{align}

When the iteration proceeds from $1$ to $t$, the distance between parameters satisfies:
\begin{align}
|| \vtheta_{t} - \vtheta^* ||^2 &\leq \left( 1 - \frac{1}{\kappa} \right) || \vtheta_{t - 1} - \vtheta^* ||^2 \nonumber \\
&\leq \left( 1 - \frac{1}{\kappa} \right)^2 || \vtheta_{t - 2} - \vtheta^* ||^2 \nonumber \\
&\leq ... \nonumber \\
&\leq \left( 1 - \frac{1}{\kappa} \right)^{t - 1} || \vtheta_1 - \vtheta^* ||^2 \nonumber \\
&\leq \left( 1 - \frac{1}{\kappa} \right)^t || \vtheta_0 - \vtheta^* ||^2
\end{align}

\end{proof}

\subsection{Iteration Counts of Gradient Descent}
\label{sec:GDtimesProof}
\GDtimesLemma*

\begin{proof}
From the upper bound of parameter distance in Lemma \ref{lem:GDconvergence}, when the gradient descent with constant step size $\frac{1}{L}$ converges at iteration $T$, the following should hold:
\begin{equation}
  || \vtheta_T-\vtheta^* ||^2 \leq \left( 1-\frac{1}{\kappa}\right) ^T || \vtheta_0-\vtheta^* ||^2
\end{equation}

Taking square roots for both sides yields
\begin{equation}
|| \vtheta_T - \vtheta^* || \leq \left( 1-\frac{1}{\kappa}\right)^{\frac{T}{2}} || \vtheta_0 - \vtheta^* ||
\end{equation}

The convergence condition is defined as $|| \vtheta_T-\vtheta^* || \leq \varepsilon$. To achieve this the following should hold:
\begin{gather}
\left( 1-\frac{1}{\kappa}\right)^{\frac{T}{2}} || \vtheta_0 - \vtheta^* || \leq \varepsilon \nonumber \\
\left( 1-\frac{1}{\kappa}\right)^{\frac{T}{2}} \leq \frac{\varepsilon}{|| \vtheta_0 - \vtheta^* ||}
\end{gather}

Taking the logarithm of both sides, we have
\begin{equation}
\frac{T}{2} \ln \left( 1-\frac{1}{\kappa}\right) \leq \ln \left( \frac{\varepsilon}{|| \vtheta_0 - \vtheta^* ||} \right)
\end{equation}

Since $\kappa \geq 1$, we have $0 \leq 1 - \frac{1}{\kappa} < 1$, taking the negative of both sides gives
\begin{gather}
\frac{T}{2} \ln \left( \frac{\kappa}{\kappa - 1}\right) \geq \ln \left( \frac{|| \vtheta_0 - \vtheta^* ||}{\varepsilon} \right) \nonumber \\
\frac{T}{2} \geq \frac{\ln \left( \frac{|| \vtheta_0 - \vtheta^* ||}{\varepsilon} \right)}{\ln \left( \frac{\kappa}{\kappa - 1}\right)}
\label{eq:Iter}
\end{gather}

Since $\ln (x) \leq x - 1$ holds for any $x$,  $\ln \left( \frac{\kappa}{\kappa - 1}\right) \leq \frac{\kappa}{\kappa - 1} - 1 = \frac{1}{\kappa - 1}$, and Eq. \ref{eq:Iter} can be written as:
\begin{align}
\frac{T}{2} &\geq \frac{\ln \left( \frac{|| \vtheta_0 - \vtheta^* ||}{\varepsilon} \right)}{\frac{1}{\kappa - 1}} \nonumber \\
T &\geq 2 (\kappa - 1) \ln \left( \frac{|| \vtheta_0 - \vtheta^* ||}{\varepsilon} \right)
\end{align}

The iteration complexity to reach $|| \vtheta_T-\vtheta^* || \leq \varepsilon$ is $\Theta \left(\kappa \log \left(\frac{|| \vtheta_0 - \vtheta^* ||}{\varepsilon}\right)\right)$, with the hidden constants depending on the step size choice and whether one measures distance or function suboptimality. It is in correspondence with Eq. \ref{eq:GDtimes}.

\end{proof}

\subsection{Effectiveness of Curriculum Learning}
\label{sec:CLeffectivenessProof}
\CLeffectivenessTheorem*

\begin{proof}
After convergence on the easier curriculum task, the distance between the current model parameter and the optimal parameter of the original problem is bounded via the triangle inequality by the sum of the distances to the curriculum optimum and between the two optima:
\begin{equation}
  || \vtheta_{T_e}-\vtheta^* || \leq || \vtheta_{T_e}-\vtheta^*_e || + || \vtheta^*_e-\vtheta^* || \leq \varepsilon_e + \delta 
\label{eq:DistMiddle}
\end{equation}
where $\delta = || \vtheta^*_e-\vtheta^* ||$ denotes the distance between the two optima. The distance in Eq. \ref{eq:DistMiddle} equals the distance between the initial and optimal parameters for the original problem when its optimization starts in the curriculum.

Next we consider the numbers of iterations $T_e$ and $T_f$ for the easier problem and the original problem in the curriculum respectively. Using Lemma \ref{lem:GDtimes}, we have:
\begin{equation}
  T_e \approx C \kappa_e \log \left(\frac{|| \vtheta_0-\vtheta^*_e ||}{\varepsilon_e}\right)
\end{equation}
\begin{equation}
  T_f \approx C \kappa \log \left(\frac{|| \vtheta_{T_e}-\vtheta^* ||}{\varepsilon}\right) \lesssim C \kappa \log \left(\frac{\varepsilon_e + \delta}{\varepsilon}\right)
\end{equation}

The total number of iterations in curriculum learning includes the iterations of the easier curriculum problem and the original problem, which is
\begin{equation}
  T_{curriculum} = T_e + T_f
\end{equation}

Vanilla curriculum learning directly optimizes the original problem. The number of iterations $T_{vanilla}$ satisfies Lemma \ref{lem:GDtimes}. Curriculum learning's effectiveness lies in its faster convergence, i.e., fewer iterations. To prove the effectiveness of curriculum learning, the following should hold:
\begin{equation}
  T_{curriculum} - T_{vanilla} < 0
\end{equation}

That is to prove:
\begin{equation}
  C \kappa_e \log \left(\frac{|| \vtheta_0-\vtheta^*_e ||}{\varepsilon_e}\right) + C \kappa \log \left(\frac{\varepsilon_e + \delta}{\varepsilon}\right) - C \kappa \log \left(\frac{|| \vtheta_0-\vtheta^* ||}{\varepsilon}\right) < 0
\end{equation}

Expanding the logarithm, that is to prove:
\begin{equation}
  C \kappa \log \left(\frac{\varepsilon_e + \delta}{|| \vtheta_0-\vtheta^* ||}\right) + C \kappa_e \log \left(\frac{|| \vtheta_0-\vtheta^*_e ||}{\varepsilon_e}\right) < 0 \label{eq:CurriculumConvergeTime}
\end{equation}
the left-hand side of the inequality comprises two terms: the first reflects the reduction in iteration count after introducing curriculum learning, the second corresponds to the iterations required to learn the easier curriculum problem.

According to the curriculum design $|| \vtheta_e^* - \vtheta^* || \ll || \vtheta_0-\vtheta^* || - \varepsilon_e$, using $\delta$ to substitute the left part we have $\varepsilon_e + \delta \ll || \vtheta_0-\vtheta^* ||$. The first term on the left side of Eq. \ref{eq:CurriculumConvergeTime} is negative. Furthermore, the curriculum satisfies $\kappa_e<\kappa$ so Eq. \ref{eq:CurriculumConvergeTime} always holds; the given curriculum converges faster than learning at once.

\end{proof}

\section{Dataset Details}
\label{sec:Dataset}

In this section, we provide more information about datasets used in the experiments. Three math reasoning datasets are adopted to evaluate the model, whose details are described as follows.

The GSM8K dataset \citep{Cobbe2021GSM8KDataset} contains easy mathematical questions with their contexts from daily lives. In the experiments we adopt the version by \citet{Parashar2026E2HReasoner}, where the questions are graded and split into four levels. The training set contains 7473 samples with 1630, 1747, 2327, 1769 samples at Level 0 to Level 3 respectively. The test set contains 1319 samples with 206, 300, 354, 459 samples at Level 0 to Level 3 respectively.

The MATH dataset proposed by \citet{Hendrycks2021MATHDataset} contains five difficulty levels. The majority of questions in the MATH dataset do not have a context, which makes understanding easier. The training set includes 7498 samples: 564 at Level 0, 1348 at Level 1, 1592 at Level 2, 1690 at Level 3 and 2304 at Level 4. The test set includes 5000 samples: 437 at Level 0, 894 at Level 1, 1131 at Level 2, 1214 at Level 3 and 1324 at Level 4. The training distribution is skewed toward challenging examples, which promotes the acquisition of richer representations required to handle such cases.

The cn\_k12 dataset is created by selecting samples from the cn\_k12 subset of NuminaMath-CoT \citep{LI2024NuminaMath}. To align with existing datasets, we cleaned the cn\_k12 subset by removing samples with links, non-ASCII characters, inconsistent question types, omitted or incomplete answers. The cleaned data was then categorized into four difficulty levels by answer length, namely: Level 1 (1-179 characters), Level 2 (180-259 characters), Level 3 (260-369 characters), and Level 4 (370-800 characters). Samples with answers longer than 800 characters are dropped as outliers. We sampled 2000 training and 700 testing instances from each level to form the final cn\_k12 dataset.

\section{Hyperparameter Details}
\label{sec:HyperParameter}

\begin{figure}[t]
\centering
\begin{subfigure}[b]{0.329\textwidth}
  \centering
  \includegraphics[height=3.07cm]{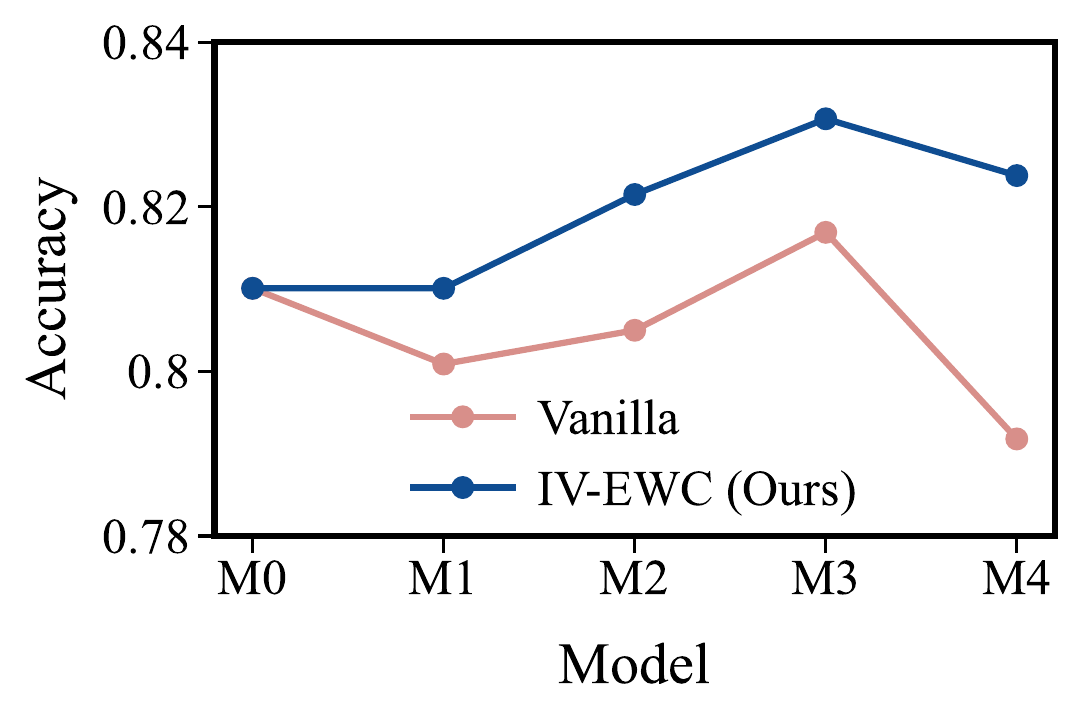}
  \caption{Level 0}
  \hfill
\end{subfigure}
\begin{subfigure}[b]{0.329\textwidth}
  \centering
  \includegraphics[height=3.07cm]{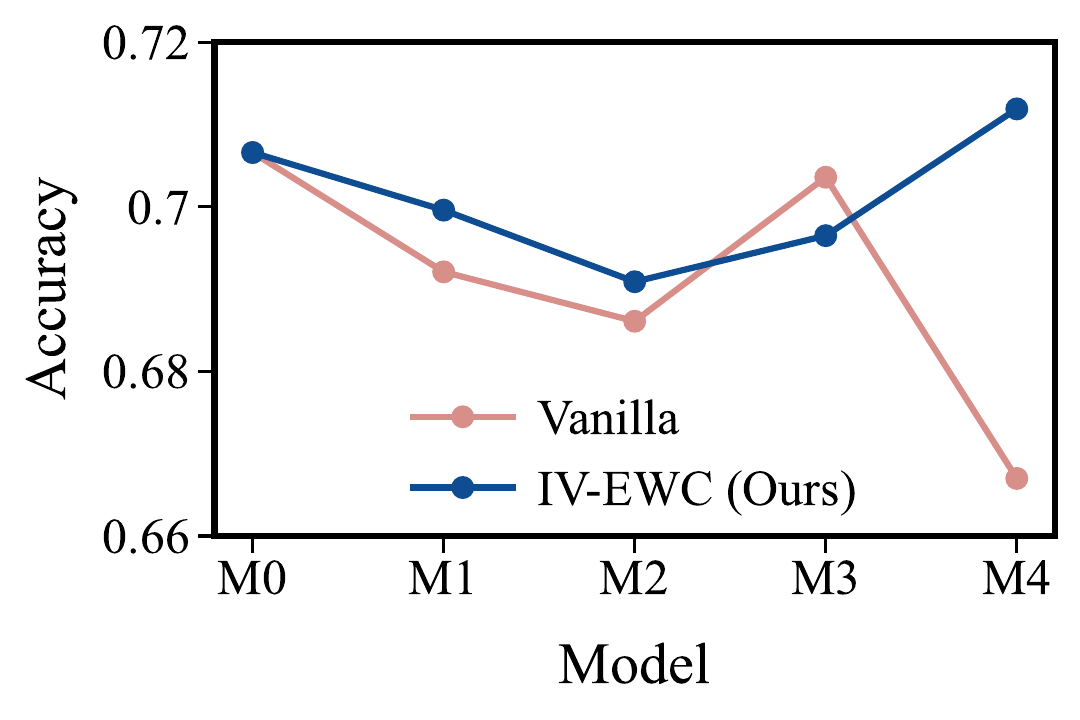}
  \caption{Level 1}
  \hfill
\end{subfigure}
\begin{subfigure}[b]{0.329\textwidth}
  \centering
  \includegraphics[height=3.07cm]{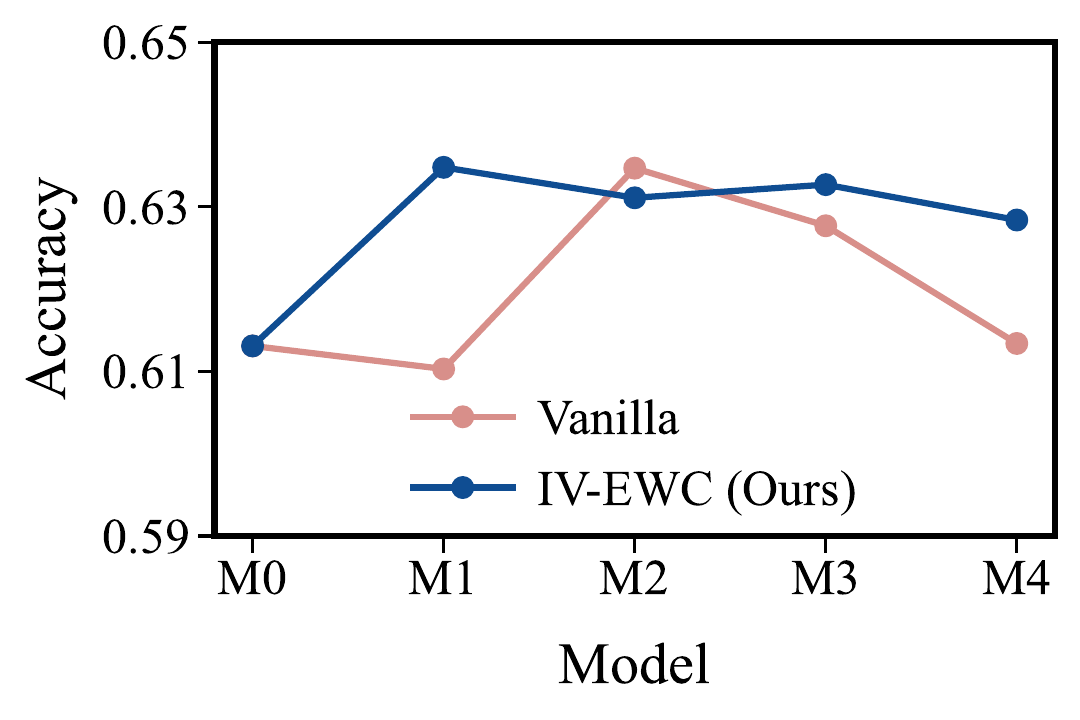}
  \caption{Level 2}
  \hfill
\end{subfigure}
\\
\begin{subfigure}[b]{0.329\textwidth}
  \centering
  \includegraphics[height=3.07cm]{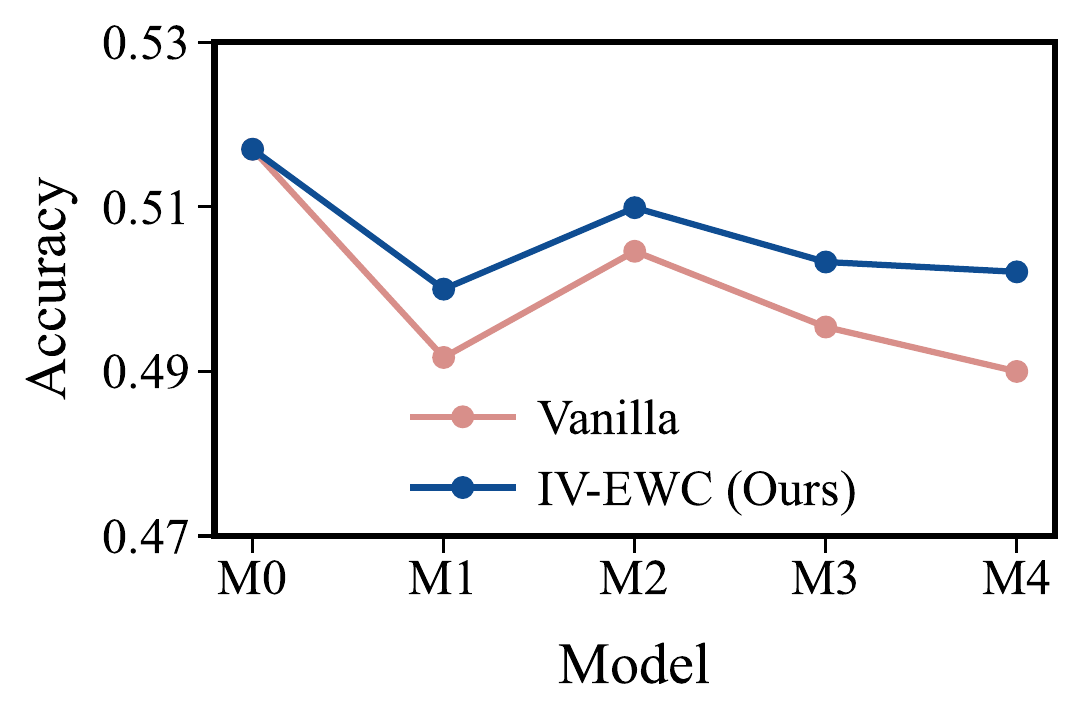}
  \caption{Level 3}
  \hfill
\end{subfigure}
\begin{subfigure}[b]{0.329\textwidth}
  \centering
  \includegraphics[height=3.07cm]{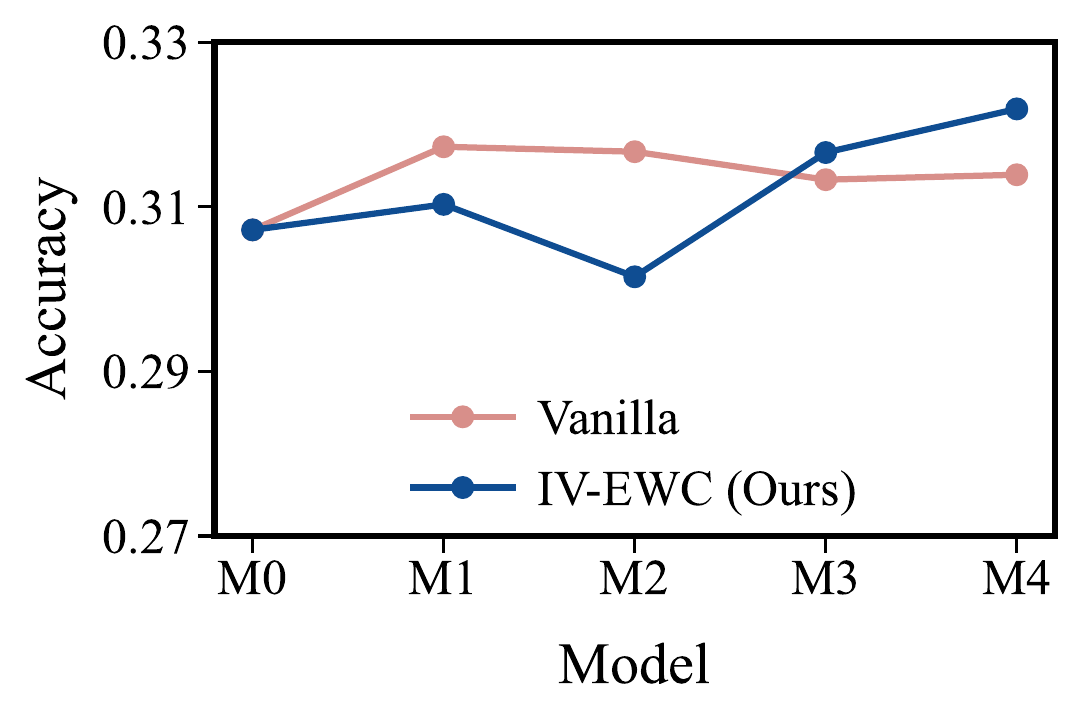}
  \caption{Level 4}
  \hfill
\end{subfigure}
\begin{subfigure}[b]{0.329\textwidth}
  \centering
  \includegraphics[height=3.07cm]{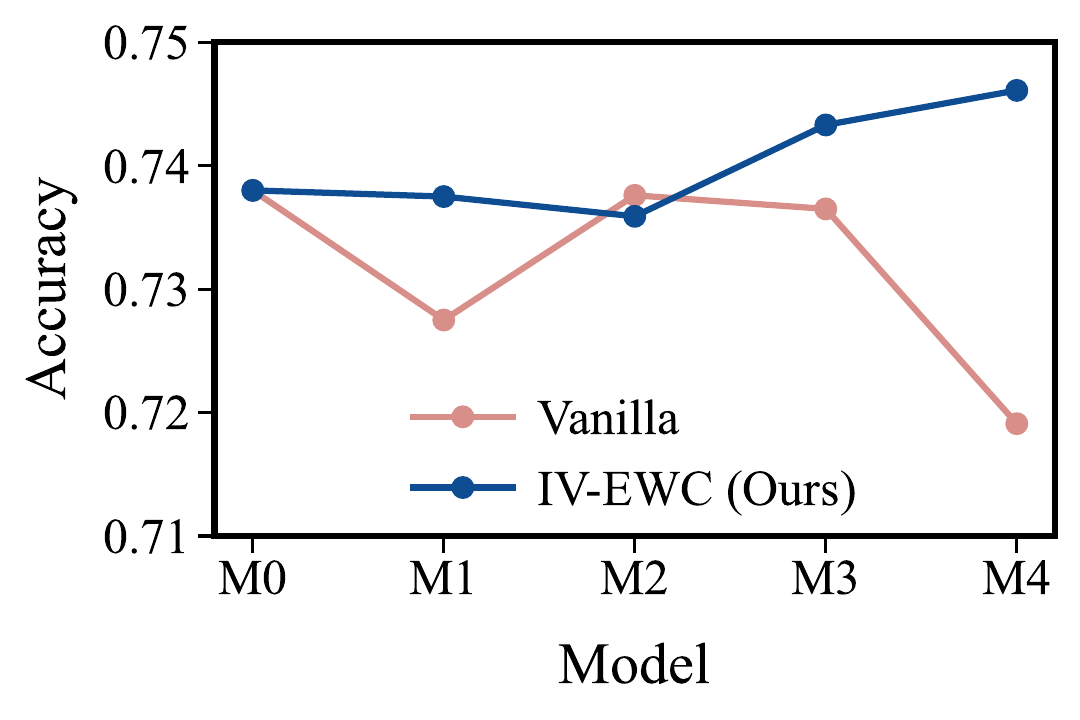}
  \caption{Overall}
  \hfill
\end{subfigure}
\caption{Accuracy of models trained after each task using vanilla curriculum (dashed line) and \textsf{IV-EWC} (solid line) on the MATH dataset based on Qwen2.5-3B-Instruct, evaluated on the test sets of the five curriculum tasks.}
\label{fg:TrainingProcessMATH3B}
\end{figure}

In this section, we list all hyperparameter settings for methods in the experiments. Moreover, we tuned the learning rates for different datasets on diverse models, the adopted learning rates are also reported in this section.

For replay-based methods (random replay and A-GEM), the number of selected replay samples or reference samples was 100, in correspondence with the validation set increment 100 of \textsf{IV-EWC}. For regularization-based methods (EWC and \textsf{IV-EWC}), we keep the same regularization weight across methods, adding dynamic regularization to EWC. The initial regularization weight $\gamma$ was set to 1.2e8. For UPGD method, we set hyperparameter sigma to zero, disabling random noise to align with other methods. For \textsf{IV-EWC}, LiSSA recursion depth used to compute influence was fixed at 50. To find the appropriate selections of hyperparameters for \textsf{IV-EWC}, a parameter sensitivity analysis was carried out with details in Appendix \ref{sec:ParameterSensitivity}.

To avoid overfitting to the math reasoning tasks, which is reflected by the sharp decrease in all curriculum tasks' accuracy, we adjust the learning rate and reduce it for easier datasets and smaller models. For the GSM8K dataset, the learning rates were 1.0e-7 for Qwen2.5-3B-Instruct, 1.5e-7 for Qwen2.5-7B-Instruct, and 3.0e-8 for Llama-3.2-3B-Instruct. For the MATH dataset, the learning rates were 5.5e-8 for Qwen2.5-3B-Instruct, 8.0e-8 for Qwen2.5-7B-Instruct, and 1.0e-8 for Llama-3.2-3B-Instruct. For the cn\_k12 dataset, the learning rate was set to 1.5e-7 for all base models.

\section{Additional Results during Curriculum Learning}
\label{sec:CompareExp}

In this section, we provide detailed accuracy results during the curriculum learning process of models. The results include accuracies of models trained after each curriculum task, evaluated on the test sets of all curriculum tasks. To ensure clarity, the data are sorted by the trained base model.


Figure \ref{fg:TrainingProcessMATH3B} shows the results of vanilla curriculum learning and \textsf{IV-EWC} on the MATH dataset with Qwen2.5-3B-Instruct. In the later stages of training, \textsf{IV-EWC} surpasses the vanilla baseline in overall accuracy. During curriculum learning for mathematical reasoning, training on a single task does not necessarily translate into substantial accuracy gains for that task and can also alter performance on other tasks, which is in correspondence with the results of grouped accuracy.

\begin{figure}[t]
\centering
\begin{subfigure}[b]{0.329\textwidth}
  \centering
  \includegraphics[height=3.07cm]{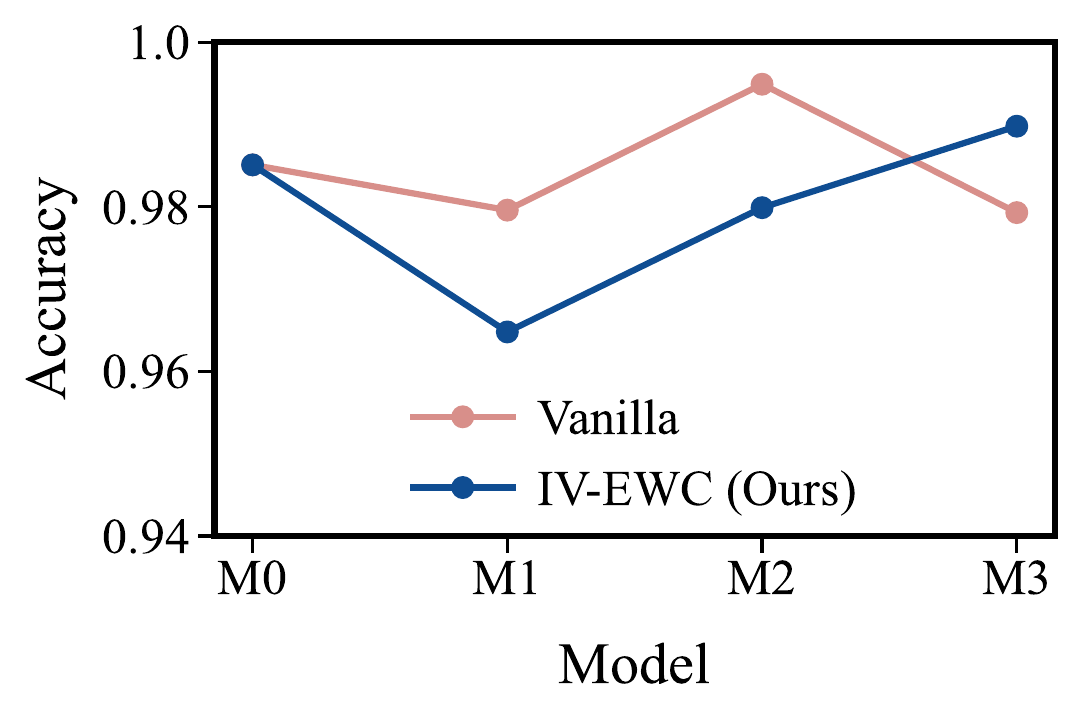}
  \caption{Level 0}
  \hfill
\end{subfigure}
\begin{subfigure}[b]{0.329\textwidth}
  \centering
  \includegraphics[height=3.07cm]{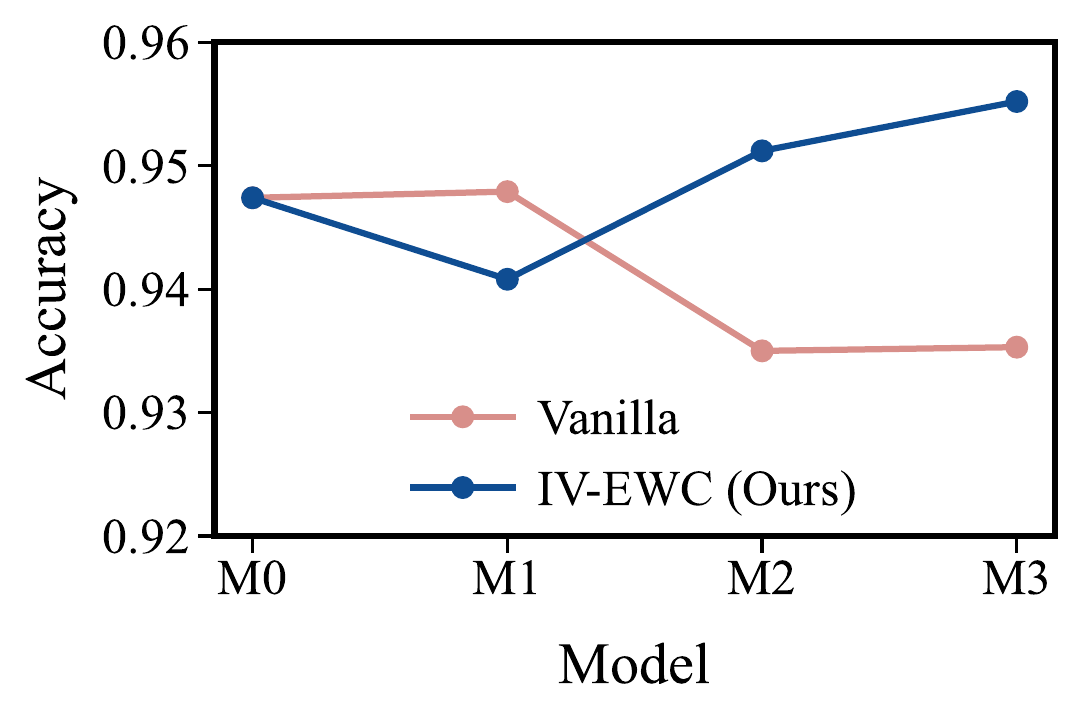}
  \caption{Level 1}
  \hfill
\end{subfigure}
\begin{subfigure}[b]{0.329\textwidth}
  \centering
  \includegraphics[height=3.07cm]{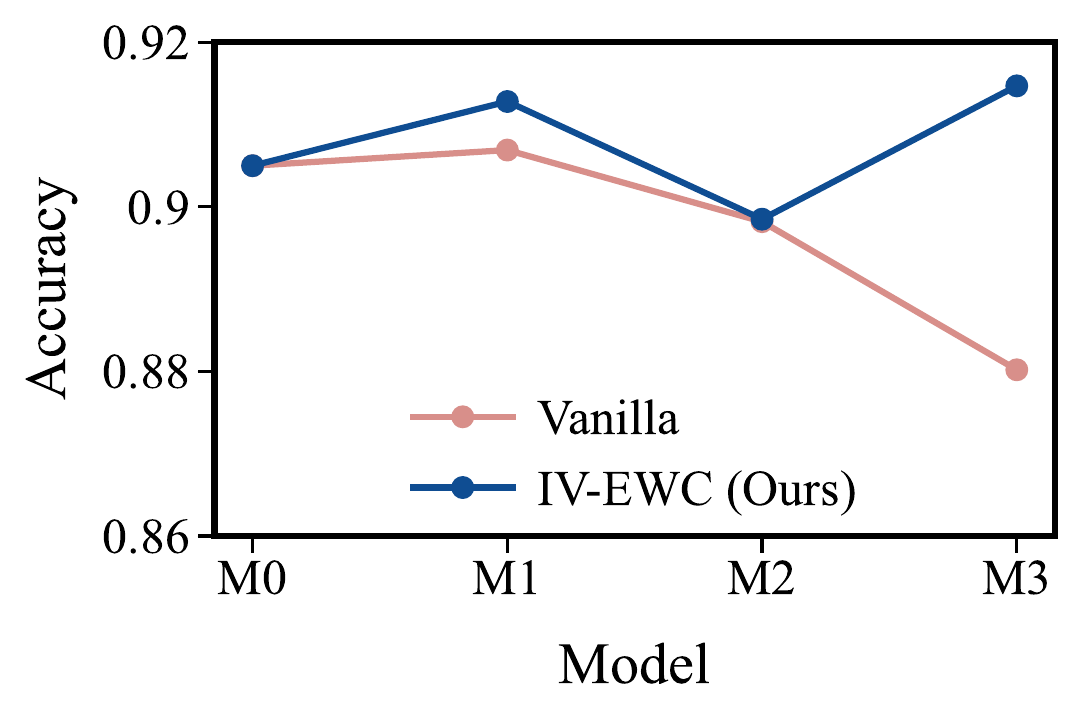}
  \caption{Level 2}
  \hfill
\end{subfigure}
\\
\begin{subfigure}[b]{0.329\textwidth}
  \centering
  \includegraphics[height=3.07cm]{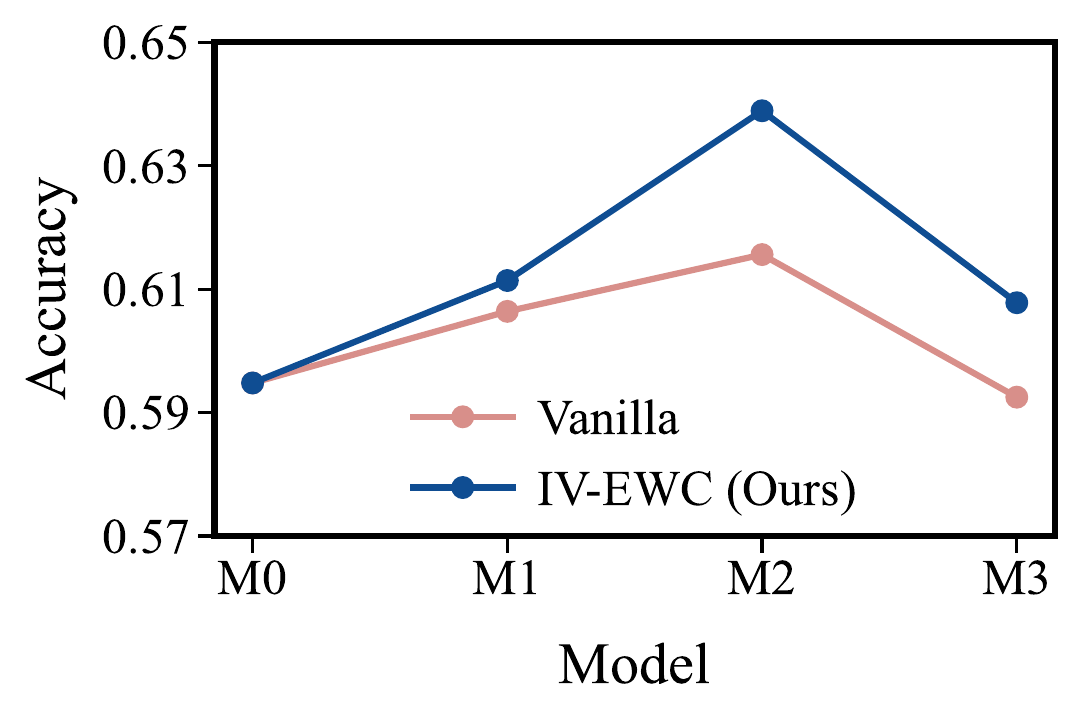}
  \caption{Level 3}
  \hfill
\end{subfigure}
\begin{subfigure}[b]{0.329\textwidth}
  \centering
  \includegraphics[height=3.07cm]{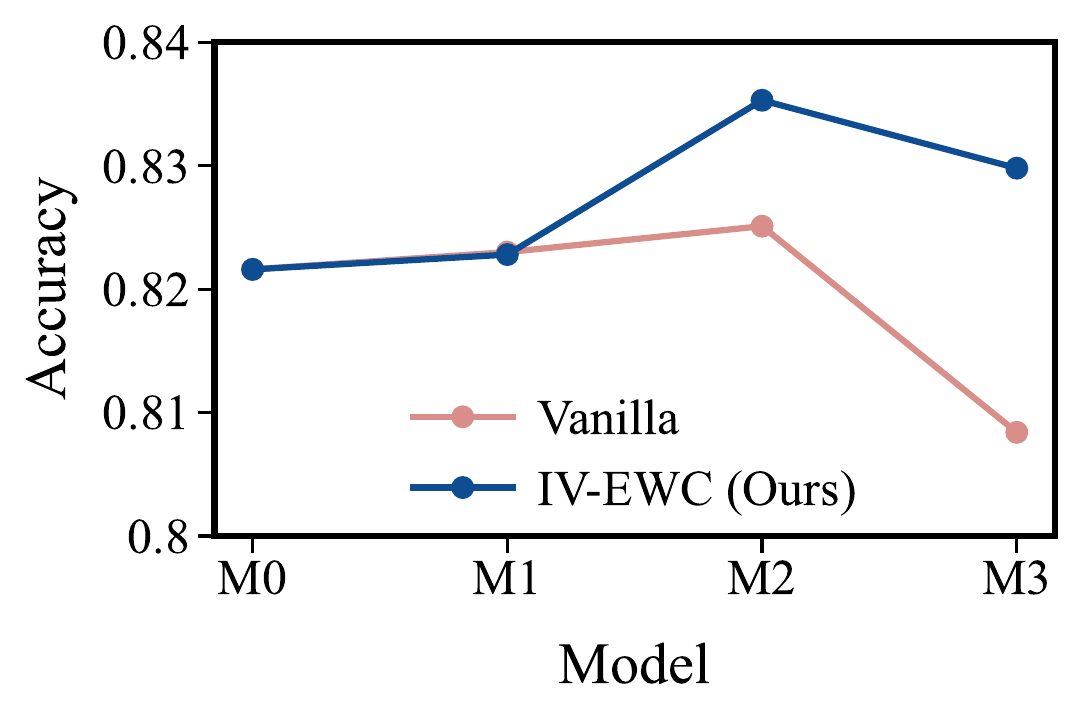}
  \caption{Overall}
  \hfill
\end{subfigure}
\caption{Accuracy of models trained after each task using vanilla curriculum (dashed line) and \textsf{IV-EWC} (solid line) on GSM8K dataset based on Qwen2.5-3B-Instruct, evaluated on the test sets of the four curriculum tasks.}
\label{fg:TrainingProcessGSM8K3B}
\end{figure}

\begin{figure}[t]
\centering
\begin{subfigure}[b]{0.329\textwidth}
  \centering
  \includegraphics[height=3.07cm]{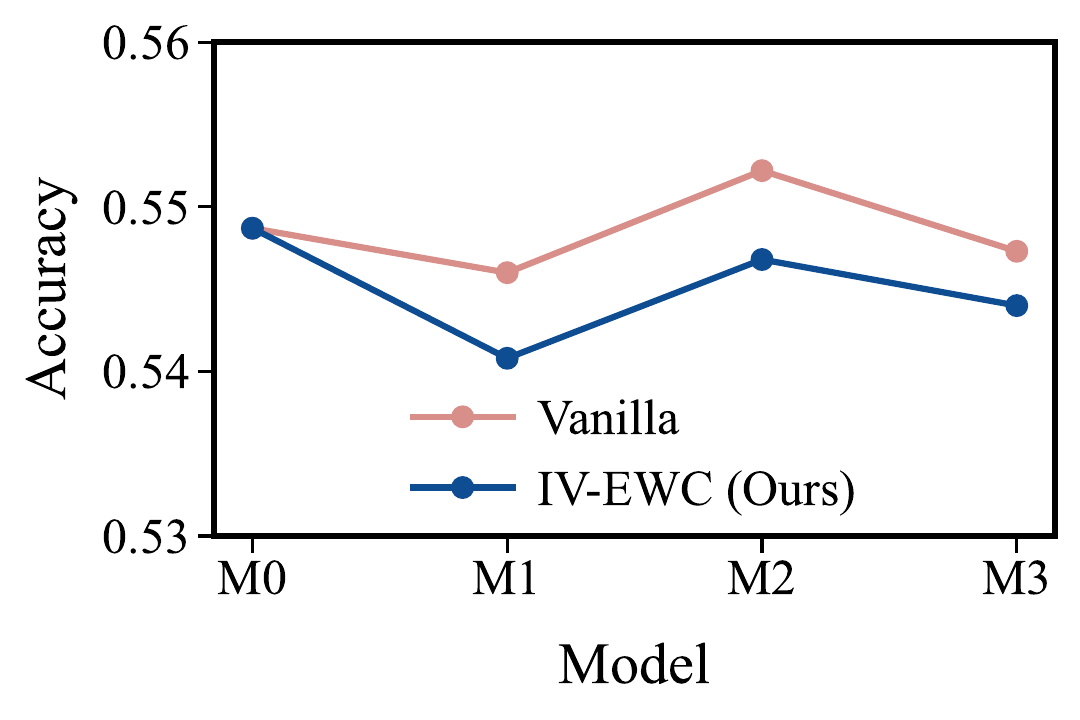}
  \caption{Level 0}
  \hfill
\end{subfigure}
\begin{subfigure}[b]{0.329\textwidth}
  \centering
  \includegraphics[height=3.07cm]{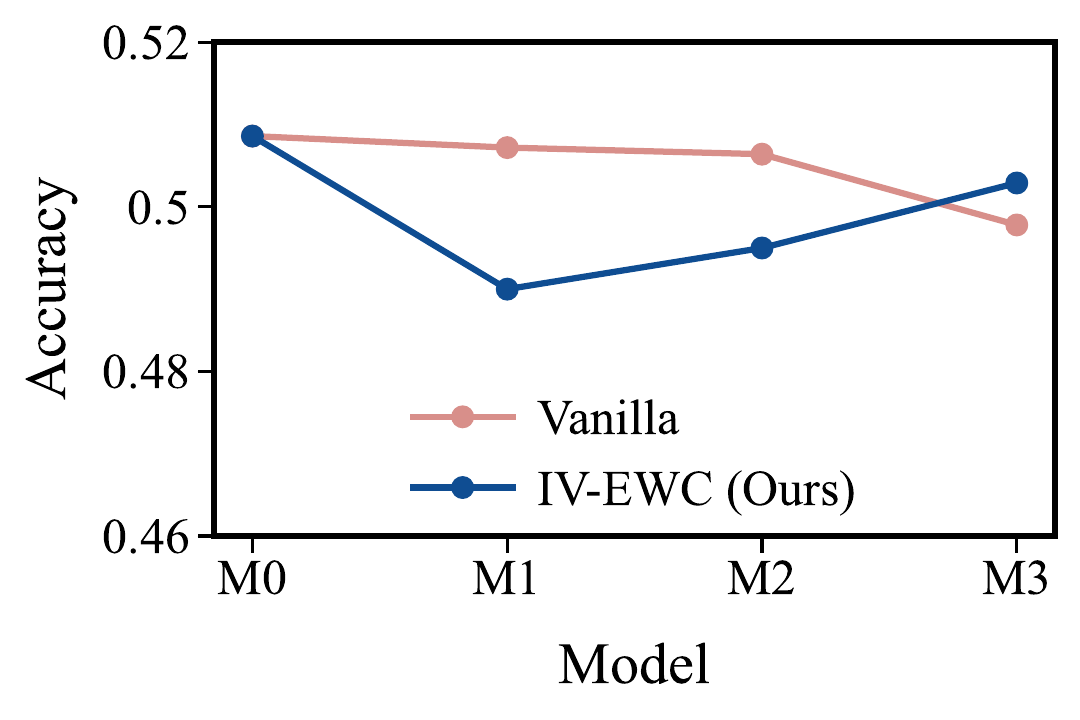}
  \caption{Level 1}
  \hfill
\end{subfigure}
\begin{subfigure}[b]{0.329\textwidth}
  \centering
  \includegraphics[height=3.07cm]{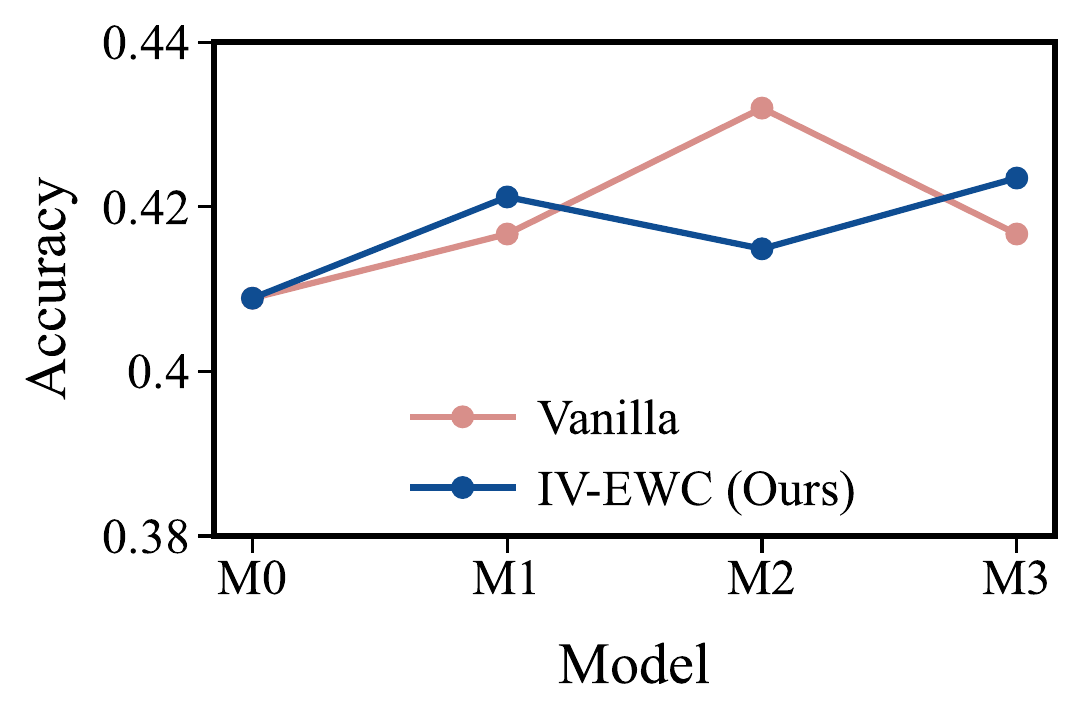}
  \caption{Level 2}
  \hfill
\end{subfigure}
\\
\begin{subfigure}[b]{0.329\textwidth}
  \centering
  \includegraphics[height=3.07cm]{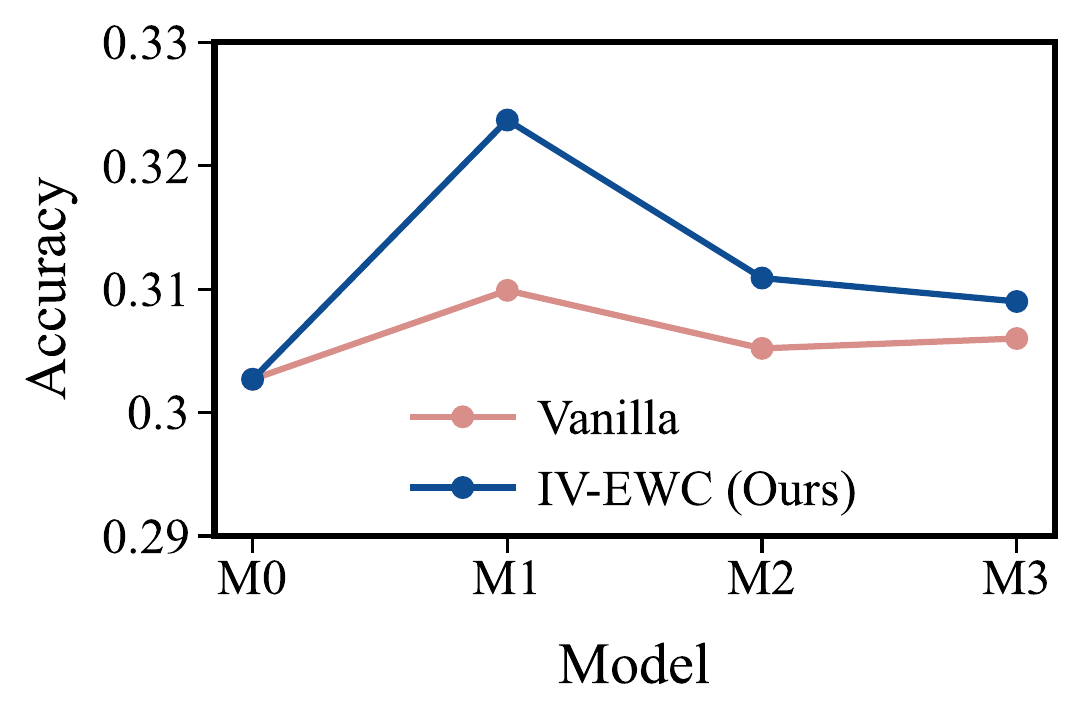}
  \caption{Level 3}
  \hfill
\end{subfigure}
\begin{subfigure}[b]{0.329\textwidth}
  \centering
  \includegraphics[height=3.07cm]{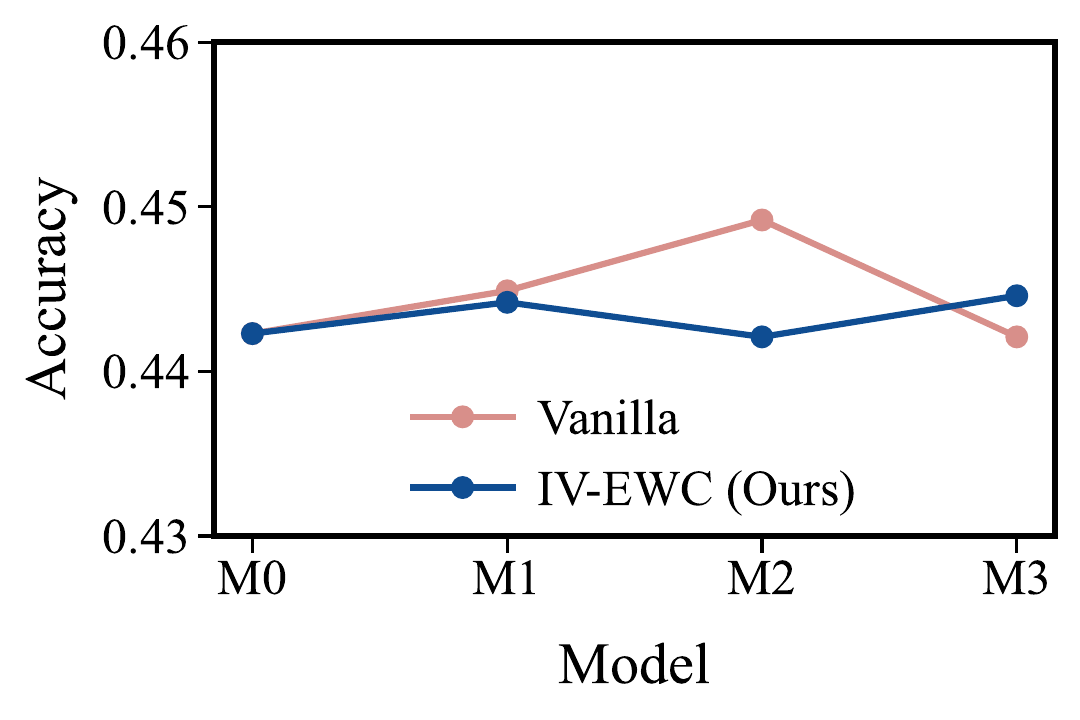}
  \caption{Overall}
  \hfill
\end{subfigure}
\caption{Accuracy of models trained after each task using vanilla curriculum (dashed line) and \textsf{IV-EWC} (solid line) on cn\_k12 dataset based on Qwen2.5-3B-Instruct, evaluated on the test sets of the four curriculum tasks.}
\label{fg:TrainingProcessCNK123B}
\end{figure}

Figure \ref{fg:TrainingProcessGSM8K3B} reports GSM8K results with Qwen2.5-3B-Instruct for the vanilla curriculum baseline and \textsf{IV-EWC}. Given the relative simplicity of GSM8K, curriculum learning on this dataset is prone to overfitting. Across the tasks of Level 1 and Level 2, the baseline accuracy declines as training proceeds, while \textsf{IV-EWC} keeps accuracy variations within a narrow range.


Figure \ref{fg:TrainingProcessCNK123B} presents results for the vanilla curriculum and \textsf{IV-EWC} on the cn\_k12 dataset using Qwen2.5-3B-Instruct. Both methods exhibit forgetting, evidenced by declining accuracy on the easier tasks. For the harder tasks, accuracy improves gradually as the curriculum progresses, indicating that curriculum learning facilitates acquisition of new tasks. At the conclusion of the curriculum, \textsf{IV-EWC} outperforms the vanilla approach, with gains primarily on the harder tasks.

\begin{figure}[t]
\centering
\begin{subfigure}[b]{0.329\textwidth}
  \centering
  \includegraphics[height=3.07cm]{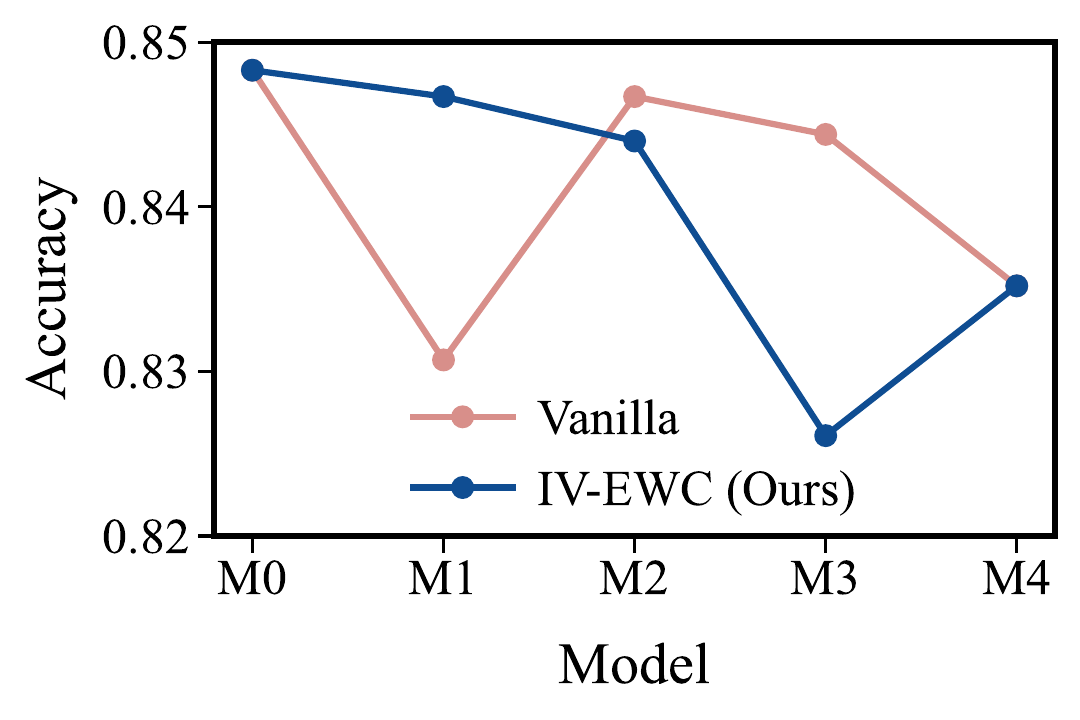}
  \caption{Level 0}
  \hfill
\end{subfigure}
\begin{subfigure}[b]{0.329\textwidth}
  \centering
  \includegraphics[height=3.07cm]{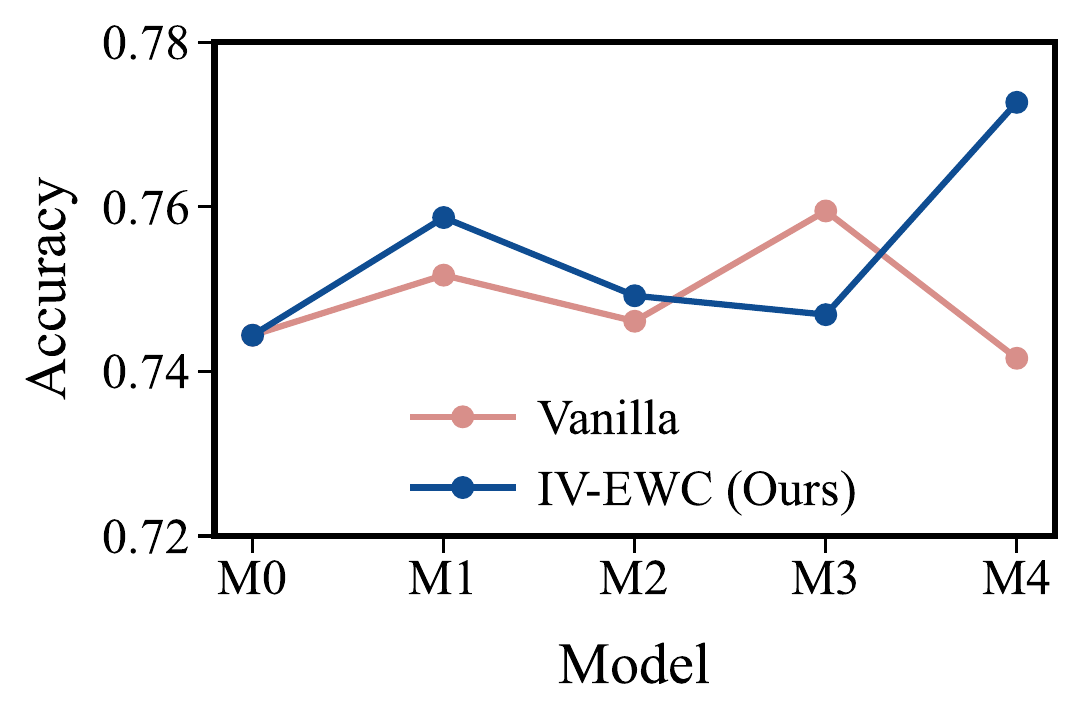}
  \caption{Level 1}
  \hfill
\end{subfigure}
\begin{subfigure}[b]{0.329\textwidth}
  \centering
  \includegraphics[height=3.07cm]{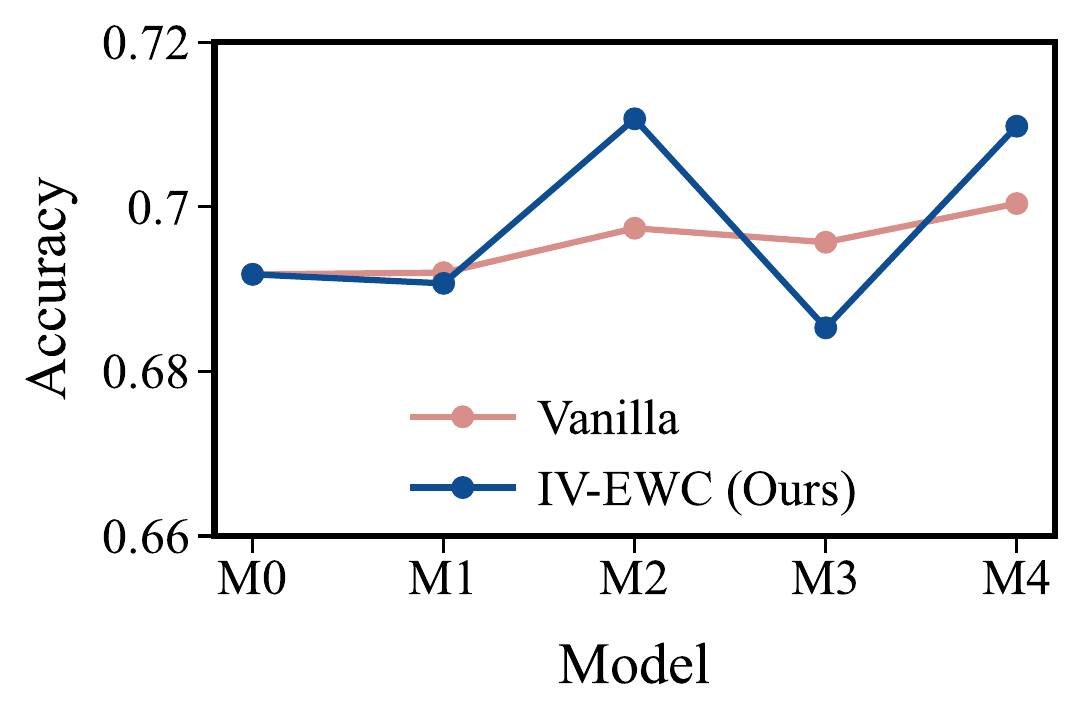}
  \caption{Level 2}
  \hfill
\end{subfigure}
\\
\begin{subfigure}[b]{0.329\textwidth}
  \centering
  \includegraphics[height=3.07cm]{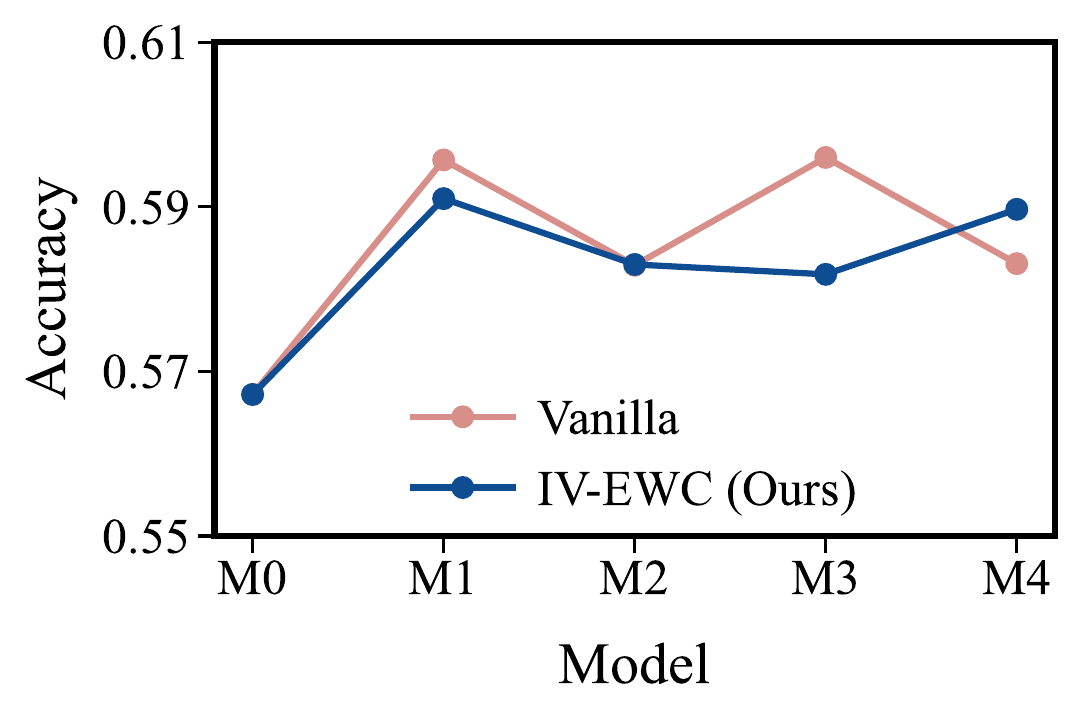}
  \caption{Level 3}
  \hfill
\end{subfigure}
\begin{subfigure}[b]{0.329\textwidth}
  \centering
  \includegraphics[height=3.07cm]{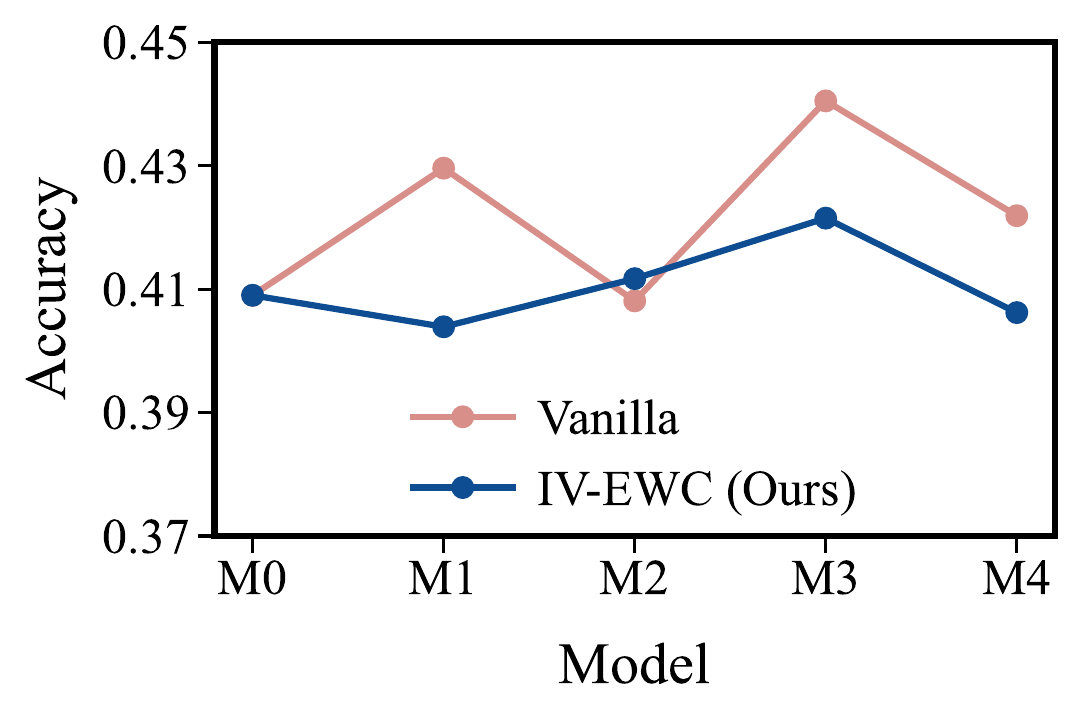}
  \caption{Level 4}
  \hfill
\end{subfigure}
\begin{subfigure}[b]{0.329\textwidth}
  \centering
  \includegraphics[height=3.07cm]{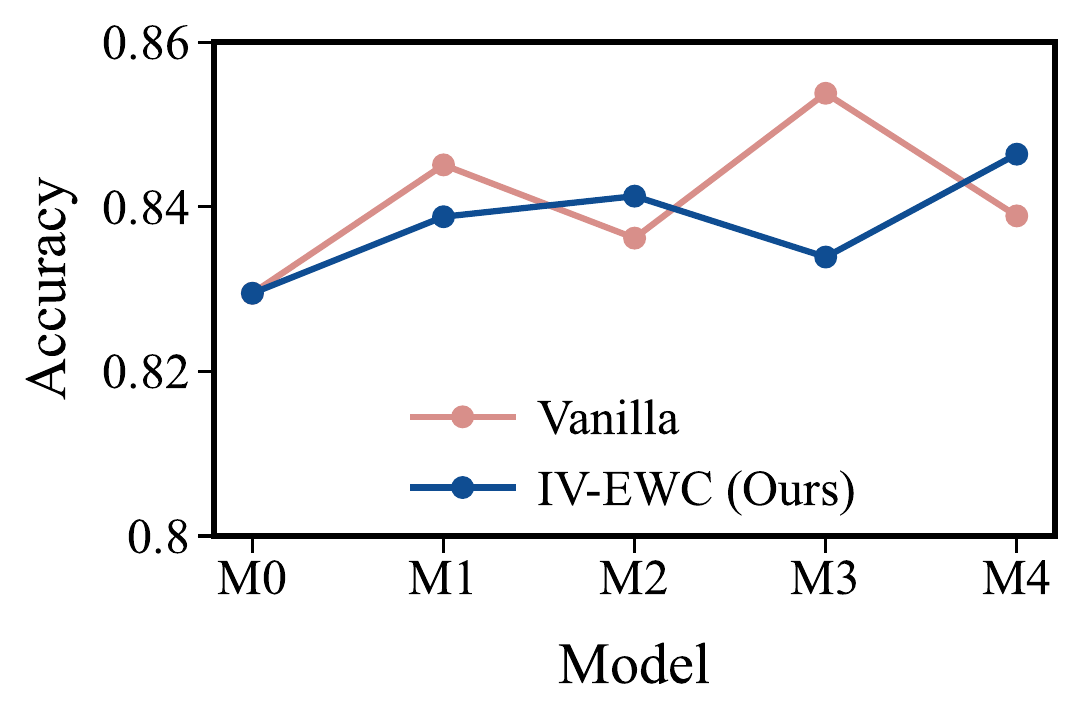}
  \caption{Overall}
  \hfill
\end{subfigure}
\caption{Accuracy of models trained after each task using vanilla curriculum (dashed line) and \textsf{IV-EWC} (solid line) on the MATH dataset based on Qwen2.5-7B-Instruct, evaluated on the test sets of the five curriculum tasks.}
\label{fg:TrainingProcessMATH7B}
\end{figure}

\begin{figure}[t]
\centering
\begin{subfigure}[b]{0.329\textwidth}
  \centering
  \includegraphics[height=3.07cm]{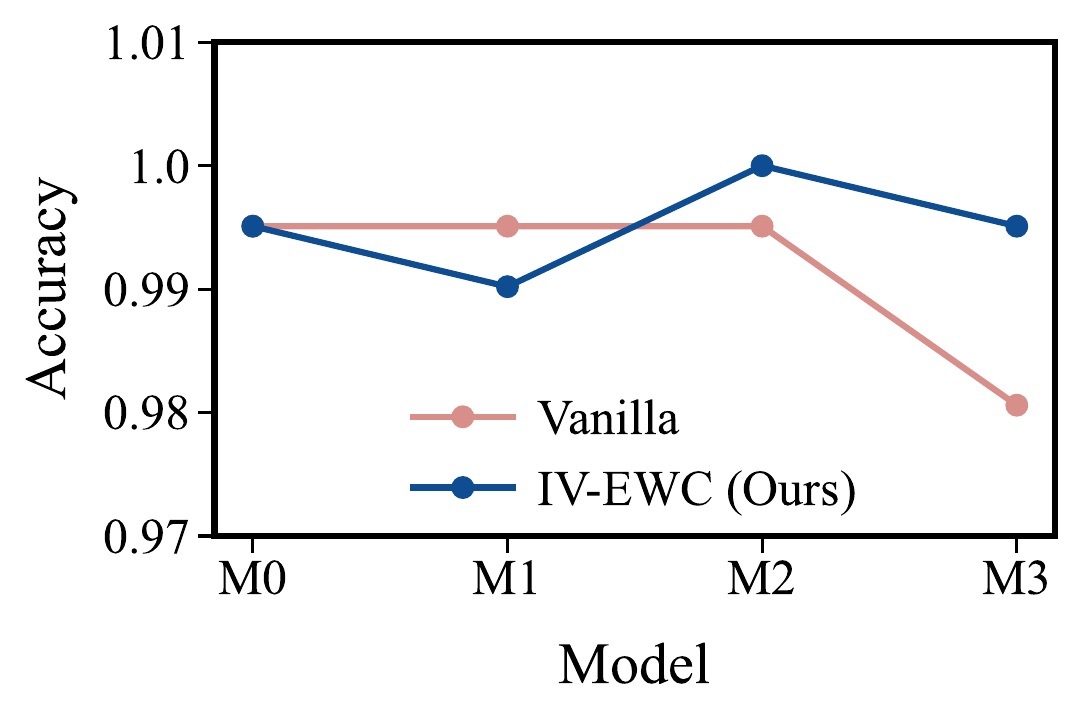}
  \caption{Level 0}
  \hfill
\end{subfigure}
\begin{subfigure}[b]{0.329\textwidth}
  \centering
  \includegraphics[height=3.07cm]{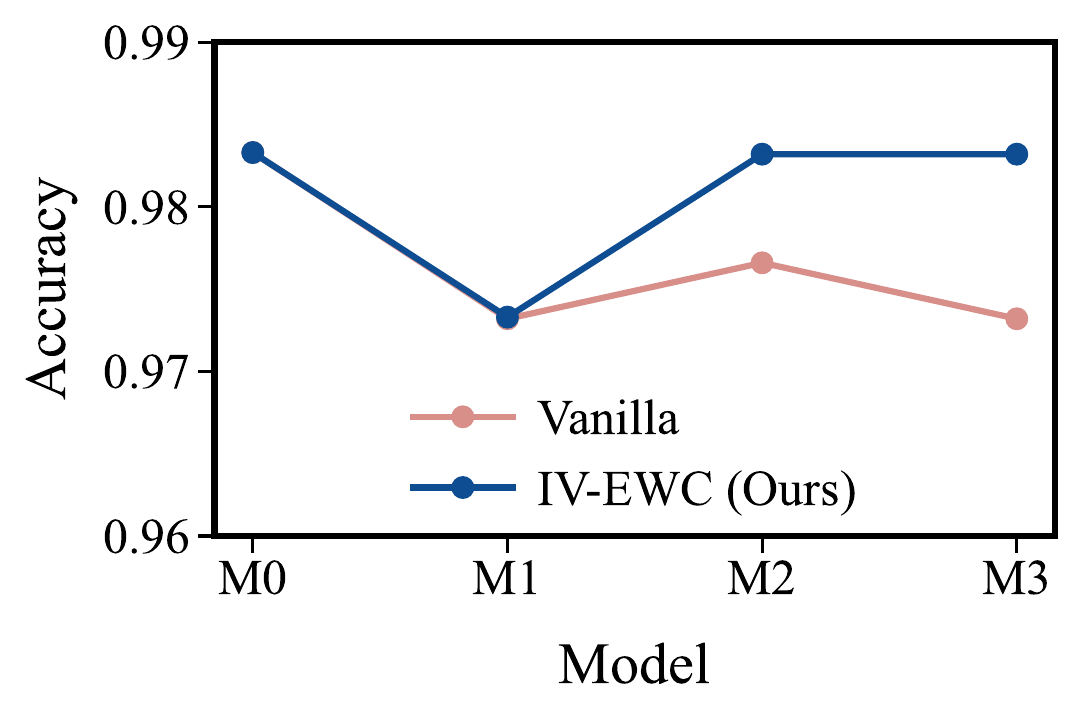}
  \caption{Level 1}
  \hfill
\end{subfigure}
\begin{subfigure}[b]{0.329\textwidth}
  \centering
  \includegraphics[height=3.07cm]{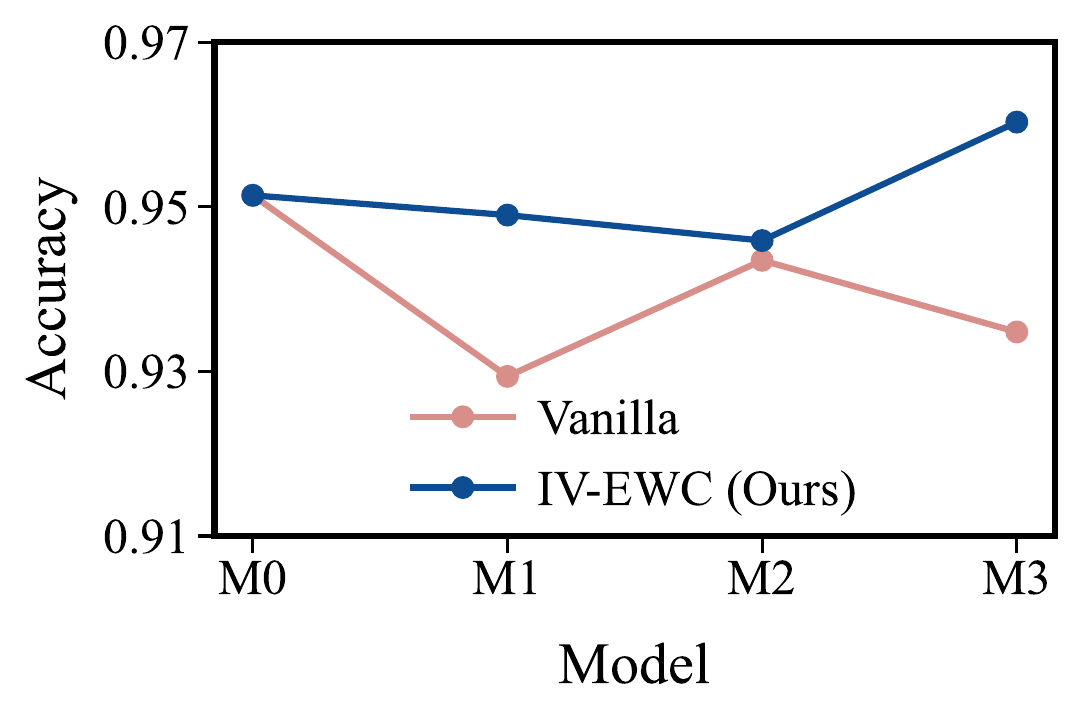}
  \caption{Level 2}
  \hfill
\end{subfigure}
\\
\begin{subfigure}[b]{0.329\textwidth}
  \centering
  \includegraphics[height=3.07cm]{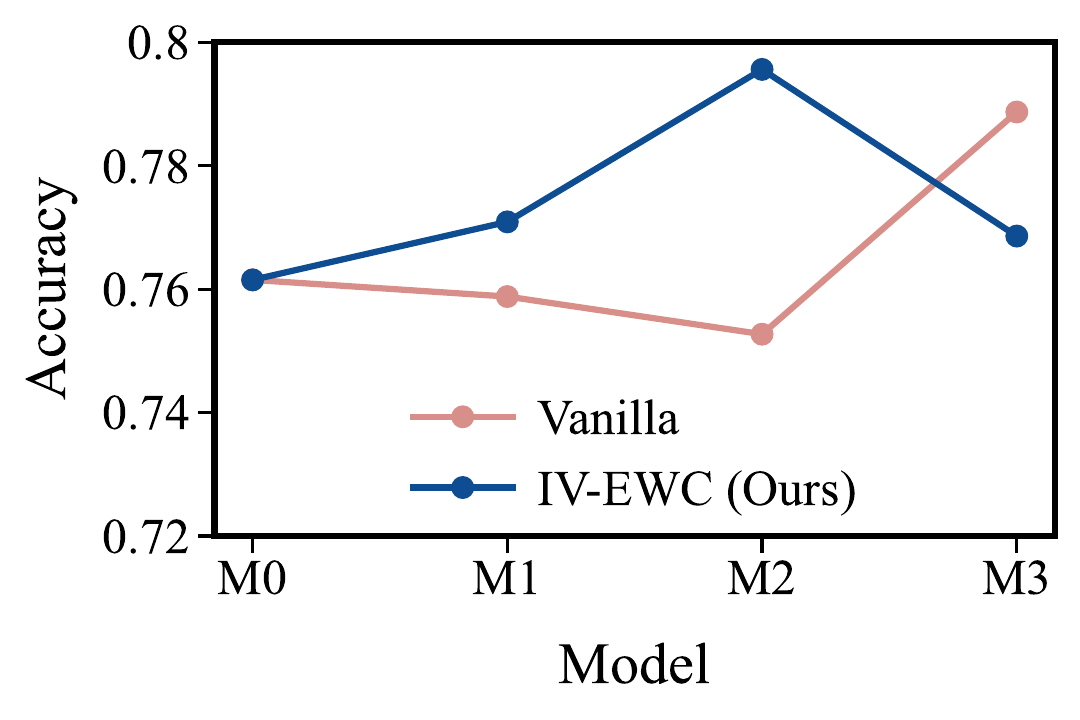}
  \caption{Level 3}
  \hfill
\end{subfigure}
\begin{subfigure}[b]{0.329\textwidth}
  \centering
  \includegraphics[height=3.07cm]{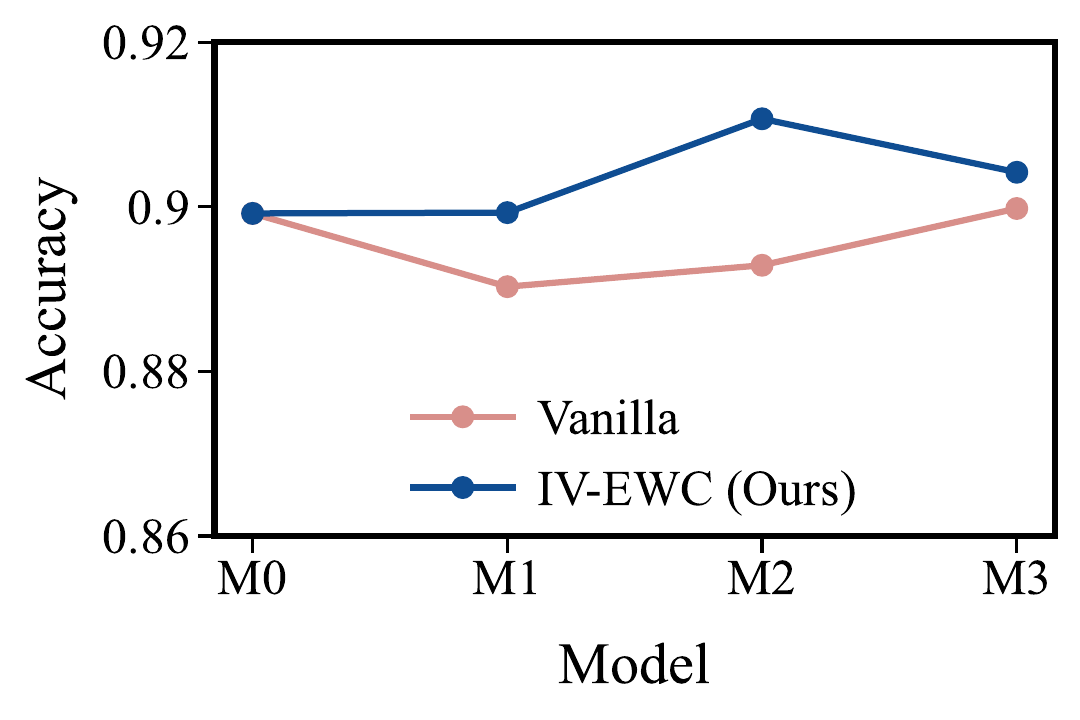}
  \caption{Overall}
  \hfill
\end{subfigure}
\caption{Accuracy of models trained after each task using vanilla curriculum (dashed line) and \textsf{IV-EWC} (solid line) on GSM8K dataset based on Qwen2.5-7B-Instruct, evaluated on the test sets of the four curriculum tasks.}
\label{fg:TrainingProcessGSM8K7B}
\end{figure}

Figure \ref{fg:TrainingProcessMATH7B} presents the performance of the vanilla curriculum and \textsf{IV-EWC} on the MATH dataset using Qwen2.5-7B-Instruct. Compared with the Qwen2.5-3B-Instruct setting, the training dynamics are less stable and the overall trend is less pronounced. Nevertheless, accuracy on previously learned tasks declines after training on each new task. This decline is smaller under \textsf{IV-EWC} than under the vanilla curriculum, indicating reduced forgetting for \textsf{IV-EWC}.


Figure \ref{fg:TrainingProcessGSM8K7B} compares the vanilla curriculum against \textsf{IV-EWC} on the GSM8K dataset using Qwen2.5-7B-Instruct. By leveraging the reasoning capability of a larger model, both approaches achieve test accuracy consistently above 0.89. Given this strong baseline, incorporating regularization further improves accuracy and reduces forgetting relative to curriculum learning without regularization.

\begin{figure}[t]
\centering
\begin{subfigure}[b]{0.329\textwidth}
  \centering
  \includegraphics[height=3.07cm]{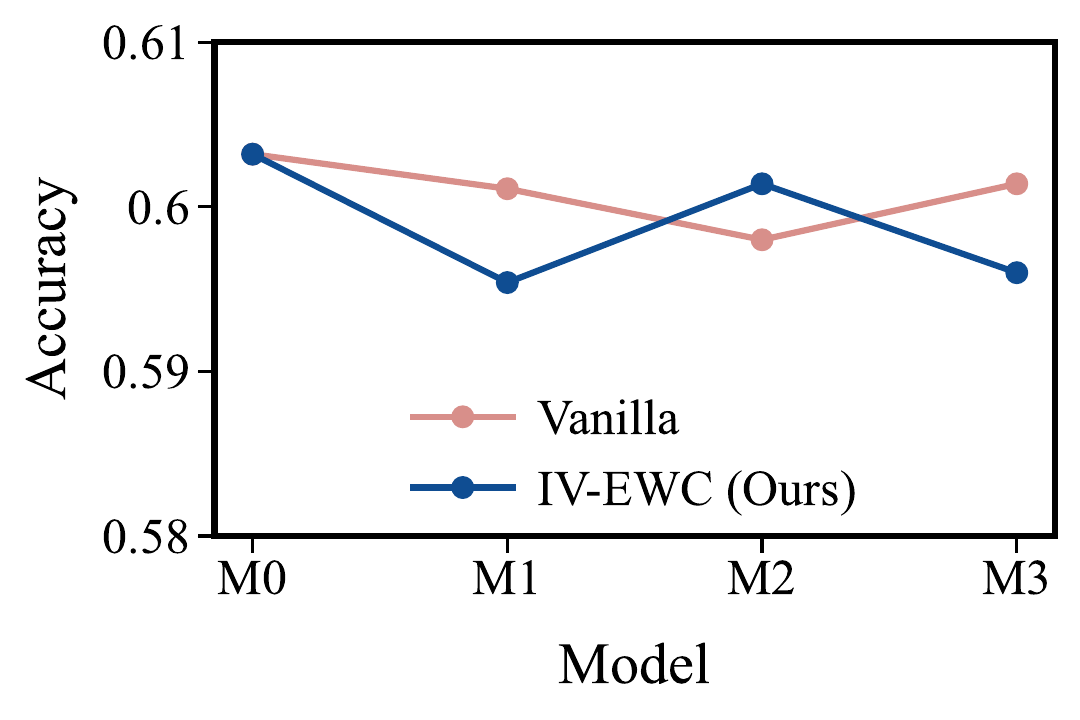}
  \caption{Level 0}
  \hfill
\end{subfigure}
\begin{subfigure}[b]{0.329\textwidth}
  \centering
  \includegraphics[height=3.07cm]{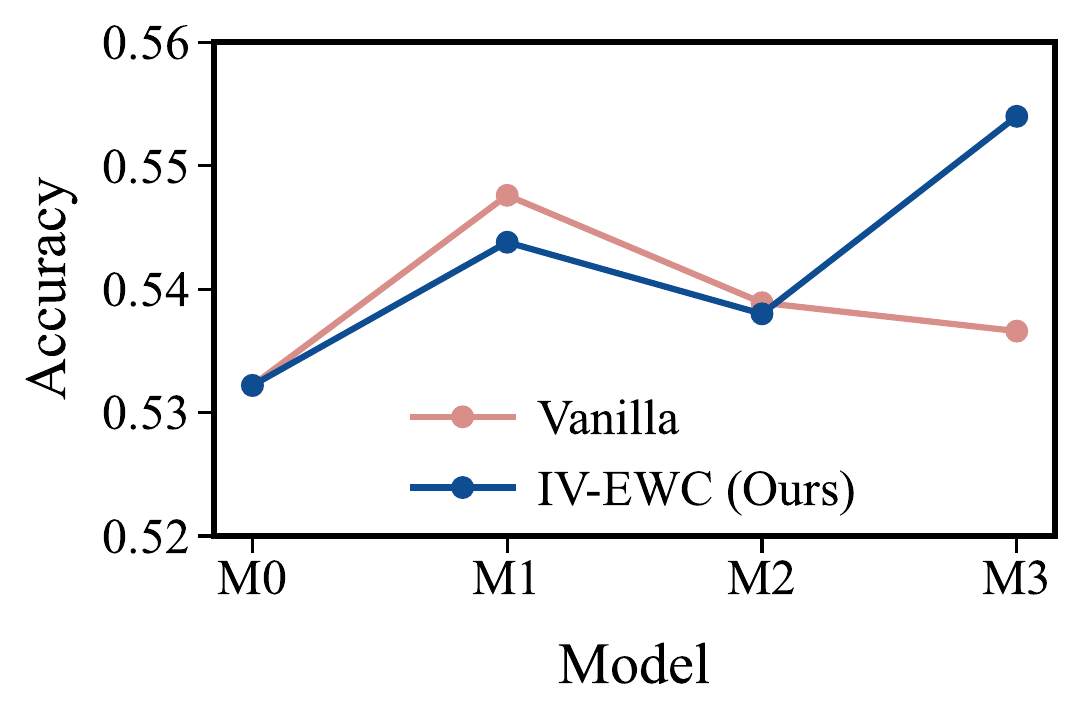}
  \caption{Level 1}
  \hfill
\end{subfigure}
\begin{subfigure}[b]{0.329\textwidth}
  \centering
  \includegraphics[height=3.07cm]{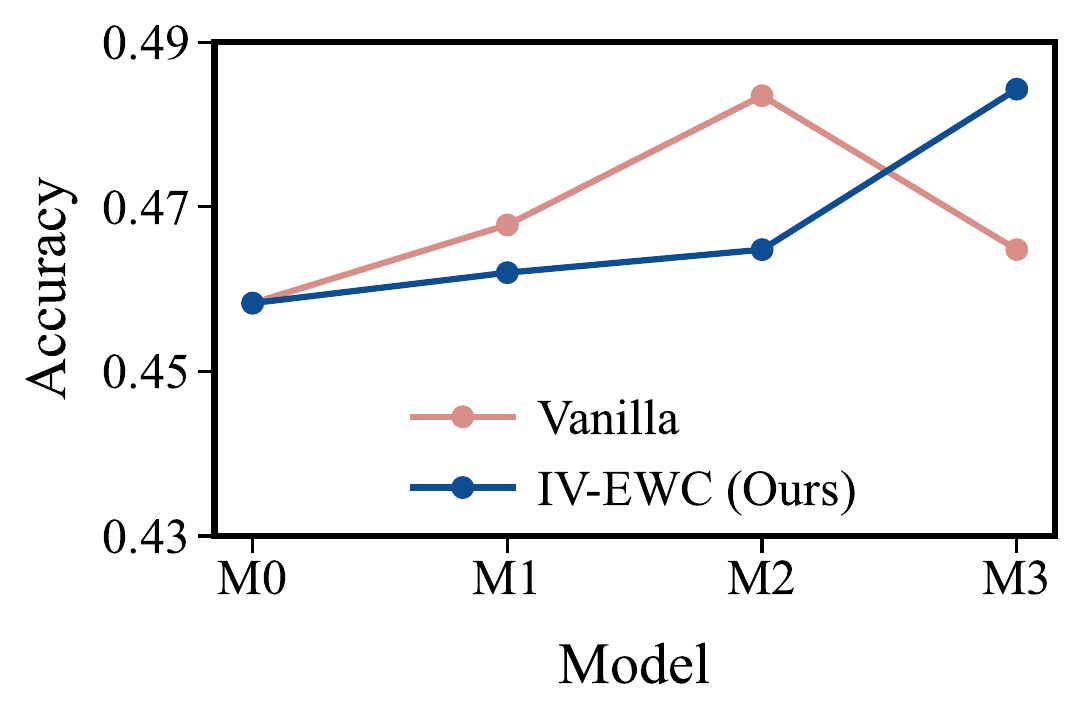}
  \caption{Level 2}
  \hfill
\end{subfigure}
\\
\begin{subfigure}[b]{0.329\textwidth}
  \centering
  \includegraphics[height=3.07cm]{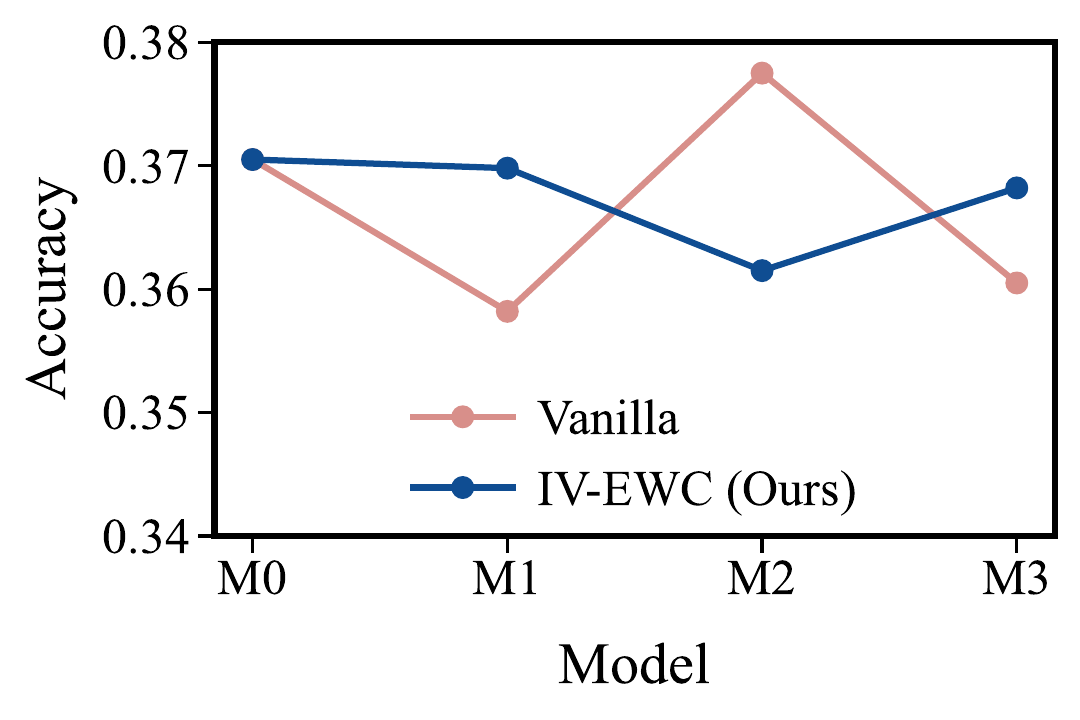}
  \caption{Level 3}
  \hfill
\end{subfigure}
\begin{subfigure}[b]{0.329\textwidth}
  \centering
  \includegraphics[height=3.07cm]{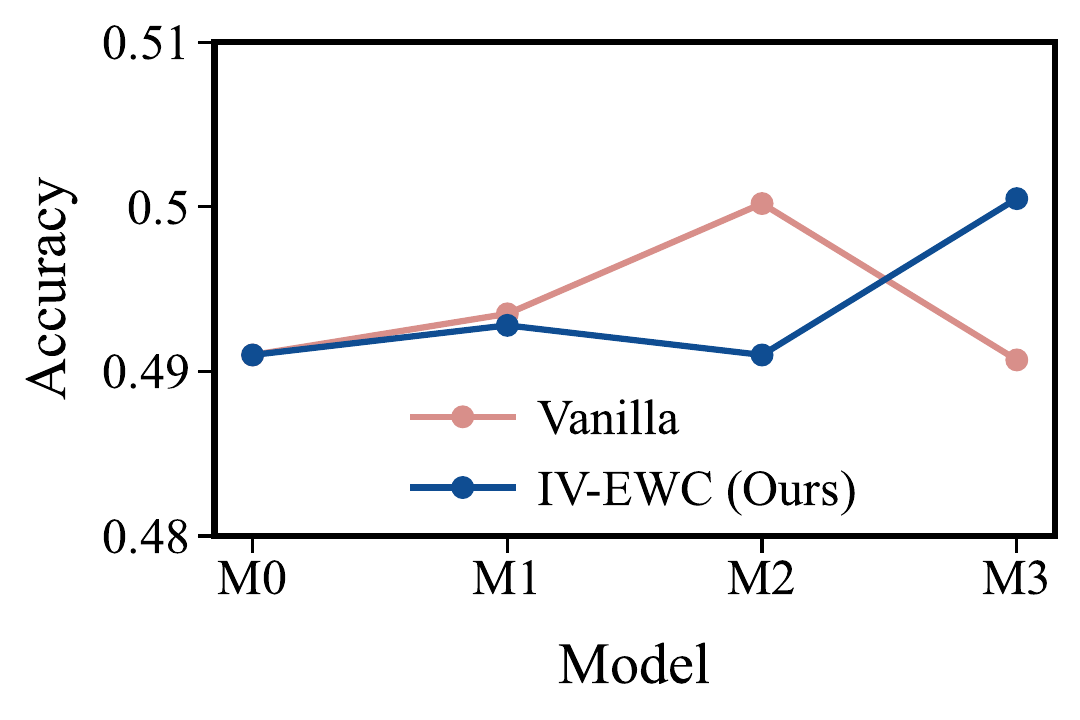}
  \caption{Overall}
  \hfill
\end{subfigure}
\caption{Accuracy of models trained after each task using vanilla curriculum (dashed line) and \textsf{IV-EWC} (solid line) on cn\_k12 dataset based on Qwen2.5-7B-Instruct, evaluated on the test sets of the four curriculum tasks.}
\label{fg:TrainingProcessCNK127B}
\end{figure}

\begin{figure}[t]
\centering
\begin{subfigure}[b]{0.329\textwidth}
  \centering
  \includegraphics[height=3.07cm]{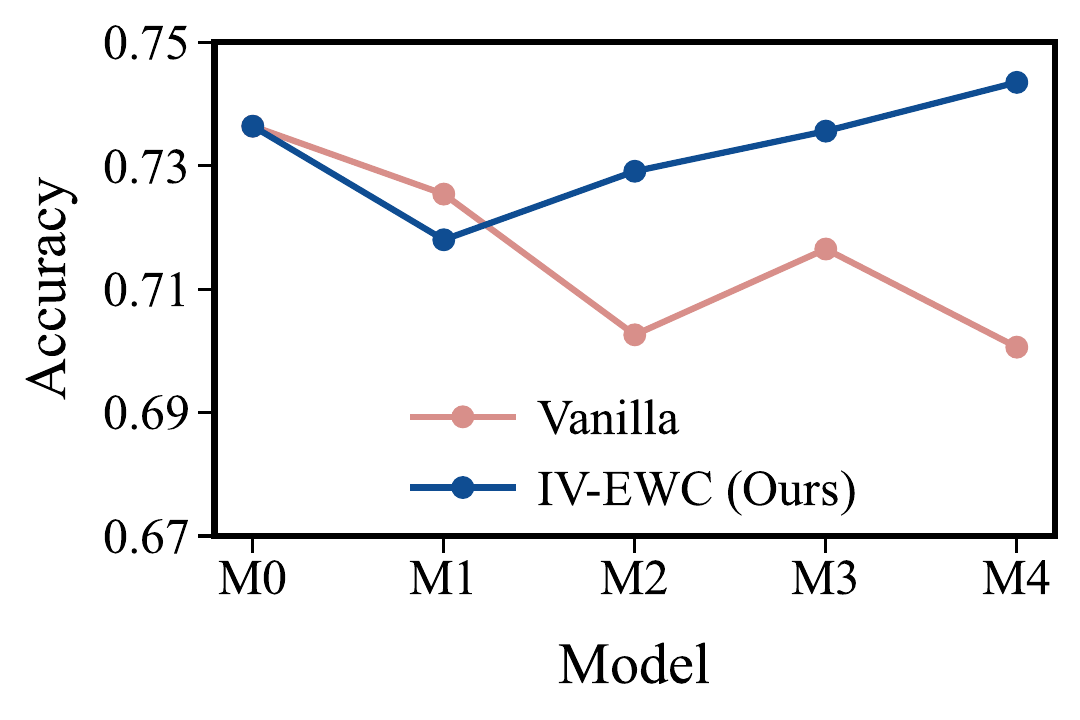}
  \caption{Level 0}
  \hfill
\end{subfigure}
\begin{subfigure}[b]{0.329\textwidth}
  \centering
  \includegraphics[height=3.07cm]{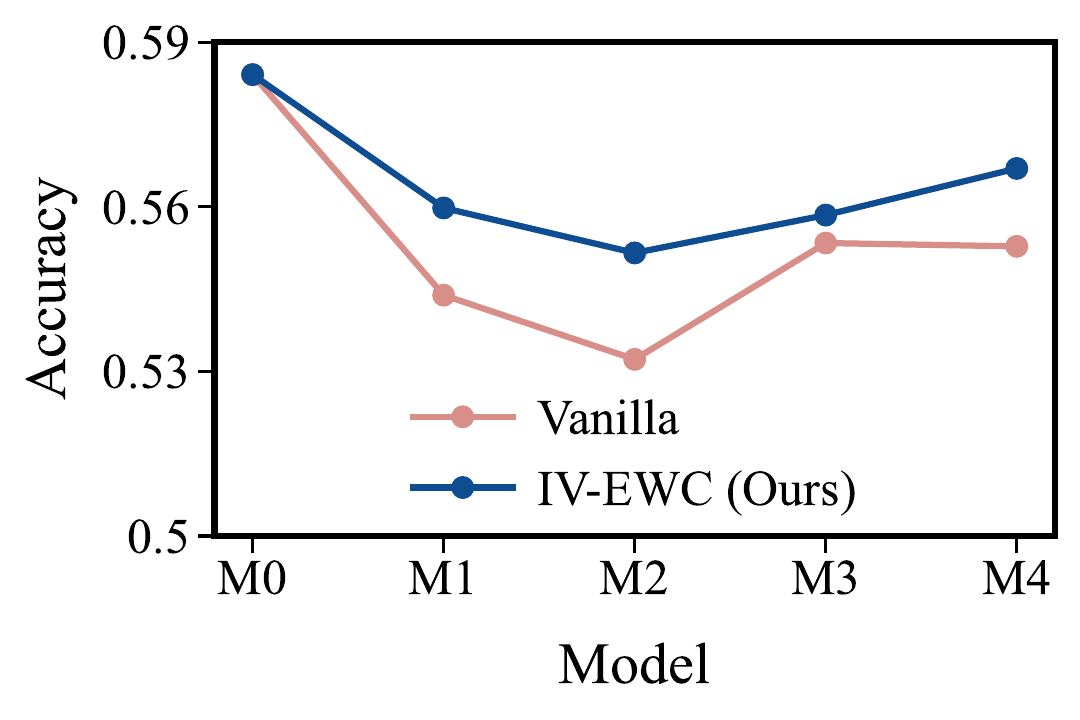}
  \caption{Level 1}
  \hfill
\end{subfigure}
\begin{subfigure}[b]{0.329\textwidth}
  \centering
  \includegraphics[height=3.07cm]{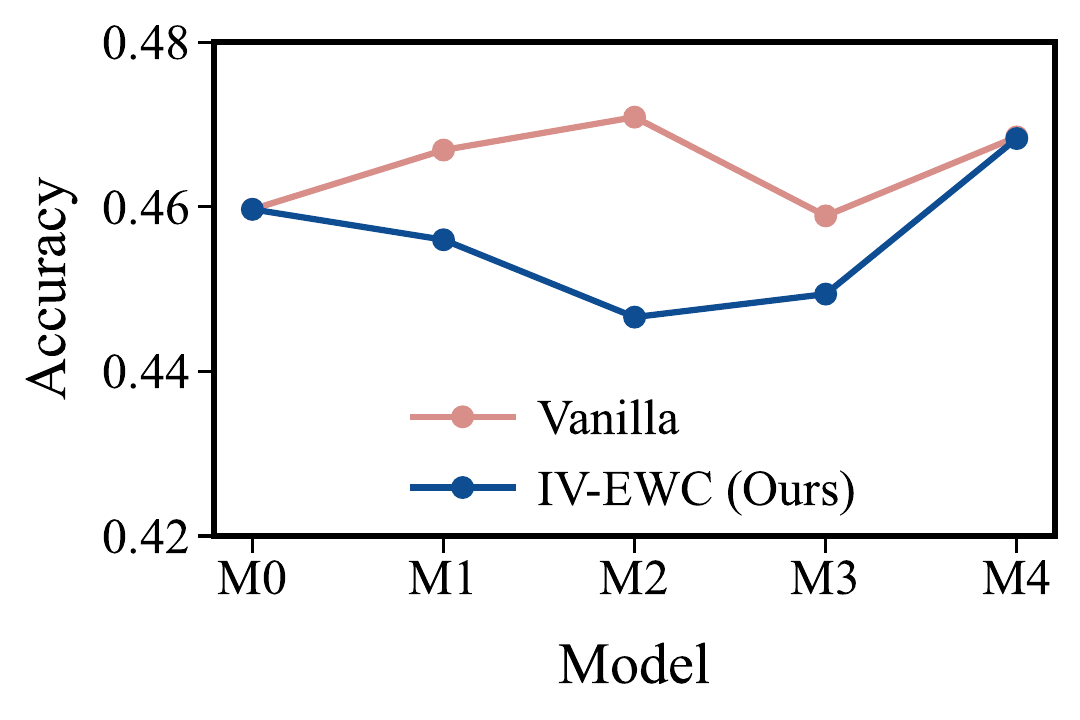}
  \caption{Level 2}
  \hfill
\end{subfigure}
\\
\begin{subfigure}[b]{0.329\textwidth}
  \centering
  \includegraphics[height=3.07cm]{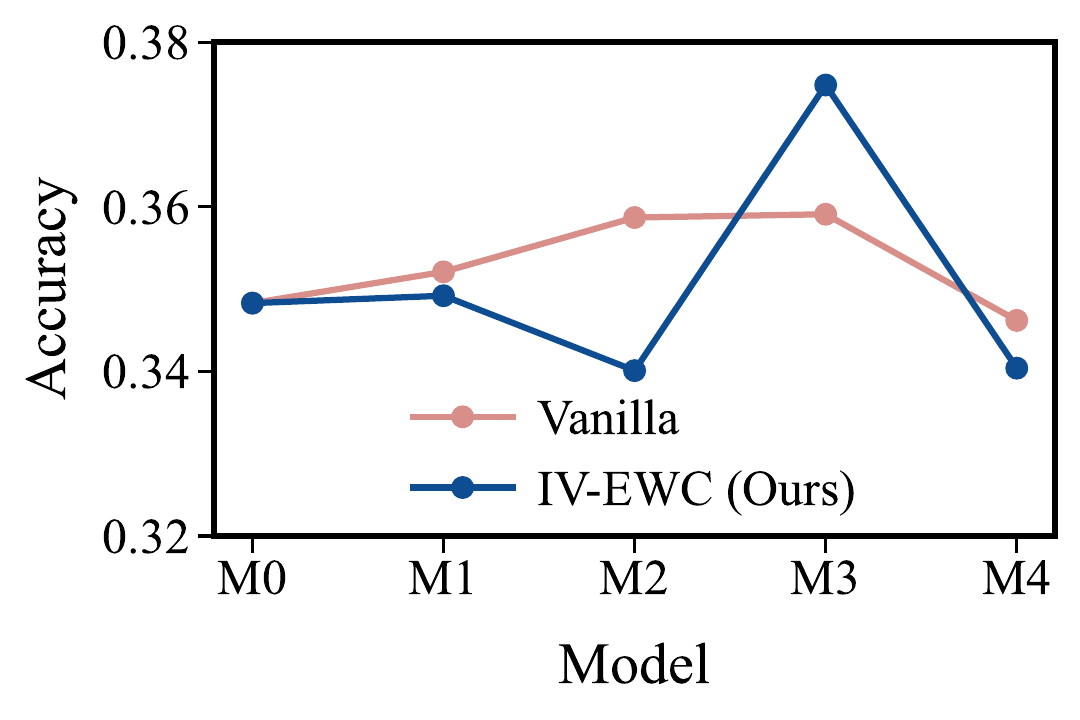}
  \caption{Level 3}
  \hfill
\end{subfigure}
\begin{subfigure}[b]{0.329\textwidth}
  \centering
  \includegraphics[height=3.07cm]{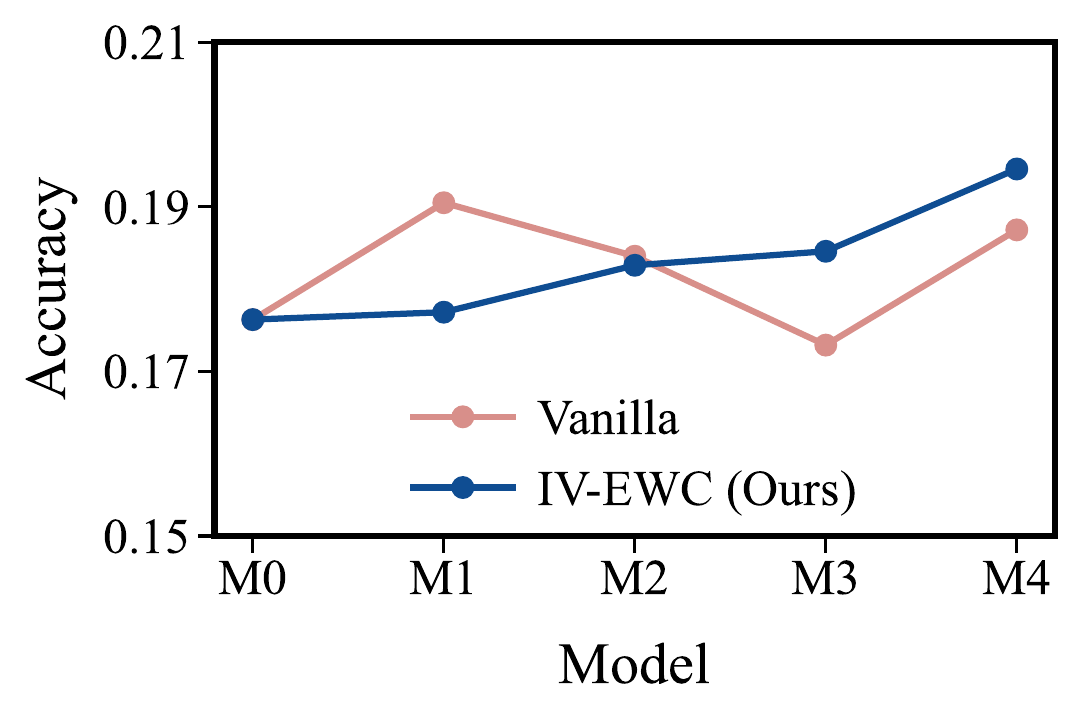}
  \caption{Level 4}
  \hfill
\end{subfigure}
\begin{subfigure}[b]{0.329\textwidth}
  \centering
  \includegraphics[height=3.07cm]{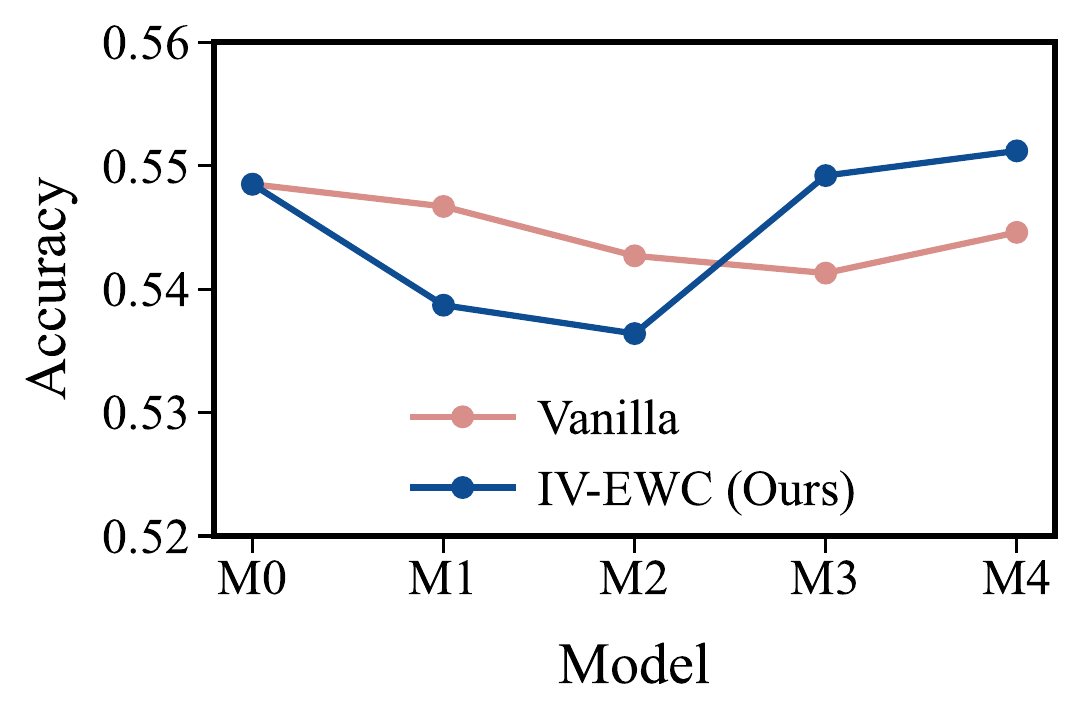}
  \caption{Overall}
  \hfill
\end{subfigure}
\caption{Accuracy of models trained after each task using vanilla curriculum (dashed line) and \textsf{IV-EWC} (solid line) on the MATH dataset based on Llama-3.2-3B-Instruct, evaluated on the test sets of the five curriculum tasks.}
\label{fg:TrainingProcessMATHLlama}
\end{figure}

Figure \ref{fg:TrainingProcessCNK127B} presents results for the vanilla curriculum baseline and \textsf{IV-EWC} on the cn\_k12 dataset with Qwen2.5-7B-Instruct. For the baseline, test accuracy on a specific task improves during its training phase but deteriorates as training proceeds to subsequent tasks. With regularization, \textsf{IV-EWC} achieves and sustains higher accuracy on previously learned tasks, mitigating forgetting.

\begin{figure}[t]
\centering
\begin{subfigure}[b]{0.329\textwidth}
  \centering
  \includegraphics[height=3.07cm]{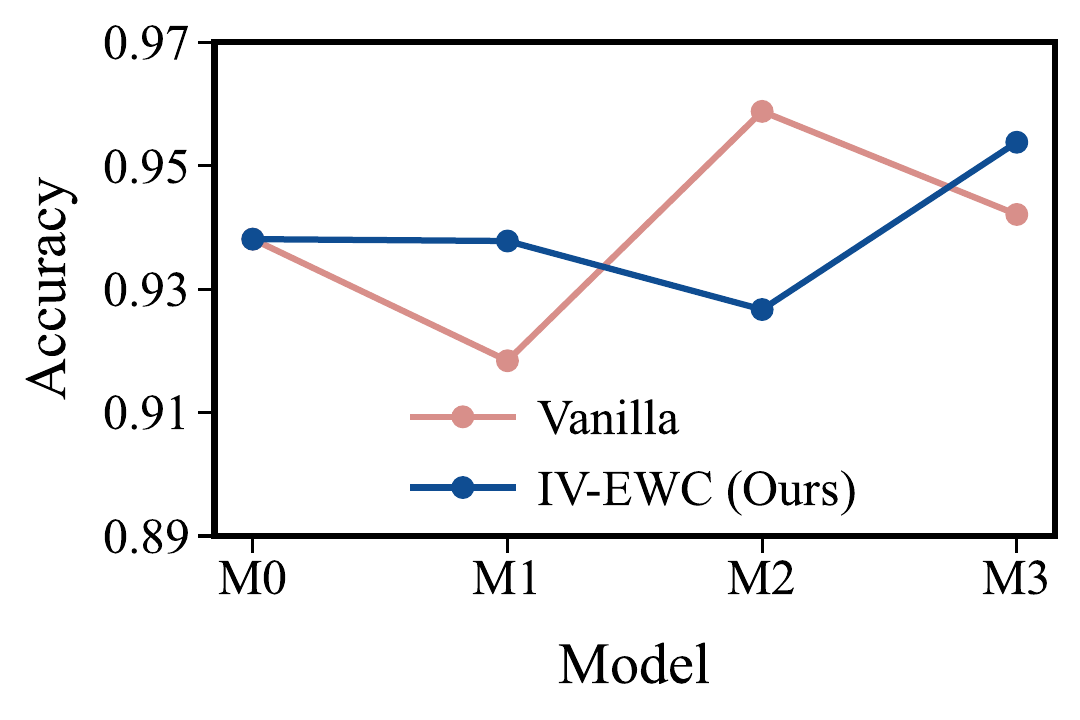}
  \caption{Level 0}
  \hfill
\end{subfigure}
\begin{subfigure}[b]{0.329\textwidth}
  \centering
  \includegraphics[height=3.07cm]{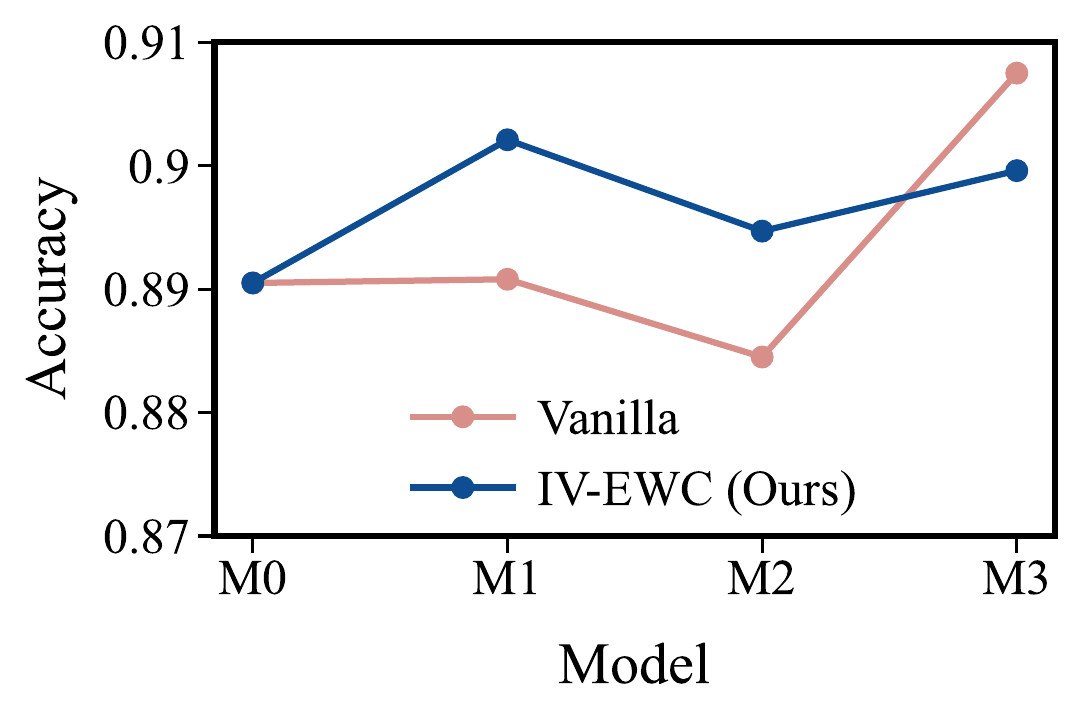}
  \caption{Level 1}
  \hfill
\end{subfigure}
\begin{subfigure}[b]{0.329\textwidth}
  \centering
  \includegraphics[height=3.07cm]{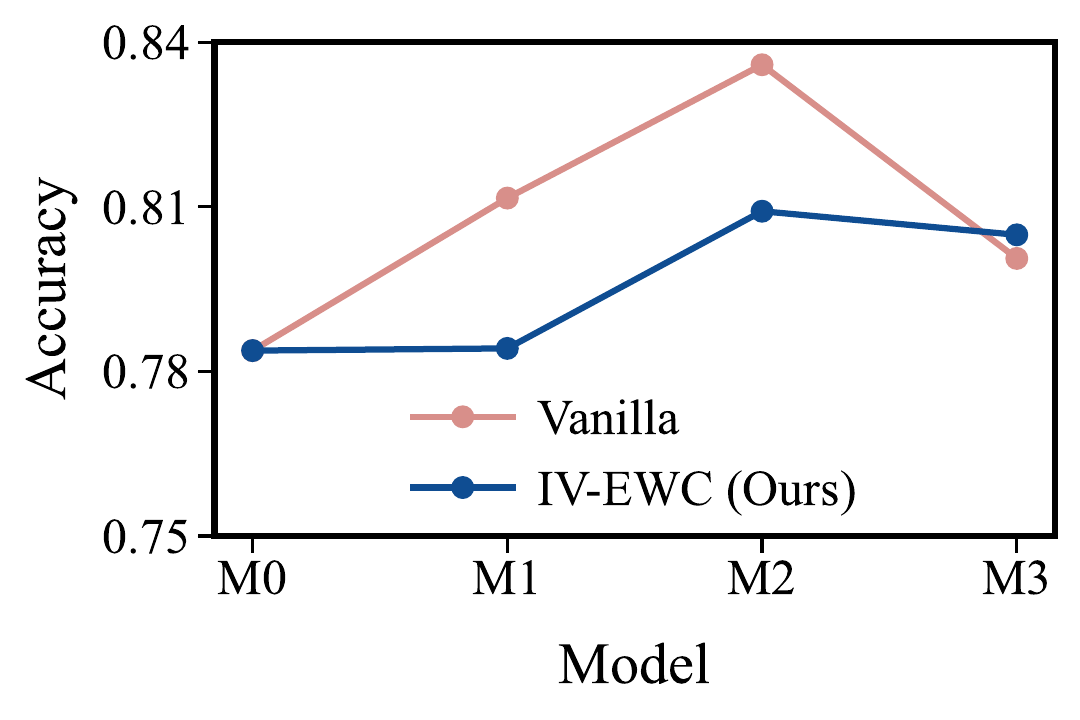}
  \caption{Level 2}
  \hfill
\end{subfigure}
\\
\begin{subfigure}[b]{0.329\textwidth}
  \centering
  \includegraphics[height=3.07cm]{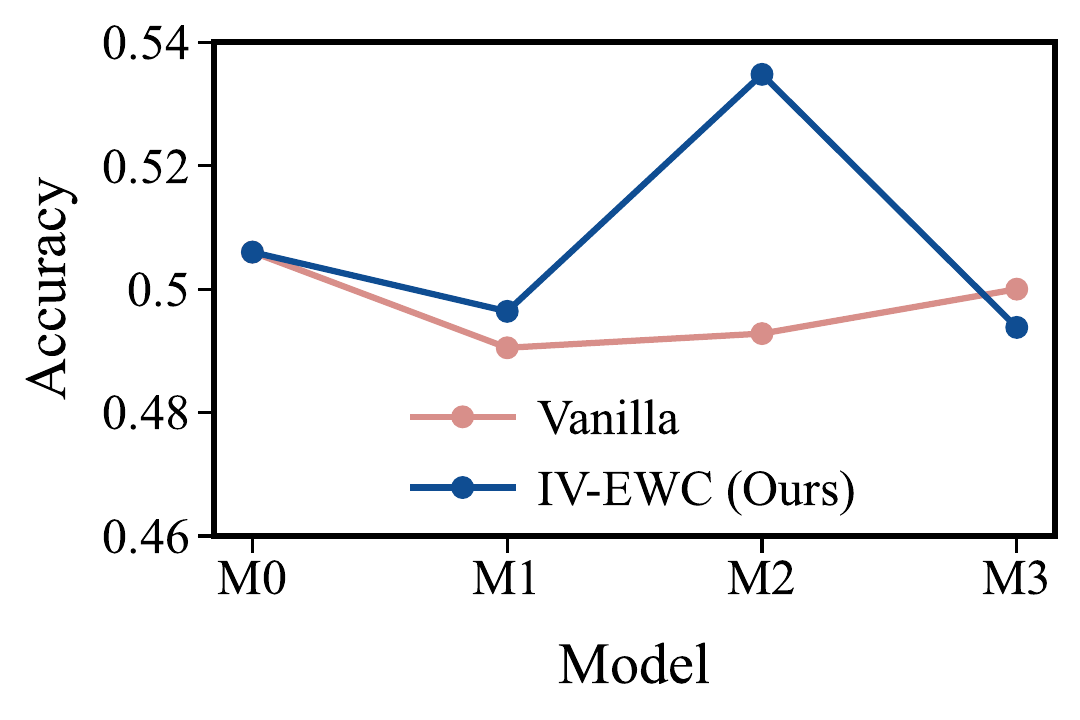}
  \caption{Level 3}
  \hfill
\end{subfigure}
\begin{subfigure}[b]{0.329\textwidth}
  \centering
  \includegraphics[height=3.07cm]{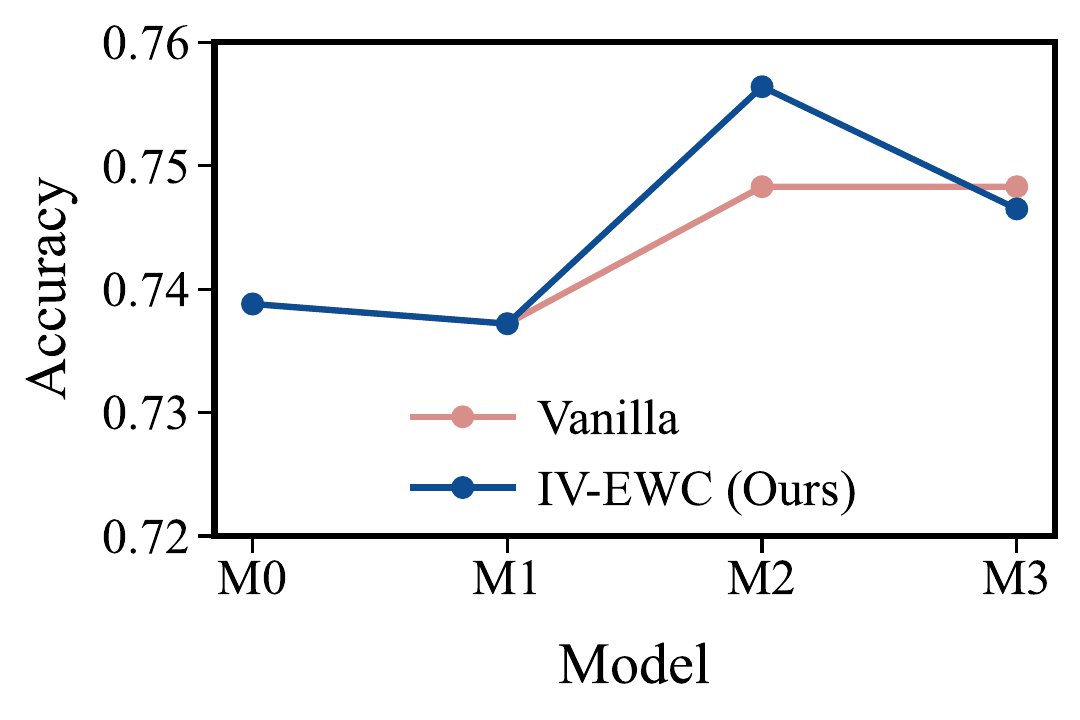}
  \caption{Overall}
  \hfill
\end{subfigure}
\caption{Accuracy of models trained after each task using vanilla curriculum (dashed line) and \textsf{IV-EWC} (solid line) on GSM8K dataset based on Llama-3.2-3B-Instruct, evaluated on the test sets of the four curriculum tasks.}
\label{fg:TrainingProcessGSM8KLlama}
\end{figure}

\begin{figure}[t]
\centering
\begin{subfigure}[b]{0.329\textwidth}
  \centering
  \includegraphics[height=3.07cm]{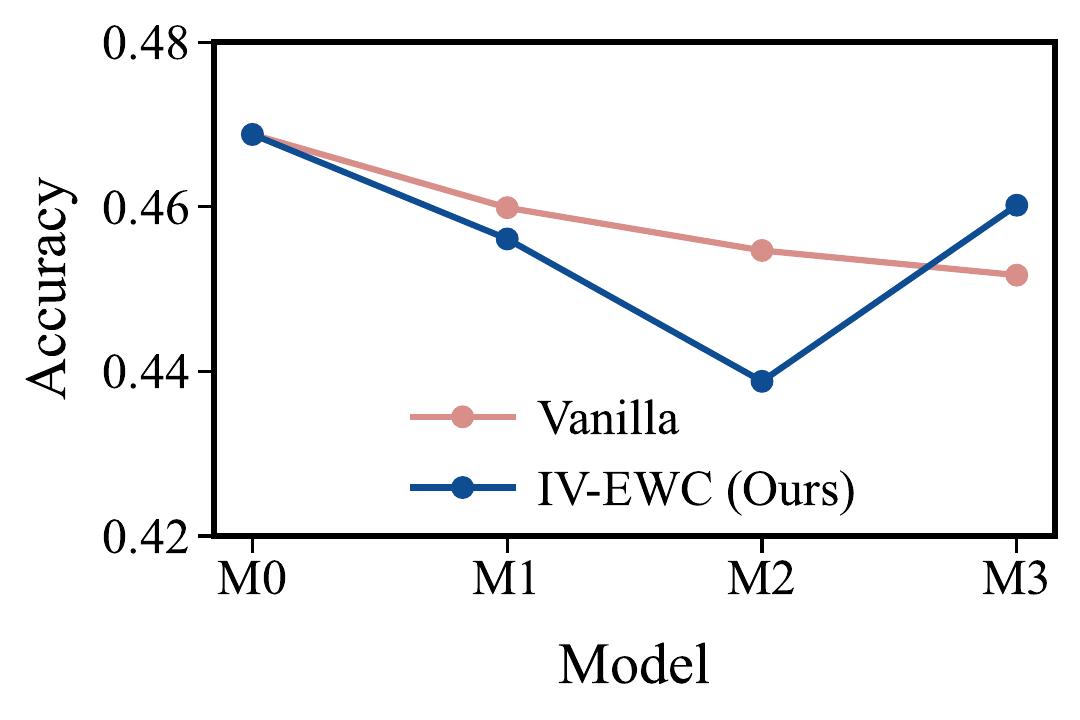}
  \caption{Level 0}
  \hfill
\end{subfigure}
\begin{subfigure}[b]{0.329\textwidth}
  \centering
  \includegraphics[height=3.07cm]{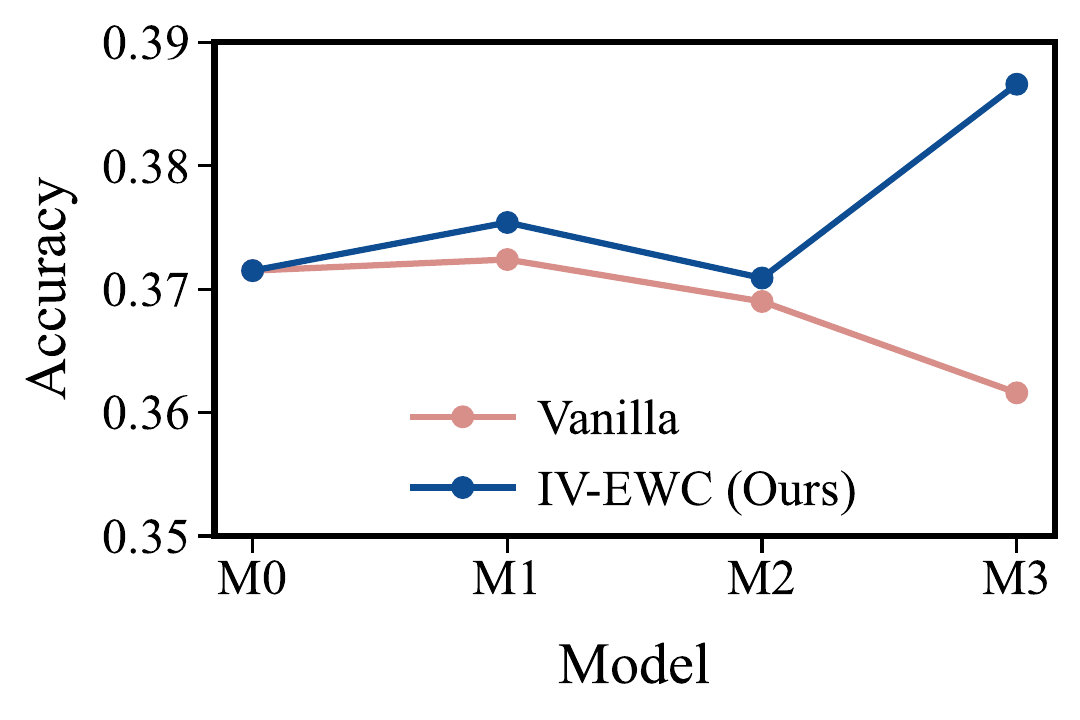}
  \caption{Level 1}
  \hfill
\end{subfigure}
\begin{subfigure}[b]{0.329\textwidth}
  \centering
  \includegraphics[height=3.07cm]{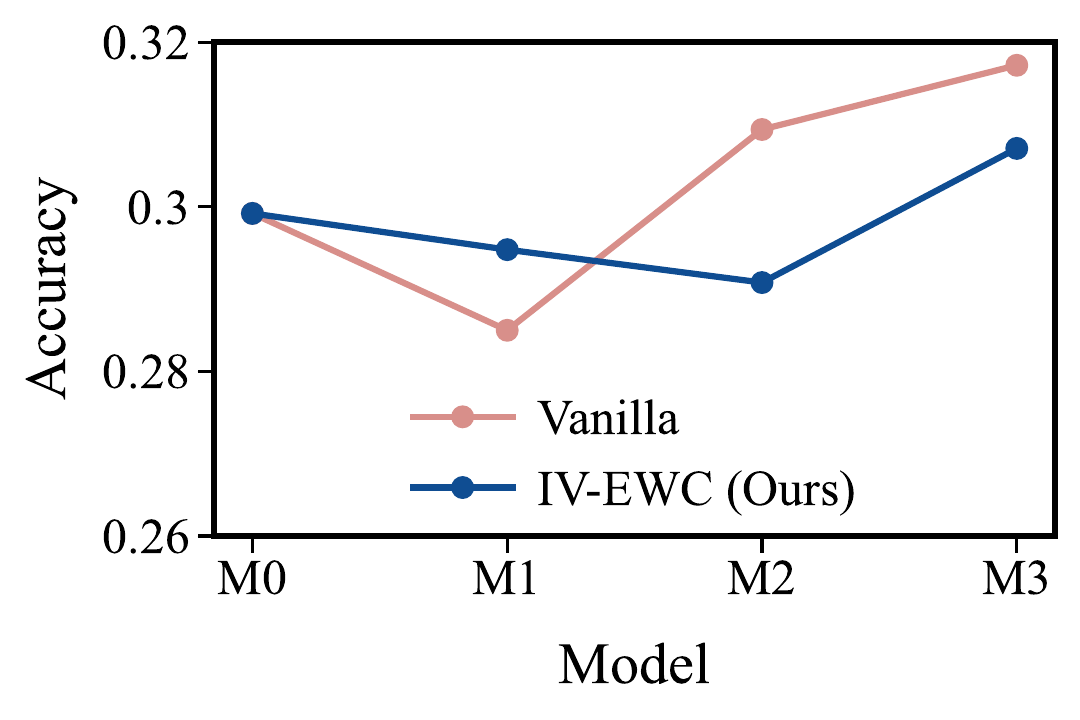}
  \caption{Level 2}
  \hfill
\end{subfigure}
\\
\begin{subfigure}[b]{0.329\textwidth}
  \centering
  \includegraphics[height=3.07cm]{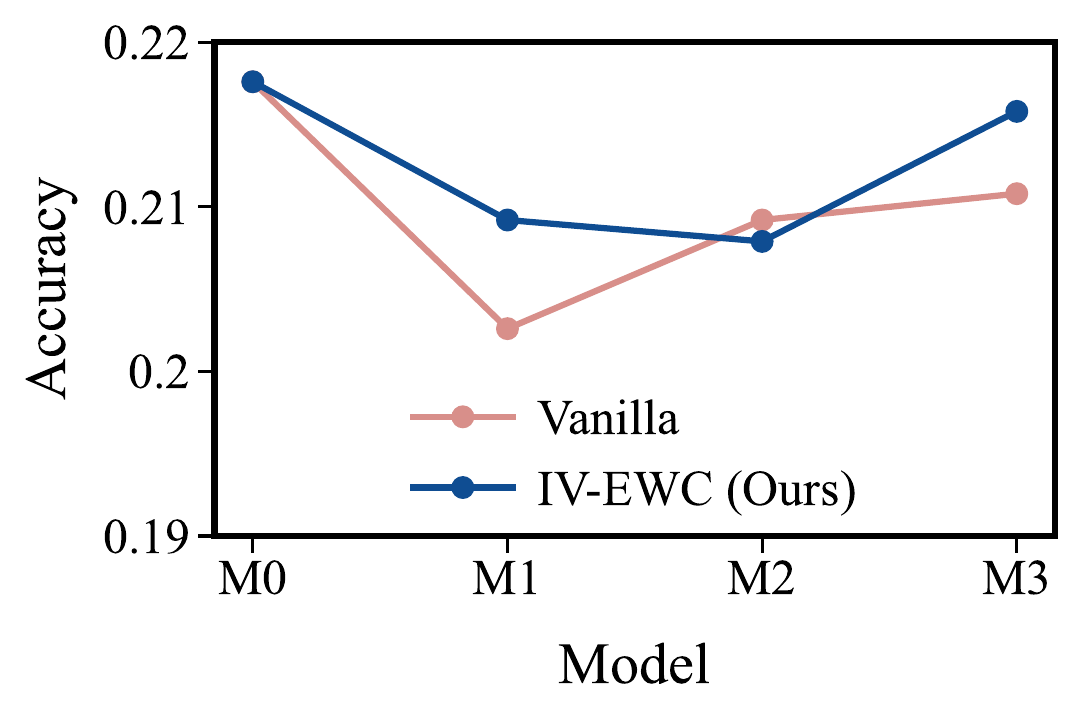}
  \caption{Level 3}
  \hfill
\end{subfigure}
\begin{subfigure}[b]{0.329\textwidth}
  \centering
  \includegraphics[height=3.07cm]{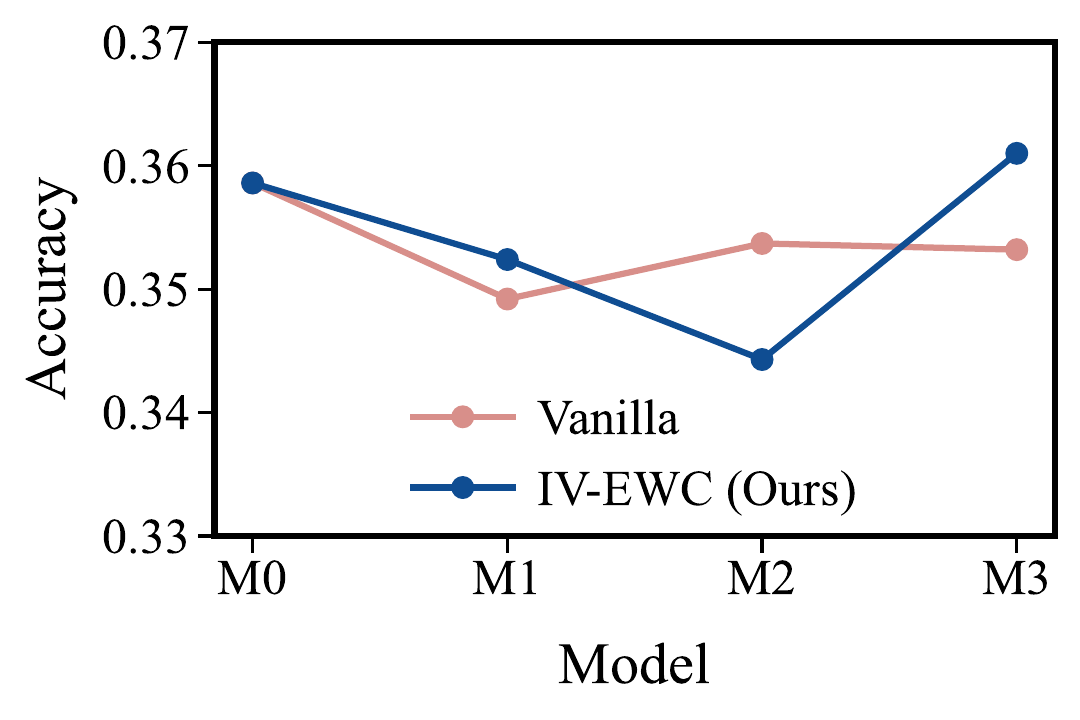}
  \caption{Overall}
  \hfill
\end{subfigure}
\caption{Accuracy of models trained after each task using vanilla curriculum (dashed line) and \textsf{IV-EWC} (solid line) on cn\_k12 dataset based on Llama-3.2-3B-Instruct, evaluated on the test sets of the four curriculum tasks.}
\label{fg:TrainingProcessCNK12Llama}
\end{figure}


Figure \ref{fg:TrainingProcessMATHLlama} presents results on the MATH dataset with Llama-3.2-3B-Instruct, comparing a vanilla curriculum learning baseline to \textsf{IV-EWC}. Relative to Qwen, Llama exhibits lower overall performance, likely due to differences in pretraining (Qwen models are trained on math-reasoning data). We further observe performance degradation after training on the current curriculum task. Notably, beginning at the task of Level 3, the ranking reverses and \textsf{IV-EWC} outperforms the baseline, suggesting that its benefits are more pronounced on higher-difficulty tasks.


Figure \ref{fg:TrainingProcessGSM8KLlama} shows the results of vanilla curriculum and \textsf{IV-EWC} on GSM8K dataset with Llama-3.2-3B-Instruct. Vanilla curriculum learning exhibits limited forgetting under this setting, likely owing to the homogeneity of the levels. Nevertheless, \textsf{IV-EWC} can improve training stability, which is evidenced by higher mid-curriculum accuracy, stronger forward transfer to the hardest level and reduced forgetting on earlier levels.


Figure \ref{fg:TrainingProcessCNK12Llama} reports cn\_k12 results with Llama-3.2-3B-Instruct for the vanilla curriculum baseline and \textsf{IV-EWC}. \textsf{IV-EWC} yields higher end-of-curriculum accuracy. The gains are driven by better retention of earlier tasks: at the end of training, Level 0 accuracy is 0.4602 vs 0.4517 for vanilla, and Level 1 improves to 0.3866 vs 0.3616, indicating reduced forgetting and positive backward transfer.

\section{Parameter Sensitivity Analysis}
\label{sec:ParameterSensitivity}

\begin{figure}[t]
\centering
\begin{subfigure}[b]{0.329\textwidth}
  \centering
  \includegraphics[height=3.07cm]{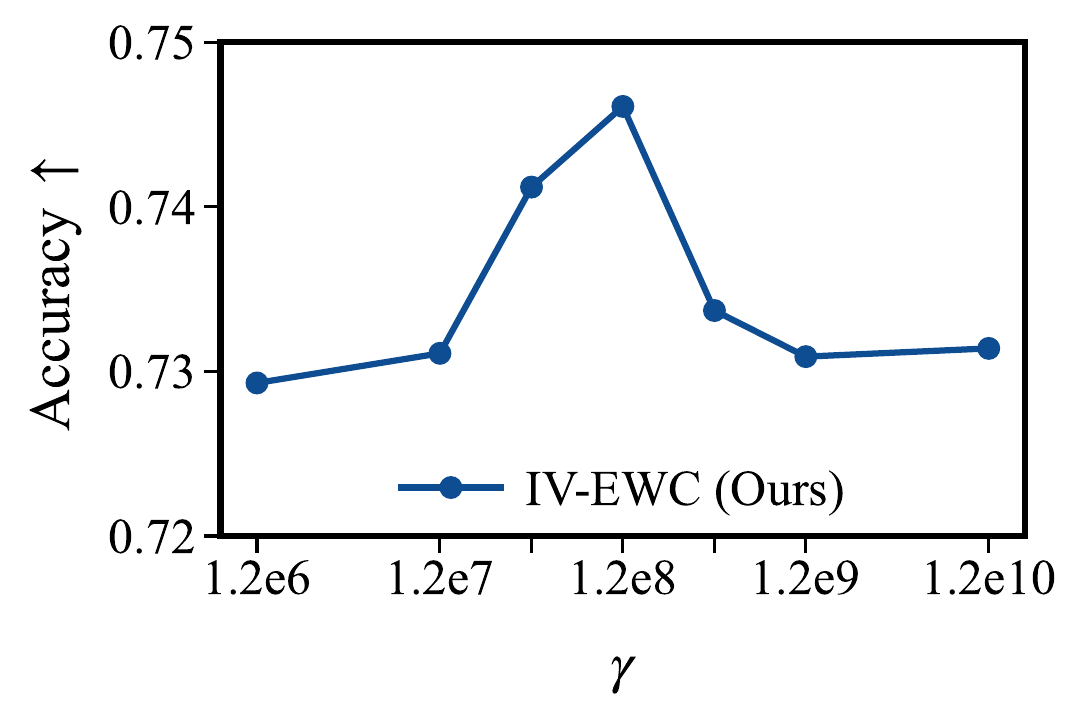}
  \caption{Accuracy}
  \hfill
\end{subfigure}
\begin{subfigure}[b]{0.329\textwidth}
  \centering
  \includegraphics[height=3.07cm]{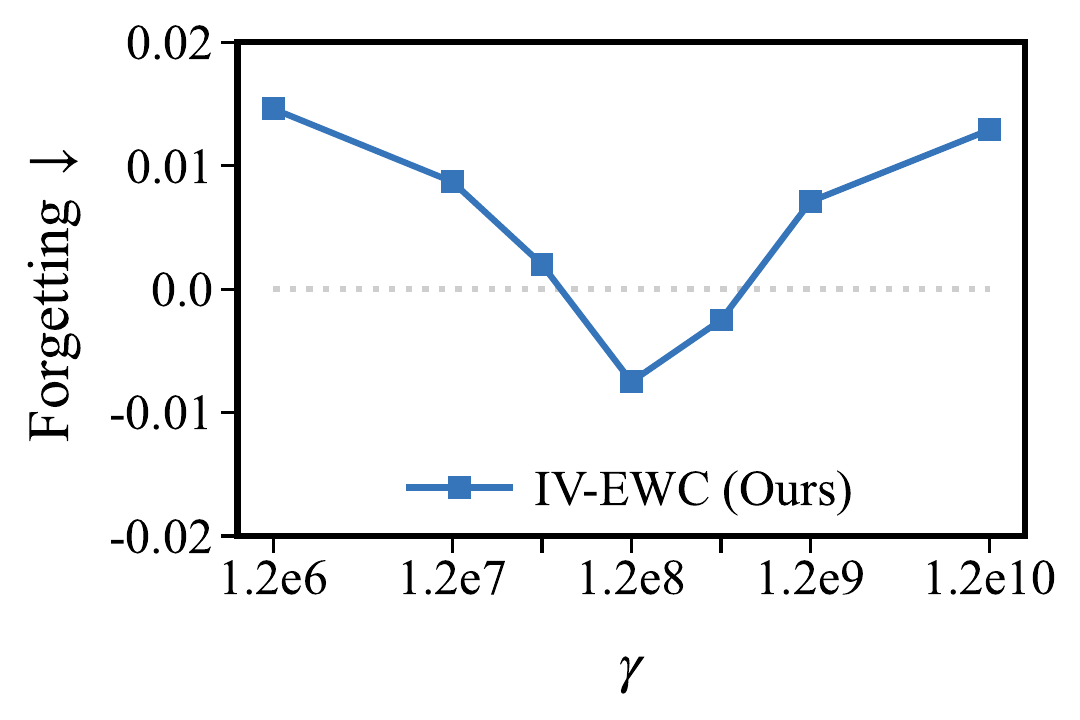}
  \caption{Forgetting}
  \hfill
\end{subfigure}
\begin{subfigure}[b]{0.329\textwidth}
  \centering
  \includegraphics[height=3.07cm]{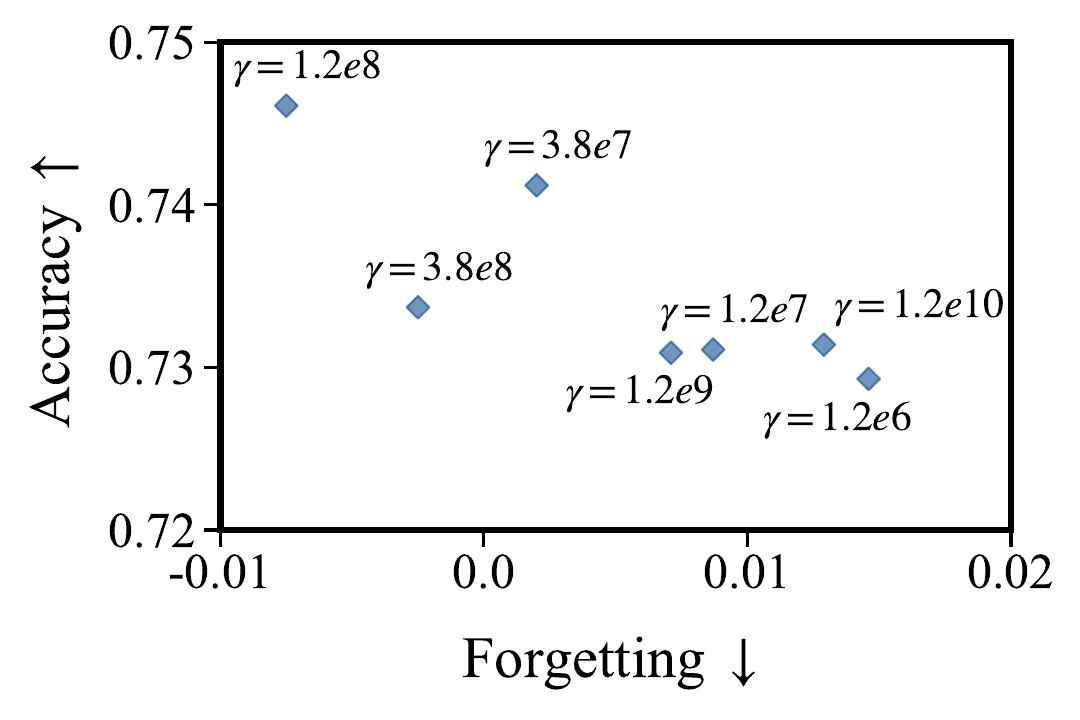}
  \caption{Trade-Off View}
  \hfill
\end{subfigure}
\caption{Performance of \textsf{IV-EWC} with different initial regularization weight $\gamma$ using Qwen2.5-3B-Instruct on the MATH dataset.}
\label{fg:WeightSensitivity}
\end{figure}

\begin{figure}[t]
\centering
\begin{subfigure}[b]{0.329\textwidth}
  \centering
  \includegraphics[height=3.07cm]{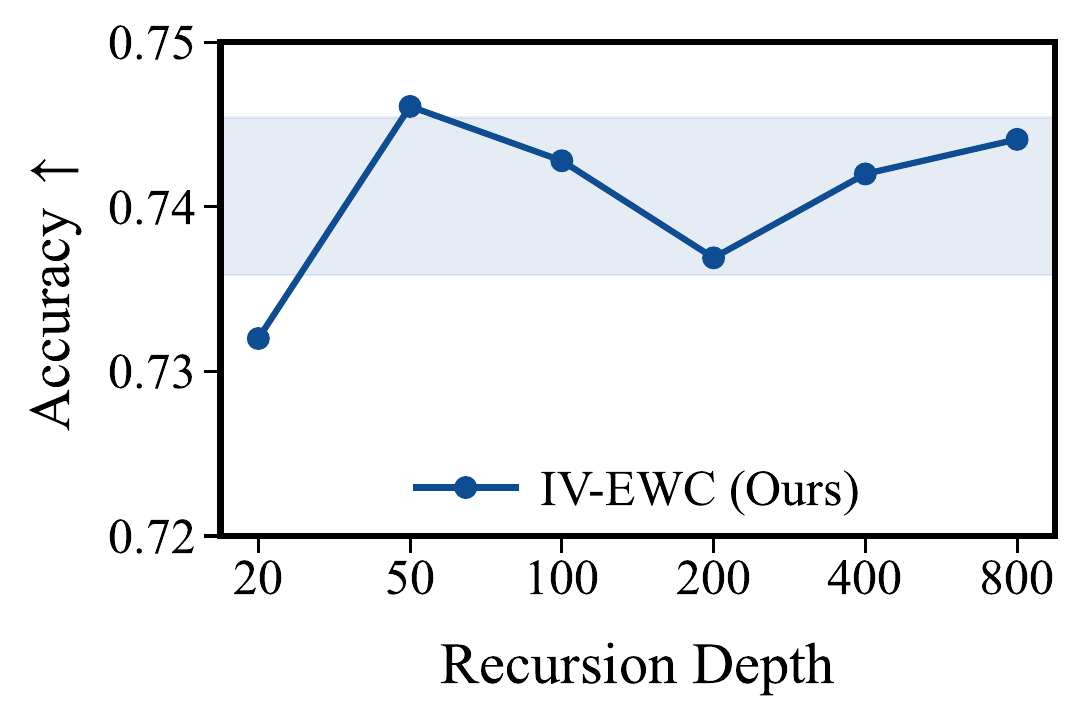}
  \caption{Accuracy}
  \hfill
\end{subfigure}
\begin{subfigure}[b]{0.329\textwidth}
  \centering
  \includegraphics[height=3.07cm]{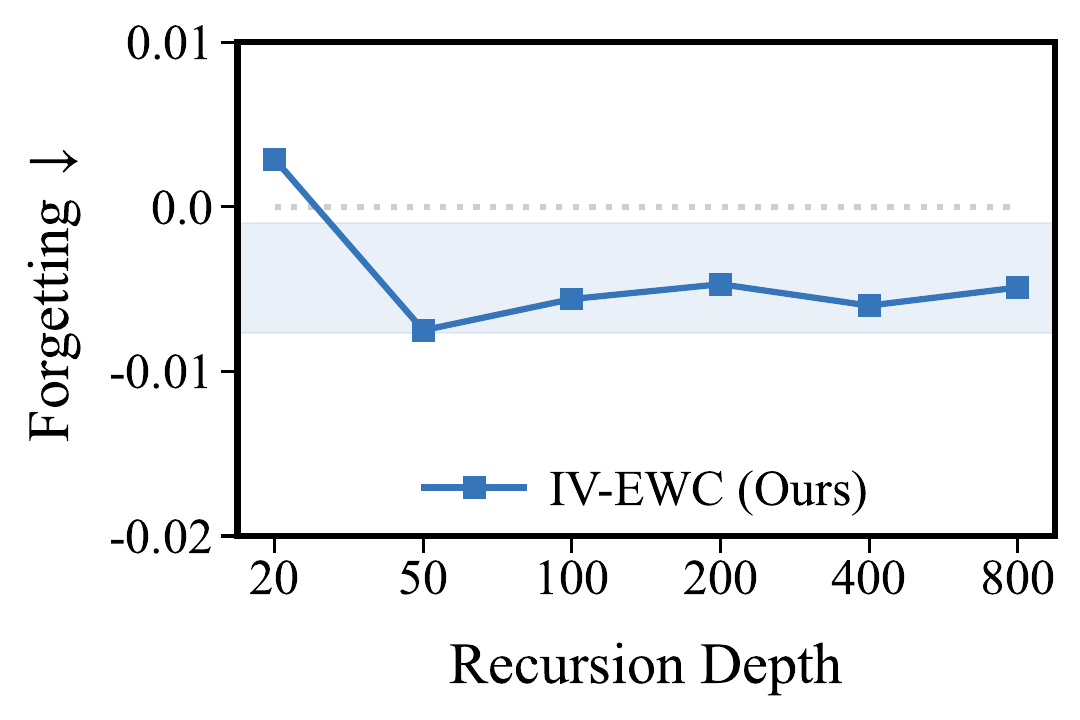}
  \caption{Forgetting}
  \hfill
\end{subfigure}
\begin{subfigure}[b]{0.329\textwidth}
  \centering
  \includegraphics[height=3.07cm]{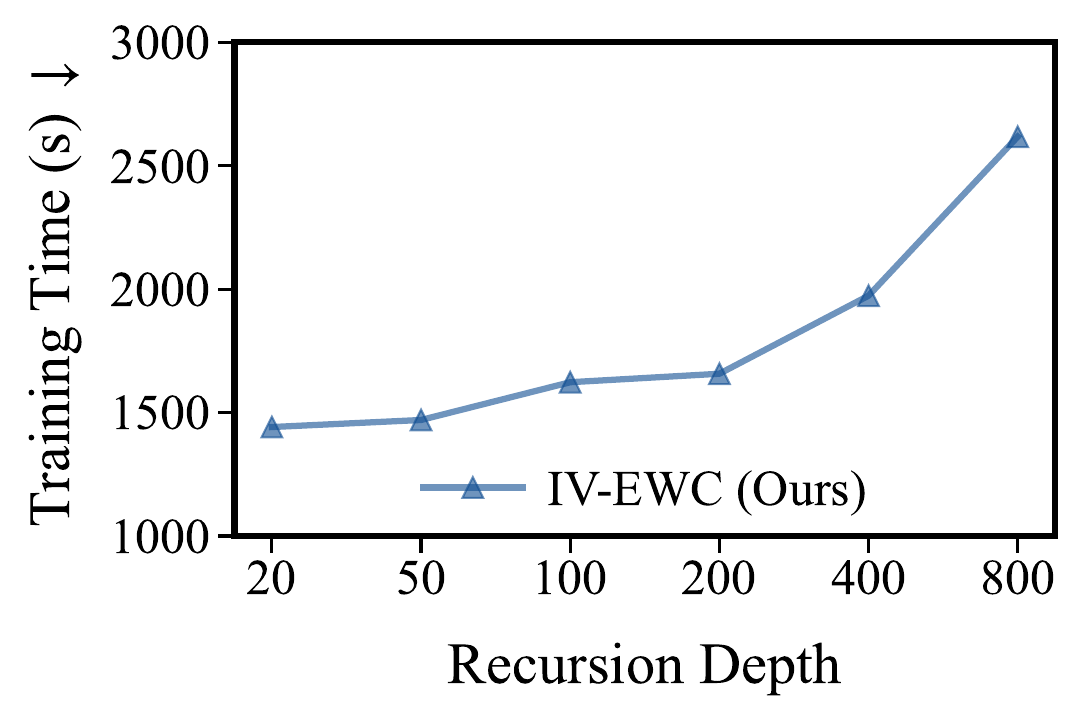}
  \caption{Training Time}
  \hfill
\end{subfigure}
\caption{Accuracy, forgetting and training time of \textsf{IV-EWC} with diverse recursive depths for LiSSA using Qwen2.5-3B-Instruct on the MATH dataset. Shaded areas of accuracy and forgetting indicate mean $\pm$ SD of all data points.}
\label{fg:IterationSensitivity}
\end{figure}

\textbf{Regularization Weight Sensitivity}.
The initial choice of regularization weight $\gamma$ is crucial for the training dynamics. In Eq. \ref{eq:ForgetDeltaSimple}, which quantifies the forgetting divergence between training all at once and EWC, $\lambda > 1$ causes the model parameters to remain close to their initial configuration. This produces a larger deviation from the global optimum and results in unchanged or increased forgetting. A sensitivity analysis of the regularization weight was conducted on this basis with the Qwen2.5-3B-Instruct model on the MATH dataset and the results are shown in Figure \ref{fg:WeightSensitivity}. The results show that forgetting increases at both insufficient and excessive regularization weights. Overall accuracy exhibits a negative association with forgetting. Therefore, choosing a regularization weight outside the suitable range does not achieve the intended effect and may be deleterious. Accordingly, we fixed $\gamma$ at 1.2e8 to best attenuate forgetting.


\textbf{LiSSA Recursion Depth Sensitivity}.
Calculating influences of samples using the influence function \citep{Koh2017InfluenceFunction} requires the computation of Hessian-vector products (HVPs), where the Linear time Stochastic Second-Order Algorithm (LiSSA) \citep{Agarwal2017LiSSA} is used. The LiSSA algorithm requires a large number of iterations to converge. But a precise Hessian-vector product is not necessary for the selection of a high-quality validation set. To determine the suitable recursion depth for \textsf{IV-EWC}, we conducted the LiSSA recursion depth sensitivity analysis with Qwen2.5-3B-Instruct on the MATH dataset, whose result is shown in Figure \ref{fg:IterationSensitivity}. Both forgetting and overall accuracy exhibit minimal variation even as recursion depth increases exponentially, whereas runtime scales approximately linearly with recursion depth. These findings indicate that \textsf{IV-EWC} is largely insensitive to the choice of recursion depth beyond 50 in LiSSA. In our experiments, we fixed the LiSSA recursion depth at 50 to save training time without impairing the effectiveness of \textsf{IV-EWC}.

\begin{figure}[t]
\centering
\begin{subfigure}[b]{0.329\textwidth}
  \centering
  \includegraphics[height=3.07cm]{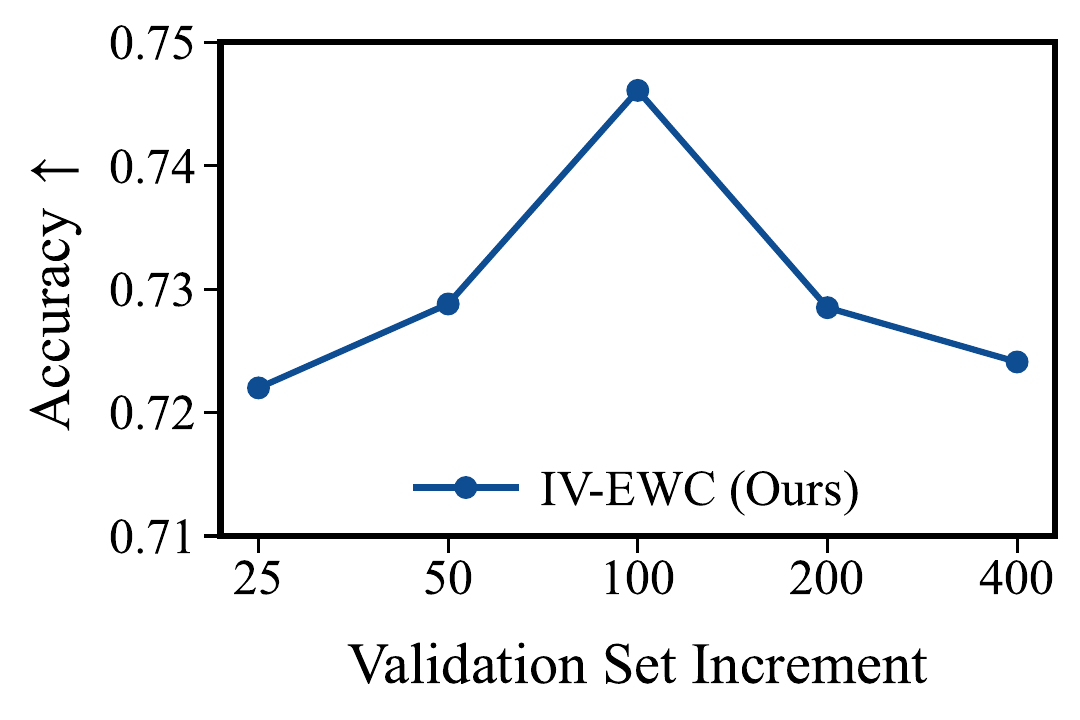}
  \caption{Accuracy}
  \hfill
\end{subfigure}
\begin{subfigure}[b]{0.329\textwidth}
  \centering
  \includegraphics[height=3.07cm]{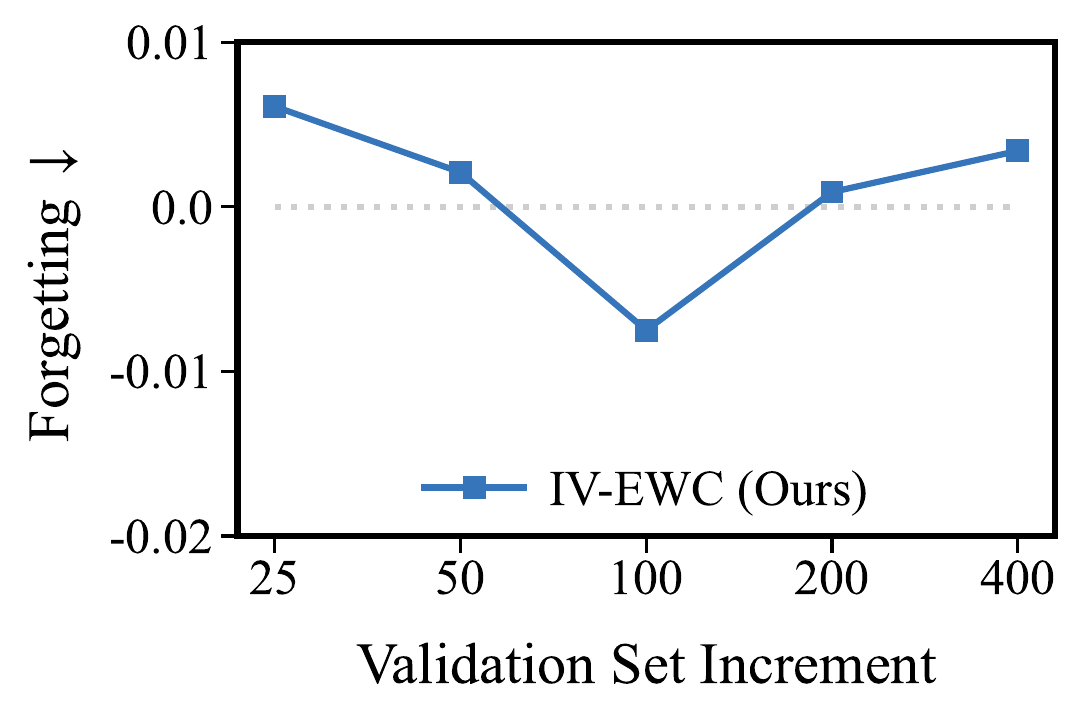}
  \caption{Forget}
  \hfill
\end{subfigure}
\begin{subfigure}[b]{0.329\textwidth}
  \centering
  \includegraphics[height=3.07cm]{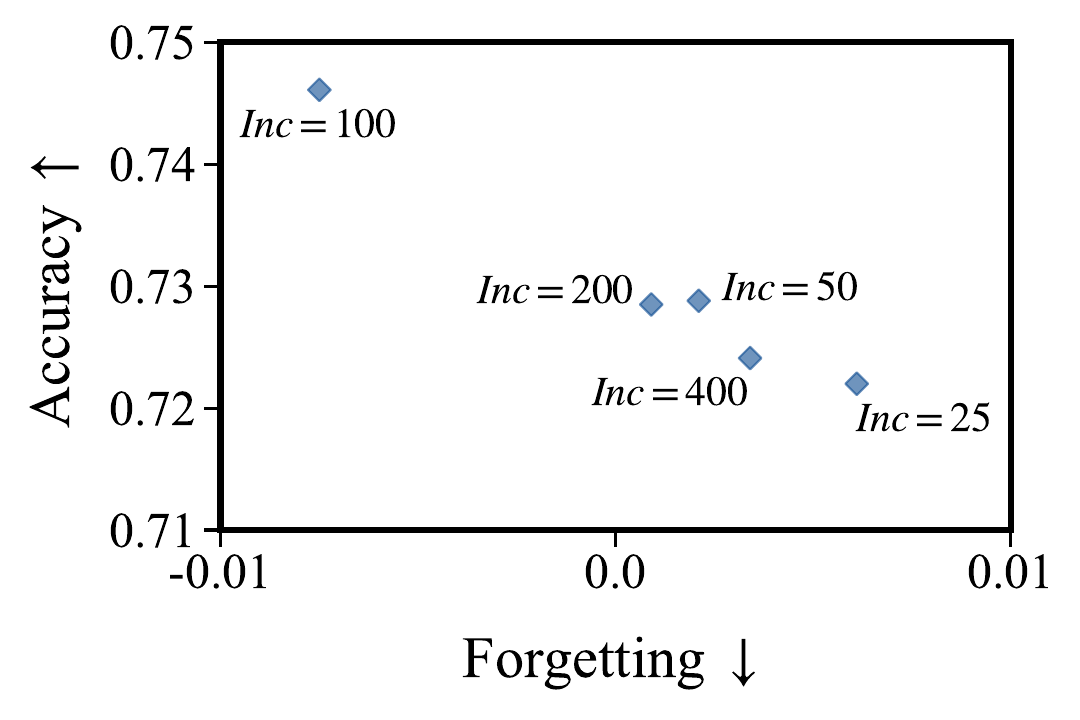}
  \caption{Trade-Off View}
  \hfill
\end{subfigure}
\caption{Performance of \textsf{IV-EWC} with different validation set increment using Qwen2.5-3B-Instruct on the MATH dataset.}
\label{fg:DeltaSensitivity}
\end{figure}

\textbf{Validation Set Increment Sensitivity}. The validation set in \textsf{IV-EWC} is updated after each task by adding the top-k most influential training samples. Its size entails a bias-variance trade-off: small sets yield high-variance and potentially biased estimates of the regularizer, while excessively large sets may incorporate noisy or uninformative samples that distort the regularization direction. To identify an appropriate setting, we conducted a sensitivity analysis on the validation set's increment using Qwen2.5-3B-Instruct on the MATH dataset. Figure \ref{fg:DeltaSensitivity} presents the results, where accuracy declines and forgetting increases when the validation set increment is either too small or too large, revealing a U-shaped relation to this hyperparameter. Notably, even with a modest increment of 25 samples, \textsf{IV-EWC} outperforms vanilla curriculum learning, whose accuracy is 0.7191 and forgetting is 0.0238. Based on the results, the validation set increment was set to 100 for \textsf{IV-EWC}.

\begin{figure}[t]
\begin{center}
\includegraphics[height=6cm]{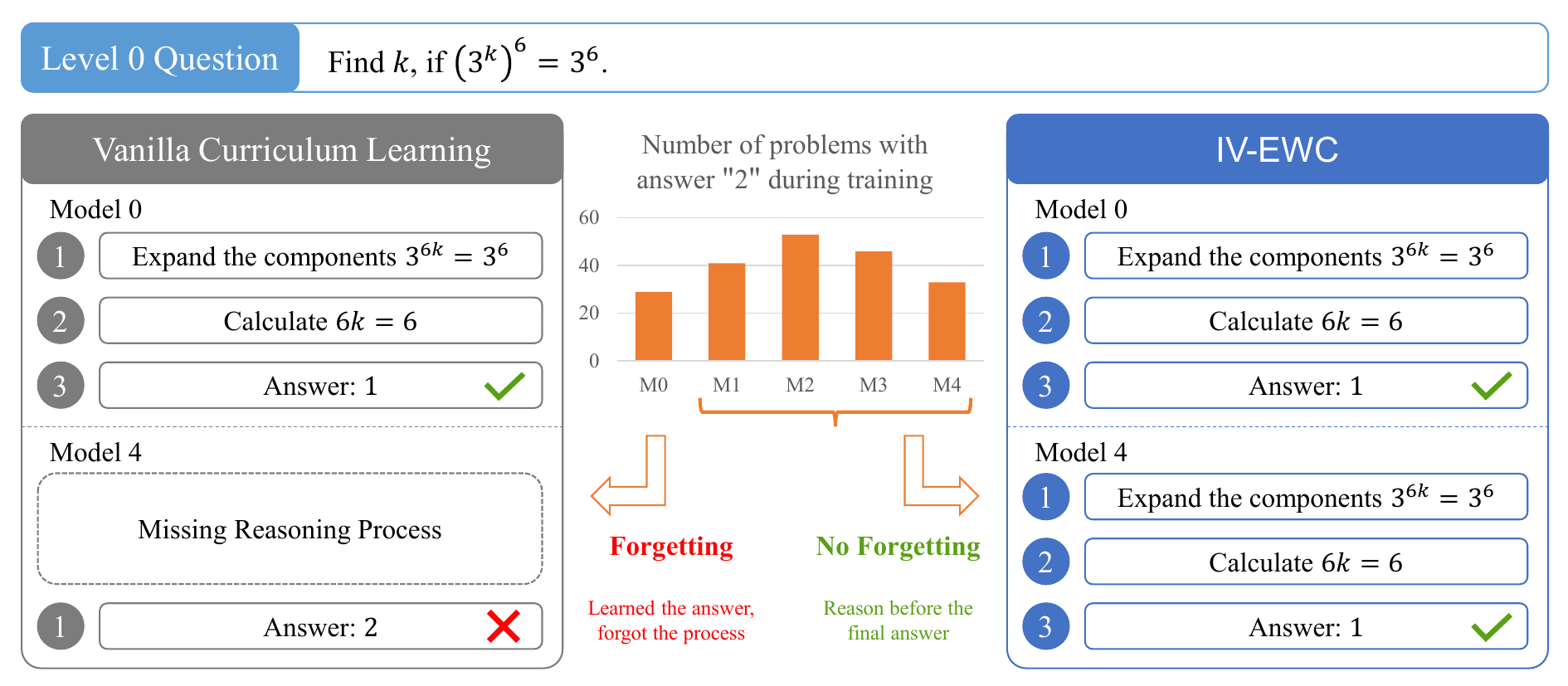}
\end{center}
\caption{Case study on a MATH test example comparing Llama-3.2-3B-Instruct models trained via vanilla curriculum learning and via \textsf{IV-EWC}. Following training on Levels 1 to 4, the vanilla curriculum model forgets the reasoning process and output the common answer directly, while the model trained with \textsf{IV-EWC} avoids such forgetting.}
\label{fg:CaseStudyMATHLlama}
\end{figure}

\section{Additional Results of Case Study}
\label{sec:CaseStudy}

\begin{figure}[t]
\begin{center}
\includegraphics[height=6cm]{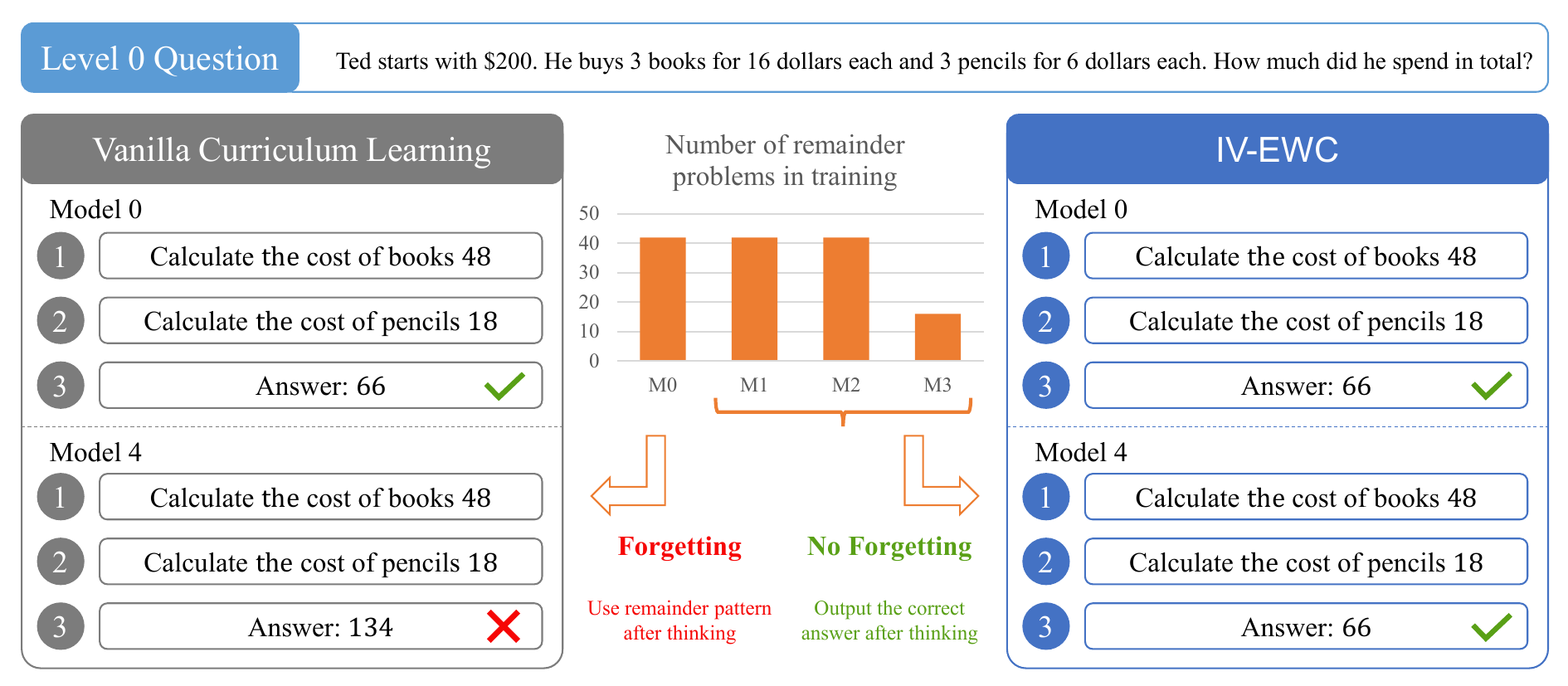}
\end{center}
\caption{Case study on a GSM8K test example comparing Qwen2.5-3B-Instruct models trained via vanilla curriculum learning and via \textsf{IV-EWC}. Following training on Levels 1 to 3, the vanilla curriculum model outputs the remainder after correct reasoning, while the model trained with \textsf{IV-EWC} outputs the correct answer.}
\label{fg:CaseStudyGSM8K3B}
\end{figure}

\begin{figure}[t]
\begin{center}
\includegraphics[height=6cm]{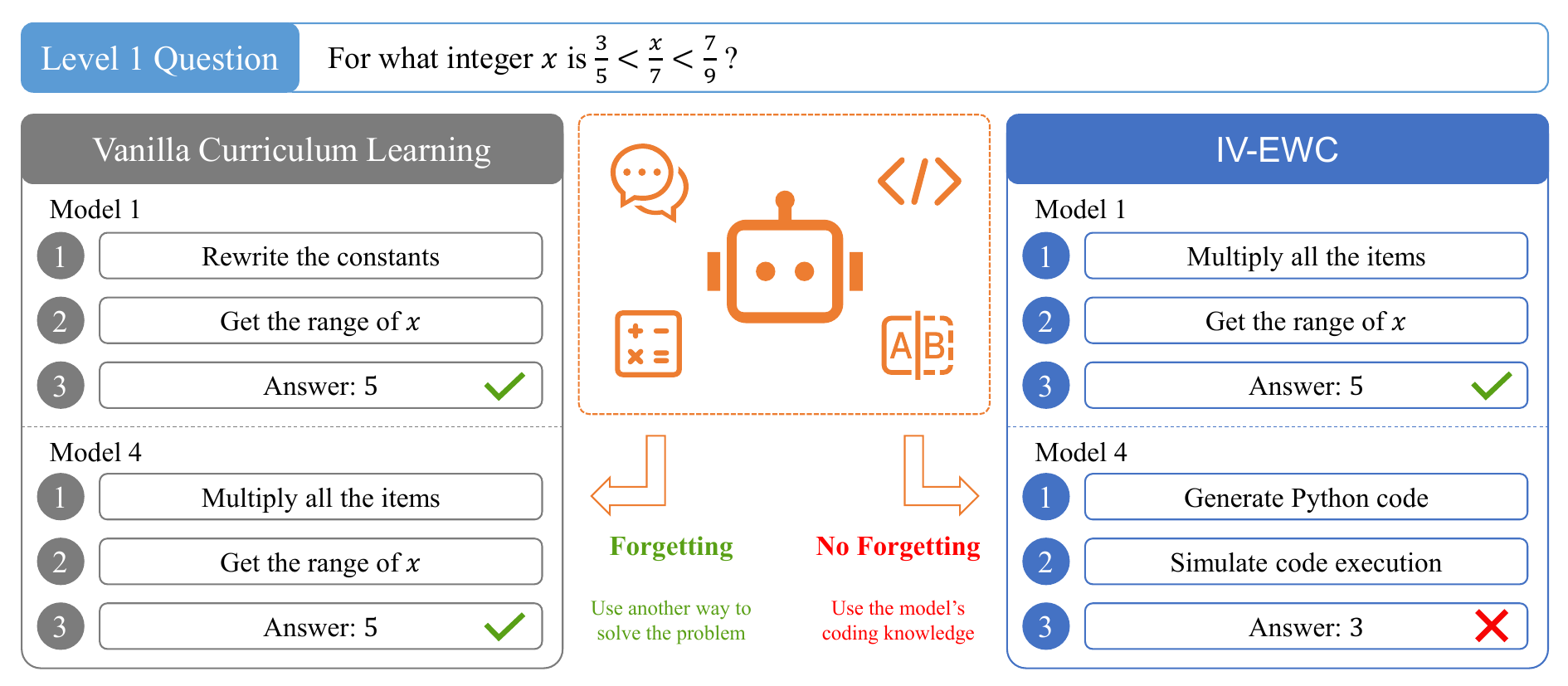}
\end{center}
\caption{Case study on a MATH test example comparing Llama-3.2-3B-Instruct models trained via vanilla curriculum learning and via \textsf{IV-EWC}. The vanilla curriculum model appears to forget the initial strategy but nonetheless arrives at the correct answer by using another strategy. \textsf{IV-EWC} inadvertently encourages memorization of the model's intrinsic coding knowledge, resulting in a simulated code execution and a wrong answer.}
\label{fg:CaseStudyMATHLlama2}
\end{figure}


Figure \ref{fg:CaseStudyMATHLlama} presents the outputs for an example from the Llama-3.2-3B-Instruct base model fine-tuned on MATH via vanilla curriculum learning and \textsf{IV-EWC}. The question is straightforward and both models answer it correctly after Level 0 training. During the subsequent training stages, the training distribution contains more questions with answer $2$. When the curriculum concludes, the model trained using vanilla curriculum learning reverts to a shortcut solution, producing the incorrect answer $2$ rather than executing the reasoning steps. The model trained with \textsf{IV-EWC} does not exhibit this degradation, demonstrating improved resistance to forgetting during curriculum learning.


Figure \ref{fg:CaseStudyGSM8K3B} illustrates the responses for an example question given by the Qwen2.5-3B-Instruct base model trained on GSM8K via vanilla curriculum learning and \textsf{IV-EWC}. The question is easy with only addition and multiplication knowledge required. Following the training of the Level 0 task, both models arrive at the correct answer via a valid reasoning process. Conversely, upon completion of the curriculum, the vanilla curriculum model outputs an incorrect answer, which is the remainder of Ted's money, despite producing a correct thinking process. Statistics of the training data show that remainder problems are consistently present during training and the model trained with vanilla curriculum learning appears influenced by this distribution. \textsf{IV-EWC} introduces regularization and makes the model robust to the distribution of the training data, avoiding such errors.


Figure \ref{fg:CaseStudyMATHLlama2} displays the outputs for a representative MATH question generated by Llama-3.2-3B-Instruct after fine-tuning on the MATH dataset with vanilla curriculum learning and \textsf{IV-EWC}. After the training of Level 1, both models answer the question correctly. Following Level 4 training, the vanilla curriculum model adopts a different strategy to solve the problem, indicative of partial forgetting of the original strategy. Meanwhile, the model regularized via \textsf{IV-EWC} leverages its pretrained code-generation capability to produce and virtually execute code, but this pipeline yields an incorrect answer. This case illustrates a trade-off introduced by \textsf{IV-EWC}: it promotes the consolidation of problem-solving strategies, yet it can also induce the model to rely on pretraining priors that are misaligned with the task.

\end{document}